\pdfoutput=1
\documentclass{article}
\ifdefined\pdfminorversion\pdfminorversion=7\fi 
\usepackage[T1]{fontenc}
\usepackage{preprint,times}

\usepackage{amsmath,amsfonts,bm}

\def\eqref#1{equation~\ref{#1}}

\def\1{\bm{1}}

\DeclareMathAlphabet{\mathsfit}{\encodingdefault}{\sfdefault}{m}{sl}
\SetMathAlphabet{\mathsfit}{bold}{\encodingdefault}{\sfdefault}{bx}{n}

\newcommand{\E}{\mathbb{E}}

\newcommand{\R}{\mathbb{R}}

\newcommand{\Var}{\mathrm{Var}}

\newcommand{\Cov}{\mathrm{Cov}}

\DeclareMathOperator{\sign}{sign}

\usepackage{hyperref}
\hypersetup{pdftitle={Feedback-Robust AI for Patient Knowledge Graphs},pdfauthor={Mohammed Sameer Syed}}
\usepackage{url}
\usepackage{graphicx}
\usepackage{booktabs}
\usepackage{multirow}
\usepackage{amsthm}
\usepackage{xcolor}
\usepackage{enumitem}
\usepackage{microtype}
\usepackage{caption}
\newtheorem{proposition}{Proposition}

\newtheorem{assumption}{Assumption}
\newtheorem{corollary}{Corollary}
\theoremstyle{definition}

\title{Feedback-Robust AI for\\ Patient Knowledge Graphs}

\author{Mohammed Sameer Syed\\
Roshan AI\\
\texttt{mohammedsameersyed1@gmail.com}}

\begin{document}

\maketitle

\begin{abstract}
Patient knowledge graphs from bedside monitoring should type their relations and state whether the data support their signs. In anesthesia and intensive care, clinicians titrate drugs and ventilation in response to the physiology, so temporal relations mix the patient's response with the clinician's policy. We introduce ClosedLoopBench: 29 relations with signs fixed by physics, pharmacology or clinical practice, on 3,442 VitalDB surgical cases (12,653 h) with negative-control action streams. When each patient's actions are replaced by another patient's, six of 12 estimators declare on average 11-18 of their 19-29 distinct relation estimates significant without calibration, and after calibration cross-correlation and Granger tests still assign ventilator rate$\to$end-tidal CO$_2$ the sign of the clinician's policy. We propose feedback-robust patient graphs that combine concept nodes with evidence pointers, typed relations admitted against negative controls, and beat-level couplings. On VitalDB under null streams, our graphs contain 0.06-0.10 false concept-level relation instances per graph, versus 10-12 for correlational construction. Patient-specific estimates of 11 slow drug and ventilator responses predict later data no better than population estimates, whereas the pulse-arrival-time-systolic-pressure slope is negative in 94.3\% of 2,884 cases and patient-specific (early-late correlation 0.67 [0.63, 0.70]).
\end{abstract}

\section{Introduction}
\label{sec:intro}

Knowledge graphs (KGs) organize clinical information for prediction and interpretation \citep{rotmensch2017learning,choi2017gram,santos2022knowledge,jiang2024graphcare,jiang2025kare}. A patient KG from bedside monitoring should contain concepts the signals support, with pointers to their evidence, and typed relations marked by whether the data support their sign. Graphs from physiological time series take their edges from lagged cross-correlation \citep{bashan2012network} or Granger tests \citep{granger1969investigating}; we call this family, applied with co-occurrence to concept activation series, correlational KG construction (CKG).

Two properties of clinical monitoring limit which relations such a graph can contain. First, the clinician closes the loop: in anesthesia and intensive care, drugs and ventilator settings are titrated in response to the physiology they act on, so an action-vital association mixes the patient's response with the clinician's policy and can take the sign opposite to the drug effect, as in the drug titration paradox \citep{schnider2021drug,schamberg2021drug} and the direct-method bias of closed-loop identification \citep{forssell1999closed}. Second, a surgical case contains few changes of each drug target, so a within-patient slow relation rests on a handful of transitions.

We measure the first property with ClosedLoopBench (3,442 VitalDB cases; \citealp{lee2022vitaldb}), whose negative-control action streams \citep{lipsitch2010negative} give each patient another patient's actions (cross-patient transplant) or its own action steps at random times (random-time). Brackets give 95\% confidence intervals (CIs) throughout. Under transplant, six of 12 estimators declare on average 11 to 18 of their 19-29 distinct relation estimates significant without calibration. After calibration, cross-correlation and Granger tests still assign ventilator rate$\to$end-tidal CO$_2$ (EtCO$_2$) the sign of the clinician's policy, whereas design-based estimators recover the physical sign (Sec.~\ref{sec:exp-benchmark}).

\begin{figure}[t]
\centering
\includegraphics[width=0.96\linewidth]{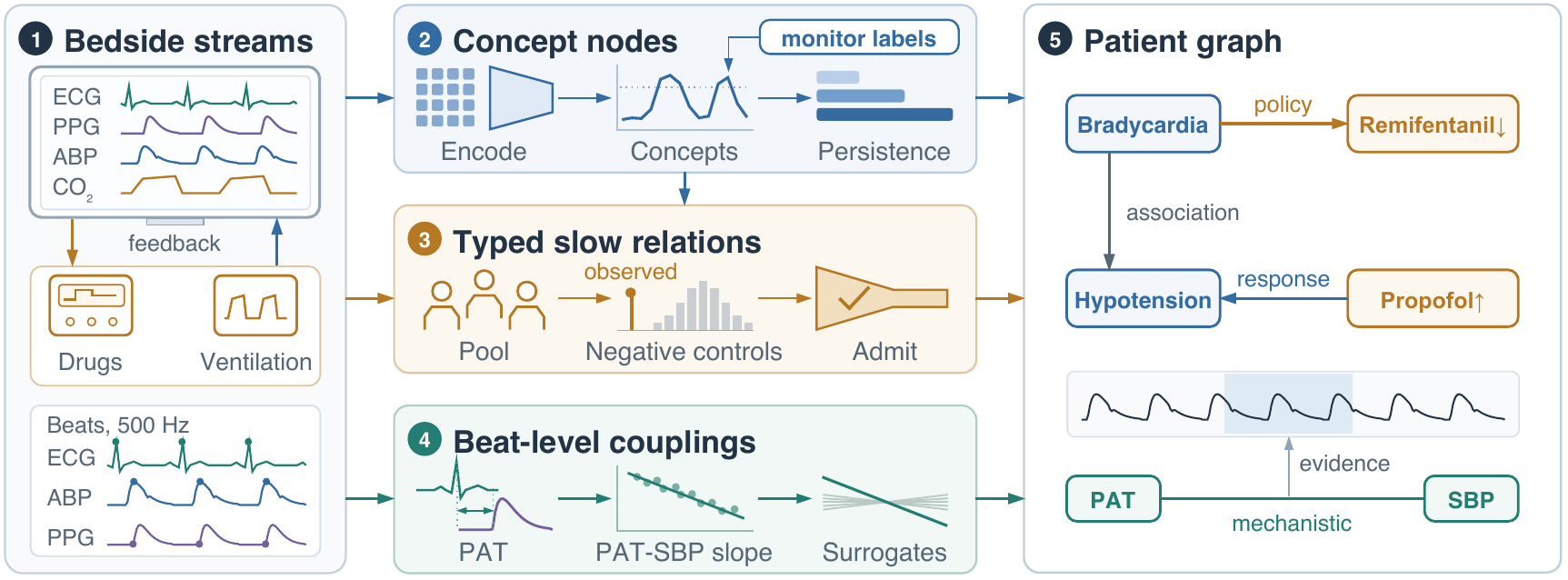}
\caption{Overview. Bedside streams (1), where clinicians titrate drugs and ventilation on the vitals, yield concept nodes (2), typed slow relations admitted against negative controls (3) and surrogate-tested beat-level couplings (4), which form a patient graph with evidence pointers (5).}
\label{fig:architecture}
\end{figure}

We propose feedback-robust patient KGs (Figure~\ref{fig:architecture}). Their concept nodes carry persistence tiers and evidence pointers; their slow relations are typed (policy, response, association), pooled across patients and admitted against negative controls under false-discovery-rate control \citep{benjamini1995controlling,benjamini2001control}; their beat-level couplings are estimated per patient. On VitalDB under null streams, our graphs contain 0.057 [0.014, 0.107] (transplant) and 0.103 [0.035, 0.184] (circular shift) false concept-level relation instances per graph, against 12.40 and 10.07 for CKG (Sec.~\ref{sec:exp-graphs}).

Finally, pre-specified split-half tests (App.~\ref{app:prereg}) locate patient-specific structure: for 11 slow drug and ventilator responses, a patient's own early-half estimate predicts the late half no better than the population estimate. In contrast, the slope of pulse arrival time (PAT, from the electrocardiogram R wave to the peripheral pulse) on systolic blood pressure (SBP) is negative in 94.3\% [93.4, 95.1] of 2,884 cases and patient-specific (early-late correlation 0.67 [0.63, 0.70]; Sec.~\ref{sec:exp-specificity}). Our contributions are:
\begin{enumerate}[leftmargin=*,itemsep=0pt,topsep=1pt]
\item \textbf{ClosedLoopBench}, a real-data benchmark for temporal relation discovery under clinical feedback: 29 relations of a priori sign, two negative-control streams and 12 estimators (Sec.~\ref{sec:benchmark}).
\item \textbf{Feedback-robust patient KGs}, whose tiered concept nodes, typed relations admitted against negative controls, and beat-level couplings point to their signal evidence (Sec.~\ref{sec:method}).
\item \textbf{Analysis} of feedback: sign reversal and future-inaction bias, specializing the titration paradox and conditioning on future treatment \citep{schnider2021drug,suissa2008immortal} to temporal relation discovery (Sec.~\ref{sec:setting}), and split-half tests in which beat-level couplings, but none of 11 slow responses, are patient-specific (Sec.~\ref{sec:exp-specificity}).
\end{enumerate}

\section{Related work}
\label{sec:related}

\paragraph{Clinical and patient KGs.}
Clinical KGs have been learned from co-occurrence in health records \citep{rotmensch2017learning}, taken from ontologies \citep{choi2017gram}, assembled from biomedical databases and literature \citep{santos2022knowledge}, and combined with large language models for patient-level prediction \citep{jiang2024graphcare,jiang2025kare}. Graphs from monitoring have correlational edges \citep{bashan2012network}. We derive typed relations from monitoring and clinician actions, attach signal evidence to every node and edge, and map concepts to SNOMED Clinical Terms (SNOMED CT; \citealp{donnelly2006snomed}) where a code exists.

\paragraph{Temporal causal discovery.}
Granger causality \citep{granger1969investigating}, PCMCI+ \citep{runge2019detecting,runge2020discovering} and VAR-LiNGAM \citep{hyvarinen2010estimation} infer lagged and contemporaneous links from time series; surrogate and shift tests control false positives under serial dependence \citep{kramer2009network,novelli2019large,harris2020shift,liang2025false}. Benchmarks simulate series with known graphs \citep{cheng2024causaltime}, use river networks \citep{stein2025causalrivers} or stress-test assumptions \citep{stein2026tcdarena}; to our knowledge, none uses real data in which an agent acts on the recorded variables and the signs of both its policy and the system's response are fixed in advance.

\paragraph{Feedback, titration and observational designs.}
Feedback biases direct identification under noise-model error \citep{forssell1999closed}, as for the baroreflex \citep{kamiya2011closed,kawada2021closed}. In anesthesia, titration makes more drug correlate with less effect \citep{schnider2021drug,schnider2022drug,schamberg2021drug,minto2024drug,sarraf2024drug}, with a sign that depends on how titration is performed \citep{sarraf2025drug}, and pharmacometric models recover the concentration-effect relation when the drivers of titration are modeled \citep{kristensen2022dose,goulooze2025drug}. \citet{hizli2023causal} model policy and response jointly; we estimate them separately. Conditioning on future treatment causes immortal-time bias \citep{suissa2008immortal,hernan2016specifying}, which target-trial and sequential-trial emulation avoid \citep{hernan2016using,hernan2008observational,danaei2013observational}; related problems arise with event-dependent exposures \citep{farrington2009case} and in event studies whose covariates respond to past outcomes and treatment \citep{botosaru2026event}. Individual responses are identified from repeated exposures \citep{senn2016mastering} or well-excited induction phases \citep{soltesz2013individualized}.

\paragraph{Concepts and beat-level coupling.}
Concept bottleneck models predict through interpretable concepts \citep{koh2020concept,syed2026shifamind}, interpretable models suit high-stakes settings \citep{rudin2019stop}, and rationale faithfulness is often scored by evidence deletion \citep{deyoung2020eraser}; we edit the signals counterfactually instead. At the beat timescale, closed-loop models \citep{porta2013model} and the sequence method \citep{parati1988evaluation} estimate baroreflex sensitivity, pulse transit time shortens as blood pressure rises and pulse arrival time tracks systolic pressure \citep{mukkamala2015toward}, and respiratory pulse pressure variation predicts fluid responsiveness \citep{michard2000relation}.

\section{Relations under clinical feedback}
\label{sec:setting}

\paragraph{Notation.}
Records (surgical cases or intensive-care stays) are indexed by $p=1,\dots,P$ and called patients. Time $t$ lies on a grid of step $\Delta$ (15\,s for concept states, 2\,s for vitals in ClosedLoopBench); beat-level quantities are indexed by beat $n$. The vitals are $y_{p,j}(t)$, $j\in\mathcal{J}$: heart rate (HR), mean arterial pressure (MAP), EtCO$_2$, oxygen saturation (SpO$_2$) and bispectral index (BIS). An action $a\in\mathcal{A}$ (an infusion target or rate, or a ventilator setting) has value $u_{p,a}(t)$ and events $\mathcal{E}_{p,a}=\{\tau: u_{p,a}(\tau)\neq u_{p,a}(\tau^-)\}$, with $u_{p,a}(\tau^-)$ the value just before $\tau$ and direction $d\in\{\uparrow,\downarrow\}$. A concept $c\in\mathcal{C}$ (e.g., hypotension) has a binary state $S_{p,c}(t)$ where observed; an episode is a maximal run of $S_{p,c}=1$. A node $v=(c,\kappa)$ pairs a concept with a persistence tier $\kappa\in\{\text{transient},\text{episodic},\text{persistent}\}$ and carries an evidence set $\Pi(v)$ of the windows or beat ranges that activated it.

A relation $r=(x\to z,\ \mathrm{type}_r,\ \sigma_r,\ h)$ links a source $x$ to a target $z$ (concepts, vitals, actions or beat variables) with expected sign $\sigma_r\in\{-1,+1\}$ at horizon $h$. A \emph{policy} relation (vital or concept$\to$action) describes the clinician's decision rule and a \emph{response} relation (action$\to$vital or concept) the effect of an action; an \emph{association} links concepts on different vitals without a causal reading, and a \emph{mechanistic} relation is a beat-level coupling. An estimator returns $\hat\theta_r$; $\theta_r(h)$ is its limit when actions are independent of the physiology (the open-loop estimand at horizon $h$).

\paragraph{Identification under feedback.}
Negative-control streams \citep{lipsitch2010negative} carry neither a response of the patient nor a policy toward it; estimates on them show what schedules and serial dependence alone produce. A null stream $\tilde u^{(k)}$, $k=1,\dots,K$, is a \emph{cross-patient transplant} (a derangement of action streams among the patients receiving $a$) or a \emph{random-time} stream (the patient's own steps at uniformly random times); concept streams are also nulled by \emph{within-patient circular shifts}.

\begin{proposition}[Sign reversal; informal]\label{prop:reversal}
In the linear closed-loop model of App.~\ref{app:theory}, a lagged or event contrast for a response relation $r$ converges to $\theta_r(h)+B_r(h)$, where the bias $B_r(h)$ depends on the clinician's policy and on the footprint of the unrecorded drivers the clinician reacts to. For a post-minus-pre level contrast, Gaussian vital fluctuations with exponentially decaying autocorrelation and a clinician who corrects deviations of the displayed vitals (Corollary~\ref{cor:crossover}), $B_r(h)$ opposes $\theta_r(h)$ below a crossover horizon, when one exists, and reinforces it above. The estimate has the wrong sign if and only if $\sign B_r(h)=-\sign\theta_r(h)$ and $|B_r(h)|>|\theta_r(h)|$.
\end{proposition}

\begin{proposition}[Future-inaction bias; informal]\label{prop:pools}
If control (placebo) times are restricted to times not followed by an action within the outcome horizon, event contrasts are biased; if later actions are one-directional threshold decisions, the bias has the sign of the policy, which for a corrective policy opposes the true effect. Control times defined from information available before $t$ remove this bias under sequential exchangeability and overlap.
\end{proposition}

Proposition~\ref{prop:reversal} specializes the titration paradox \citep{schnider2021drug} and direct-method bias \citep{forssell1999closed} to event-triggered, anticipatory controllers, with its horizon dependence; Proposition~\ref{prop:pools} is the event-study form of conditioning on future treatment \citep{suissa2008immortal,hernan2016specifying}, remedied by sequential trial emulation \citep{hernan2008observational}. Proofs and simulations are in App.~\ref{app:theory}.

\section{Feedback-robust patient graphs}
\label{sec:method}

\subsection{Concept nodes and evidence}
\label{sec:method-nodes}

Concept states are defined on 30-s windows $w$ at a 15-s stride. A vital concept is a threshold on a bedside-monitor numeric (e.g., hypotension if MAP $<65$\,mmHg; App.~\ref{app:concepts-vocab}), applied to the window median. Where only waveforms are available, a multimodal model predicts it: Transformer encoders \citep{vaswani2017attention} for the electrocardiogram (ECG), photoplethysmogram (PPG), arterial blood pressure (ABP) and respiration, pretrained by masked-token reconstruction \citep{he2022masked}, attention fusion and one linear head per concept \citep{koh2020concept} (App.~\ref{app:concepts-model}). A head outputs the probability $\hat q_{p,c}(w)$, and $S_{p,c}(t)=\mathbf{1}[\hat q_{p,c}(w_t)\ge\lambda_c]$ for a threshold $\lambda_c$, with $w_t$ the window ending at step $t$. \emph{Monitor-supervised} heads are trained on the monitor's numeric for the same window, \emph{rule-supervised} heads on the same thresholds applied to waveform-derived estimates. Beat concepts (e.g., irregular R-R intervals, prolonged PAT) are computed from beat series and describe signal properties, not diagnoses. An episode is transient, episodic or persistent if its duration is at most $\vartheta^{(1)}_c$, at most $\vartheta^{(2)}_c$ (the 33rd and 67th duration percentiles), or longer. Node $v=(c,\kappa)$ collects the episodes of its tier, and $\Pi(v)$ stores their time spans and, for beat concepts, their beat ranges.

\subsection{Typed slow relations}
\label{sec:method-slow}

\paragraph{Estimation.}
For a concept relation $r=(x\to z)$, a discrete-time hazard of the onset of $z$ is fitted pooled across patients on the 15-s grid, with patient fixed effects $\alpha_p$:
\begin{equation}\label{eq:hazard}
\operatorname{logit}\Pr\bigl(S_{p,z}(t)=1 \,\big|\, S_{p,z}(t-1)=S_{p,z}(t-2)=0,\ \cdot\,\bigr)=\alpha_p+\beta_r^{\top}e_{p,x}(t)+\gamma^{\top}w_p(t),
\end{equation}
where the exposure $e_{p,x}(t)$ holds lag-bin indicators of the source state and of source onsets (for an action source, its changes in direction $d$) up to 10\,min earlier, with coefficients $\beta_r$, and $w_p(t)$ holds adjustment covariates (time since $z$ was last on, case phase, recent doses and actions) with coefficients $\gamma$. The estimate $\hat\theta_r=\hat\omega_r$ is the average of $\hat\beta_r^{\top}e_{p,x}(t)$ over exposed time steps, an average conditional log odds ratio of onset; policy relations model changes of action $a$ in direction $d$ instead of onsets (App.~\ref{app:relations}). For vital relations, a direction-split vector autoregression with exogenous inputs (VARX) with patient fixed effects gives the response of $y_{p,j}$ at horizon $h$ to a unit change of $u_{p,a}$ or, for a policy, the effect of the vital's 2-min trend on the action's increments (App.~\ref{app:benchmark-estimators}).

\paragraph{Calibration and admission.}
Each estimator is rerun on $K$ null streams, giving the calibrated estimate $\tilde\theta_r=\hat\theta_r(u)-K^{-1}\sum_{k=1}^{K}\hat\theta_r(\tilde u^{(k)})$, where $u$ denotes the recorded streams. Its standard error $\widehat{\mathrm{se}}_r$ adds the between-draw variance to a paired case-bootstrap variance, and $p_r=2\Phi(-|\tilde\theta_r|/\widehat{\mathrm{se}}_r)$, with $\Phi$ the standard normal distribution function. Relation $r$ is admitted if and only if at least five cases contribute and
\begin{equation}\label{eq:admission}
0\notin\tilde\theta_r\pm1.96\,\widehat{\mathrm{se}}_r
\;\;\wedge\;\;
p_r^{\mathrm{BY}}\le0.05
\;\;\wedge\;\;
\frac{1}{K}\sum_{k=1}^{K}\mathbf{1}\bigl[\mathrm{call}_k(r)=\sign\tilde\theta_r\bigr]\ge 0.75,
\end{equation}
where $p_r^{\mathrm{BY}}$ is the Benjamini-Yekutieli adjusted $p$-value over the relation family \citep{benjamini2001control}, and $\mathrm{call}_k(r)$ is the sign of the paired 95\% CI of $\hat\theta_r(u)-\hat\theta_r(\tilde u^{(k)})$, or 0 if that CI covers 0. We use $K=8$ transplants. Applied to VARX, Eq.~\ref{eq:admission} is our \emph{typed pooled estimator}; ClosedLoopBench applies it to every estimator. Concept relations are calibrated against 38 null replicates (19 transplants, 19 circular shifts) and admitted if they pass Benjamini-Hochberg control in their family \citep{benjamini1995controlling} and their Wald statistic exceeds every null replicate's.

\paragraph{Instantiation.}
An admitted relation carries its type, the sign of $\tilde\theta_r$ and the population estimate with its CI, but no patient-specific weight; it is instantiated in patient $p$ if $p$'s record contains an evidence pair within its lag window (e.g., a source episode, then a target onset), stored as a pointer.

\subsection{Beat-level couplings}
\label{sec:method-fast}

From ECG, arterial pressure and PPG at 500\,Hz we extract beats $n$ with R-R interval $\mathrm{RRI}_n$, systolic pressure $\mathrm{SBP}_n$, $\mathrm{PAT}_n$ and capnographic respiratory phase $\phi_n$ (App.~\ref{app:fast}); PAT contains a device delay that is constant within a record, and its relation to blood pressure is subject-specific \citep{mukkamala2015toward}, so only its within-patient variation is used. On each 5-min epoch $e$, $b_{p,e}$ is the least-squares slope of detrended $\mathrm{PAT}_n$ on detrended $\mathrm{SBP}_n$; respiration$\to$SBP and respiration$\to$RRI (respiratory sinus arrhythmia, RSA) are first-harmonic amplitudes on $\phi_n$. Per patient, the epoch mean is tested against 499 surrogates that destroy the coupling but keep each series, with Holm control across couplings \citep{holm1979simple}, and epochs passing Benjamini-Hochberg control become pointers. A signed coupling is tested per patient only if the share of cases with its expected sign has a 95\% CI lower bound of at least 0.90; otherwise the graph stores an abstention with its reason. Mechanistic edges from the respiratory phase are directed; the PAT-SBP edge is undirected.

\subsection{Graph assembly}
\label{sec:method-graph}

A patient graph has concept$\times$tier nodes with their episodes, action nodes with time-stamped changes, and variable nodes. Each edge is a relation instance storing type, population sign, estimate with 95\% CI, lag, pointers and a sign status: a relation is \emph{sign-stable} if its sign is admitted consistently (for concept relations, in both halves of the records; App.~\ref{app:graphs-status}), \emph{sign-unresolved} if a conflicting sign is admitted or its sign contradicts a sign check, and \emph{unconfirmed} otherwise; concept responses are never sign-stable because their family fails pre-specified pharmacological sign checks (App.~\ref{app:relations-signs}). Sign stability is an empirical check, not causal identification. In the graphs of Sec.~\ref{sec:exp-graphs}, the vital-level layer applies the typed pooled estimator with one transplant draw and no stability condition to 18 relations.

\section{ClosedLoopBench}
\label{sec:benchmark}

\paragraph{Data and relations.}
We use 3,442 cases (12,653 recorded hours) of the 6,388 in the single-center VitalDB \citep{lee2022vitaldb}, with vitals on a 2-s grid from monitor numerics at their native time stamps, and 171,758 action events: propofol and remifentanil infusion targets, four ventilator settings, three vasoactive infusion rates and derived volatile-anesthetic steps (App.~\ref{app:benchmark-data}). The 29 relations (App.~\ref{app:benchmark-relations}) comprise 19 responses (five gas-exchange relations with signs from respiratory physics, eleven drug responses with signs from pharmacology, two volatile-exposure relations and positive end-expiratory pressure (PEEP)$\to$MAP), scored at $h=60$\,s (ventilation and PEEP) or 120\,s (drugs, volatile steps and oxygenation), and 10 policies with signs from documented practice (e.g., rising MAP$\to$remifentanil increase), which relate the vital's 2-min trend to the action. Increases and decreases are separate relations.

\paragraph{Estimators and metrics.}
We evaluate 12 estimators (App.~\ref{app:benchmark-estimators}): lag-based estimators (positive-peak, signed-peak and direction-aware lagged cross-correlation; pairwise Granger \citep{granger1969investigating}; PCMCI+ \citep{runge2020discovering}; VAR-LiNGAM \citep{hyvarinen2010estimation}; VARX) and event designs (naive event and level contrasts, and three design-based estimators: a matched event design with bias-corrected placebo times \citep{abadie2011bias}, an interrupted-trend design \citep{hausman2018regression} and a sequential-trial contrast \citep{hernan2008observational,danaei2013observational}). We report (i) uncalibrated null discoveries, the estimates whose 95\% CI excludes 0 on a transplant draw, counted over distinct estimates because some estimators share one estimate across relations; (ii) calibrated admissions on real data under Eq.~\ref{eq:admission}, correct when the sign equals $\sigma_r$; and (iii) false admissions when a held-out transplant draw or the random-time stream replaces the real streams.

\section{Experiments}
\label{sec:experiments}
\subsection{Temporal relation discovery under clinical feedback}
\label{sec:exp-benchmark}

\begin{figure}[t]
\centering
\includegraphics[width=\linewidth]{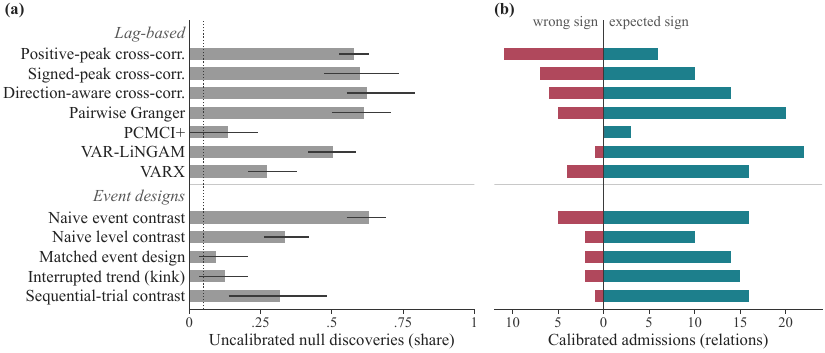}
\caption{ClosedLoopBench by estimator (2-s grid). (a) Share of distinct estimates significant without calibration under transplanted actions (mean and range over 8 draws; dotted: nominal 5\%). (b)~Calibrated admissions on real data with the expected and the wrong sign. PCMCI+ runs on 30-s block means of 60 cases.}
\label{fig:benchmark}
\end{figure}

\begin{table}[t]
\caption{ClosedLoopBench (2-s grid). Null: uncalibrated significant/distinct estimates per transplant draw (mean of 8). Calibrated admissions with the expected/wrong sign (PEEP$\to$MAP: none); PCMCI+: 30-s block means, 60 cases; Vol.: volatile steps; RT: random-time false admissions. Top: lag-based; bottom: event designs. $^\dagger$Typed pooled estimator.}
\label{tab:benchmark}
\centering
\footnotesize
\setlength{\tabcolsep}{4pt}
\begin{tabular}{@{}lcccccc@{}}
\toprule
 & Null & \multicolumn{4}{c}{Calibrated, expected/wrong} & RT \\
\cmidrule(lr){3-6}
Estimator & (uncal.) & Gas (5) & Drug (11) & Vol. (2) & Policy (10) & \\
\midrule
Positive-peak cross-corr. & 11.00/19 & 1/3 & 1/6 & 0/0 & 4/2 & 13 \\
Signed-peak cross-corr. & 11.38/19 & 2/2 & 5/0 & 2/0 & 1/5 & 13 \\
Direction-aware cross-corr. & 18.12/29 & 2/2 & 6/0 & 2/0 & 4/4 & 10 \\
Pairwise Granger & 14.75/24 & 2/3 & 8/0 & 2/0 & 8/2 & 16 \\
PCMCI+ & 3.88/28-29 & 0/0 & 0/0 & 1/0 & 2/0 & 0 \\
VAR-LiNGAM & 12.12/24 & 5/0 & 8/0 & 2/0 & 7/1 & 12 \\
VARX$^\dagger$ & 7.88/29 & 4/1 & 4/1 & 2/0 & 6/2 & 4 \\
\midrule
Naive event contrast & 18.25/29 & 4/0 & 2/4 & 2/0 & 8/1 & 7 \\
Naive level contrast & 6.38/19 & 3/0 & 5/2 & 2/0 & n/a & 2 \\
Matched event design & 2.75/29 & 4/0 & 3/0 & 2/0 & 5/2 & 0 \\
Interrupted trend (kink) & 3.62/29 & 4/0 & 4/0 & 2/0 & 5/2 & 0 \\
Sequential-trial contrast & 9.25/29 & 4/0 & 3/0 & 2/0 & 7/1 & 4 \\
\bottomrule
\end{tabular}
\end{table}

\paragraph{Null discoveries and calibration.}
Table~\ref{tab:benchmark} and Figure~\ref{fig:benchmark} summarize ClosedLoopBench (details in App.~\ref{app:benchmark}). Although transplanted streams carry no response or policy, the naive event contrast and direction-aware cross-correlation declared 18.25 and 18.12 of 29 estimates per draw significant, and pairwise Granger, VAR-LiNGAM and the positive- and signed-peak rules 11.00-14.75 of their 19-24 distinct estimates, against about 1.5 of 29 expected at the nominal level. After calibration, no estimator admitted any relation with a held-out transplant draw as input; with random-time input, the matched, interrupted-trend and PCMCI+ estimators admitted none and the other nine 2 to 16.

\paragraph{Responses.}
The four ventilation$\to$EtCO$_2$ relations were admitted with the physical sign by the naive event contrast, VAR-LiNGAM and the three design-based estimators (matched design: $-0.69$ [$-0.73$, $-0.64$]\,mmHg 60\,s after a set-rate increase, 3,099 cases). The cross-correlation rules and Granger admitted the set-rate relations with the opposite sign ($+0.053$ [0.044, 0.063] and $+0.0019$ [0.0013, 0.0025] for rate increases), the sign of the policy of raising the rate when EtCO$_2$ is high, consistent with Proposition~\ref{prop:reversal}. Propofol increase$\to$MAP and propofol increase$\to$BIS were admitted with the pharmacological sign by 10 of 12 estimators, and the matched and interrupted-trend estimates keep this sign after the first 30\,min of propofol changes (App.~\ref{app:benchmark-sensitivity}). Remifentanil decreases are consistent with the titration paradox \citep{schnider2021drug}: remifentanil decrease$\to$HR was admitted with the wrong sign by both naive contrasts, the positive-peak rule and VARX ($+0.22$ [0.12, 0.32]\,bpm per ng\,mL$^{-1}$) and with the pharmacological sign by no design-based estimator.

\paragraph{Policies and summary.}
Rising MAP$\to$remifentanil increase was admitted with the expected sign by 10 of the 11 policy estimators and rising EtCO$_2{\to}$set-rate increase by 9; the policies on tidal-volume and derived volatile steps were admitted with the opposite sign by 6 and 8 (App.~\ref{app:benchmark-results}). The matched and interrupted-trend designs had no false admissions under either null and admitted 9 and 10 responses, all with the expected sign.
\subsection{Patient graphs}
\label{sec:exp-graphs}

We build one graph per VitalDB case ($P=3{,}442$) and per record of the MIMIC-IV Waveform Database (MIMIC-IV-WDB, 167 records; \citealp{moody2022mimic4wdb}). MIMIC-IV-WDB graphs carry association relations only (no action records), estimated on 51 patients, with rule-supervised concept nodes. The baseline CKG derives edges from the same concept series by lagged cross-correlation, Jaccard co-occurrence and pairwise Granger tests (App.~\ref{app:graphs-ckg}).

\begin{table}[t]
\centering
\caption{Our graphs vs.\ CKG. $\overline{\mathrm{FE}}$: false concept-level relation instances per graph (mean [95\% CI] over 19 null replicates; CKG: one replicate, per-patient tests of 48 concept pairs; ours: 176 candidate relations). Replication: first to second half (presence-matched chance). Graph size: all 3,442 VitalDB and 167 MIMIC-IV-WDB graphs.}
\label{tab:graphs}
\footnotesize
\setlength{\tabcolsep}{2.5pt}
\begin{tabular}{@{}lcccc@{}}
\toprule
 & \multicolumn{2}{c}{VitalDB} & \multicolumn{2}{c}{MIMIC-IV-WDB} \\
\cmidrule(lr){2-3}\cmidrule(l){4-5}
 & Ours & CKG & Ours & CKG \\
\midrule
$\overline{\mathrm{FE}}$, transplant null & 0.057 [0.014, 0.107] & 12.40 & 0.095 [0.000, 0.237] & 11.10 \\
$\overline{\mathrm{FE}}$, shift null & 0.103 [0.035, 0.184] & 10.07 & 0.254 [0.098, 0.458] & 9.53 \\
Relation edges with pointers & 191{,}490 / 191{,}490 & n/a & 2{,}586 / 2{,}586 & n/a \\
Replication, policy & 0.430 (0.344) & n/a & n/a & n/a \\
Replication, response & 0.244 (0.160) & n/a & n/a & n/a \\
Replication, concept pairs & 0.341 (0.317) & 0.352 (0.347) & 0.491 (0.286) & 0.683 (0.638) \\
Nodes / edges per graph & 28.14 / 58.05 & 13.12 / 188.11 & 14.05 / 15.49 & 14.05 / 120.91 \\
\bottomrule
\end{tabular}
\end{table}

\paragraph{False relations.}
We replace real concept or action streams by null streams (19 transplant and 19 circular-shift replicates), repeat admission and instantiation, and count the resulting relation instances in patient $p$'s graph, $\mathrm{FE}_p$, all of them false; $\overline{\mathrm{FE}}$ is its mean over patients and replicates (Table~\ref{tab:graphs}). On VitalDB, our graphs contain $\overline{\mathrm{FE}}=0.057$ [0.014, 0.107] false concept-level relation instances per graph under transplant and 0.103 [0.035, 0.184] under shift (0.20\% and 0.37\% of real concept-level instances), against 12.40 and 10.07 for CKG (per-patient Granger with Benjamini-Hochberg control: 2.78 and 1.86; PCMCI+: 0.79 and 0.85; App.~\ref{app:graphs-fe}); on MIMIC-IV-WDB, 0.095 [0.000, 0.237] and 0.254 [0.098, 0.458] against 11.10 and 9.53. App.~\ref{app:graphs-fe} evaluates the vital-level layer.

\paragraph{Evidence and replication.}
Every edge of our graphs points to the events that instantiate it (episode pairs, action events or beat epochs; on VitalDB, 191,490 of 191,490 relation edges and 8,312 of 8,312 mechanistic edges; all 12,346 pointers checked in 100 cases are valid); CKG stores none. Relations re-estimated on the first and second halves of each record replicate beyond presence-matched chance (the replication expected if each half's relations were random present pairs; App.~\ref{app:relations-replication}) by 0.086 [0.084, 0.088] (policy), 0.084 [0.081, 0.087] (response) and 0.024 [0.020, 0.027] (association), and CKG temporal edges on the same concept pairs by 0.005 [0.002, 0.008]. Per-patient deviations from pooled concept relations lower the held-out log-likelihood by 0.257 [0.241, 0.274] nats per onset (App.~\ref{app:relations-ll}), so graphs store population relations only.

\paragraph{Content.}
Of the 35.7 relation instances per VitalDB graph (55.6 node-level edges; App.~\ref{app:graphs-size}), 20.4 are sign-stable, 2.0 sign-unresolved and 13.3 unconfirmed. The respiration$\to$SBP, RSA and PAT-SBP couplings are instantiated in 100.0\%, 98.3\% and 88.8\% of 2,904, 2,904 and 2,877 eligible patients (App.~\ref{app:graphs-fast}). A cross-case pairing null (600 cases) yields false respiration$\to$SBP and RSA edges in 11.5\% [9.1, 14.3] and 3.2\% [1.9, 4.9] of patients, and three within-patient nulls yield false PAT-SBP edges in at most 1\%. Figure~\ref{fig:casestudy} (App.~\ref{app:graphs-case}) shows one graph.

\subsection{Patient-specific structure}
\label{sec:exp-specificity}

A patient-specific edge requires the patient's own data to predict its later behavior better than the population estimate \citep[cf.][]{senn2016mastering}. For each relation, the units of patient $p$ (action events for slow responses, 5-min epochs for beat-level couplings) are split in time into an early and a late half with unit means $\hat\theta^{\mathrm{E}}_p$ and $\hat\theta^{\mathrm{L}}_p$. From early halves we form the pooled mean $\mu$ and an empirical-Bayes shrinkage estimate $\hat m_p$ of each patient's value (App.~\ref{app:fast-spec}). A relation is patient-specific if the 95\% patient-bootstrap CI of the difference in mean squared error, $\mathrm{MSE}(\hat m)-\mathrm{MSE}(\mu)$, on late-half units lies below 0; the early-late correlation $r_{\mathrm{EL}}$ correlates $\hat\theta^{\mathrm{E}}_p$ and $\hat\theta^{\mathrm{L}}_p$ across patients (App.~\ref{app:fast-spec}).

\begin{figure}[t]
\centering
\includegraphics[width=0.88\linewidth]{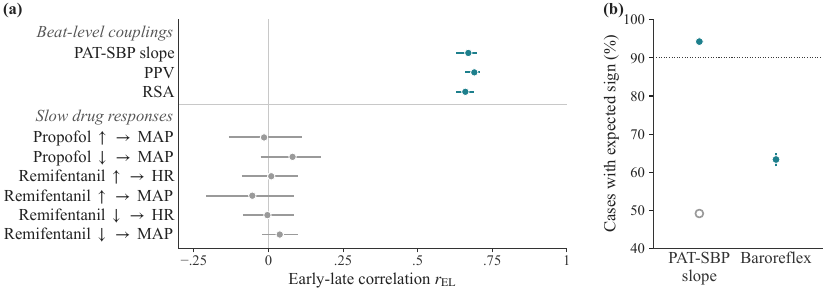}
\caption{(a) $r_{\mathrm{EL}}$ (95\% CI) of beat-level couplings (adjusted for epoch mean HR and SBP) and six slow drug responses (30-s windows). (b) Cases with the expected coupling sign (95\% CI); open circle: within-case shift null; dotted: pre-specified 90\% bar.}
\label{fig:specificity}
\end{figure}

\begin{table}[t]
\centering
\caption{Concept grounding on held-out MIMIC-IV-WDB patients: AUROC against the monitor reference (heads: mean $\pm$ s.d., 3 seeds); Gain: monitor- minus rule-supervised, seed-pooled [95\% CI]; T1: seeds passing the counterfactual test.}
\label{tab:grounding}
\footnotesize
\setlength{\tabcolsep}{2pt}
\begin{tabular}{@{}lcccccc@{}}
\toprule
 & \multicolumn{3}{c}{AUROC vs.\ monitor reference} & & \multicolumn{2}{c}{T1 passes} \\
\cmidrule(lr){2-4}\cmidrule(l){6-7}
Concept & Detector & Rule-sup. & Monitor-sup. & Gain [95\% CI] & Rule & Mon. \\
\midrule
Tachycardia (ECG) & 0.823 & $0.899\pm0.028$ & $0.990\pm0.001$ & $+0.091$ [0.023, 0.170] & 3/3 & 3/3 \\
Bradycardia (ECG) & 0.941 & $0.887\pm0.031$ & $0.985\pm0.007$ & $+0.098$ [0.039, 0.257] & 0/3 & 3/3 \\
Tachycardia (PPG) & 0.533 & $0.812\pm0.016$ & $0.979\pm0.001$ & $+0.167$ [0.040, 0.282] & 3/3 & 3/3 \\
Tachypnea & 0.584 & $0.814\pm0.004$ & $0.924\pm0.002$ & $+0.110$ [0.080, 0.142] & 3/3 & 3/3 \\
Bradypnea & 0.633 & $0.674\pm0.015$ & $0.888\pm0.018$ & $+0.214$ [0.082, 0.323] & 1/3 & 2/3 \\
Hypotension & 0.988 & $0.972\pm0.009$ & $0.983\pm0.002$ & $+0.012$ [$-0.033$, 0.039] & 3/3 & 3/3 \\
Hypertension & 0.994 & $0.995\pm0.001$ & $0.987\pm0.003$ & $-0.008$ [$-0.032$, $-0.0005$] & 3/3 & 3/3 \\
\midrule
Mean / total & 0.785 & $0.865\pm0.008$ & $0.962\pm0.002$ & & 16/21 & 20/21 \\
\bottomrule
\end{tabular}
\end{table}

\paragraph{Slow response relations.}
None of the 11 slow response relations passes, with 30-s windows (3,442 cases) or on the 2-s grid (a 500-case subsample; App.~\ref{app:fast-spec}), and $r_{\mathrm{EL}}$ is near zero even for the best-sampled relations, e.g., $-0.05$ [$-0.21$, 0.09] for remifentanil increase$\to$MAP (898 cases; Figure~\ref{fig:specificity}a).

\paragraph{Beat-level couplings.}
On the 2,954 VitalDB cases that pass beat quality control, the rule declares all five pre-specified beat-level statistics patient-specific (Holm-adjusted permutation $p$-value 0.005), including the PAT-SBP slope, pulse pressure variation (PPV) and RSA, with $r_{\mathrm{EL}}=0.67$ [0.63, 0.70] (2,861 cases), 0.69 [0.66, 0.71] and 0.66 [0.63, 0.69] (2,921 cases each) after residualizing on epoch mean HR and SBP. PPV and RSA remain patient-specific after regressing out surrogates that keep each epoch's signal properties but destroy the coupling ($r_{\mathrm{EL}}=0.71$ [0.69, 0.74] and 0.65 [0.62, 0.68]); sequence baroreflex sensitivity does not (0.16 [0.08, 0.25]). The slope $b_{p,e}$ has a negative case median in 94.3\% [93.4, 95.1] of 2,884 cases (median $-0.58$ [$-0.59$, $-0.56$]\,ms/mmHg), which meets the pre-specified positive-control bar, against 49.1\% and 52.3\% under within-case nulls that break beat alignment (Figure~\ref{fig:specificity}b). The beat-level baroreflex has the expected sign in 63.3\% [61.5, 64.9] of 2,944 cases, below the bar, and our graphs abstain from this edge.

\subsection{Concept grounding}
\label{sec:exp-grounding}

We evaluate the concept heads on held-out MIMIC-IV-WDB test patients (25 for rate concepts, 8 with an arterial line; App.~\ref{app:grounding}) against the thresholded monitor numeric of each window. Against a rule detector that thresholds a waveform-derived estimate and rule-supervised heads trained on its labels, monitor-supervised heads rank windows better: mean area under the receiver operating characteristic curve (AUROC; mean $\pm$ s.d.\ over three seeds) $0.962\pm0.002$, against $0.865\pm0.008$ and 0.785, with paired-gain CIs above 0 for the five rate concepts only (Table~\ref{tab:grounding}).

\paragraph{Counterfactual faithfulness.}
Time-warps and level edits set the rate or MAP of normal windows to 16 or 17 targets; a concept passes the counterfactual test (T1) if its probability increases with the edited quantity, its decisions match the counterfactual labels (balanced accuracy $>0.70$), and sham and amplitude edits leave its decisions nearly unchanged, all with 95\% patient-cluster CIs (App.~\ref{app:grounding-protocol}). Rule-supervised heads pass 5, 5 and 6 of 7 concepts across seeds and monitor-supervised heads 7, 7 and 6; re-initialized models fail all seven concepts and randomized heads pass at most one. Evidence-pointer (T2) and modality-masking (T3) tests pass for 3-4 and 3-5 of 7 concepts (App.~\ref{app:grounding}).

\section{Limitations}
\label{sec:limitations}

The surgical data come from one center, the intensive-care grounding cohort is small (26 test patients), and monitor-supervised heads are scored against the monitor's own labels. Calibration does not control random-time false admissions for 9 of 12 estimators (Table~\ref{tab:benchmark}), and the vital-level graph layer (one transplant draw, not pre-specified) yields 1.15 false edges per graph under a random-time null at 30\,s. Admission controls errors within, not across, relation families (a false population relation in 7 of 19 transplant and 9 of 19 shift replicates). Concept-level responses fail their sign checks, vasoactive and PEEP relations are underpowered, and boluses and surgical stimulation are not time-stamped. Concept hazard models adjust for time to recording end and so estimate retrospective associations; typed relations improve held-out onset prediction over a history-only baseline in only one of three codings (App.~\ref{app:relations-ll}). Pointers record the events that instantiate a relation, not evidence for its direction, and omit the estimator's prior-dose restriction for 147 of 23,998 policy instances. PAT includes a per-case device delay, and circular shifts are not a valid null for ventilator-locked couplings (App.~\ref{app:graphs-fast}). Inference resamples cases (96 people contribute 214 cases); patient specificity is measured within one record and refers to routine maintenance titration, not to repeated or induction-phase exposures \citep{senn2016mastering,soltesz2013individualized}; tiers and beat-concept thresholds are cohort-relative. No clinician has evaluated the graphs, which are not validated for clinical use.

\section{Conclusion}
\label{sec:conclusion}

On ClosedLoopBench, uncalibrated estimators report relations under null action streams, whereas at 2\,s every response admitted by matched and interrupted-trend designs has the expected sign. Our graphs reduce false concept-level relation instances on VitalDB under null streams from 10-12 to 0.06-0.10 per graph, and split-half tests find patient-specific beat-level couplings but no patient-specific slow response (0 of 11).

\section*{AI use statement}

In this work, we used generative AI assistants for code implementation, analysis scripting, drafting of analysis protocols, re-checking of results, literature search and reference verification, and drafting and editing of the text. We reviewed all AI-assisted work. The core library has unit tests, confirmatory analyses ran from protocols fixed before the runs (App.~\ref{app:prereg}), every reported number was traced to the analysis outputs that produced it, and every reference was checked against its DOI, arXiv or proceedings record. We take full responsibility for the content of this paper, including text, code, claims and artifacts produced with the aid of generative AI.

\section*{Ethics statement}

This retrospective study uses de-identified data released by their creators under institutional review board approval \citep{lee2022vitaldb,johnson2023mimic,moody2022mimic4wdb}; it involved no contact with patients, and no attempt was made to re-identify individuals. The PhysioNet release of VitalDB \citep{lee2022vitaldbphysionet} is openly available under the Creative Commons Attribution 4.0 International License. The MIMIC-IV Waveform Database is openly available under the Open Data Commons Open Database License; the MIMIC-IV clinical tables used to link waveform records to intensive-care stays require credentialed access, human-subjects research training and a signed data use agreement, all of which were obtained. No data are redistributed. Every relation in a patient graph carries its type (policy, response, association or mechanistic), so that a record of clinician behavior is not read as a drug effect. The graphs and estimators are research artifacts: they have not been evaluated by clinicians or validated for clinical decision-making and must not be used to guide patient care.

\section*{Reproducibility statement}

Data sources, cohorts, preprocessing, clock alignment, calibration and quality gates are described in App.~\ref{app:data}. The ClosedLoopBench relation table, estimators, null constructions, admission rule and metrics are in App.~\ref{app:benchmark}; the formal model, proofs and simulator are in App.~\ref{app:theory}. The analysis protocols, their SHA-256 hashes, deviations, compute and software versions are listed in App.~\ref{app:prereg}; every confirmatory output records the hash of the protocol it was run under, and random seeds are fixed. Both datasets are publicly available (App.~\ref{app:data}). Code for data processing, graph construction, ClosedLoopBench and all analyses will be released together with the protocols.

\bibliography{references}
\bibliographystyle{preprint}

\appendix
\section{Data and preprocessing}
\label{app:data}

\subsection{Cohorts}
\label{app:data-cohorts}

\paragraph{VitalDB.}
VitalDB \citep{lee2022vitaldb}, distributed through PhysioNet \citep{lee2022vitaldbphysionet,goldberger2000physiobank}, contains 6,388 non-cardiac surgical cases from one tertiary hospital. Each case is a continuous recording in which waveforms, patient-monitor numerics, bispectral index (BIS), anesthesia-ventilator records and infusion-pump records share one clock. We use the 3,442 cases whose track index lists lead-II electrocardiogram (ECG), photoplethysmogram (PPG), arterial pressure and capnogram waveforms (12,653 recorded hours; Table~\ref{tab:data-cohorts}). Reading the file headers confirms all four waveforms in 3,227 cases (12,199~h); 206 cases have no arterial waveform. Cases without a required waveform enter only the analyses that do not use it. Each case is one patient $p$ in the notation of Sec.~\ref{sec:setting}; the 3,442 cases come from 3,324 individuals, of whom 96 contribute 214 cases (at most 12 each). Bootstrap inference across patients resamples cases, and no person-level clustering is applied. Waveforms are sampled at 500~Hz (ECG, PPG, arterial pressure; 100~Hz in 4 cases) and 62.5~Hz (capnogram, airway pressure). Native update intervals of the numerics are about 2~s for the patient monitor, 1~s for BIS and the infusion pumps, and 7~s for the ventilator.

\paragraph{MIMIC-IV-WDB.}
The MIMIC-IV Waveform Database \citep{moody2022mimic4wdb} contains 200 intensive-care records from 198 individuals; 172 records link to an intensive care unit (ICU) stay in the MIMIC-IV clinical tables \citep{johnson2023mimic}. Records were split 140/30/30 into training, validation and test sets, stratified by quartile of ICU length of stay (the 28 records without an ICU link are in the training set). After preprocessing, 167 records have windows (114/27/26). One individual contributes one validation and one test record, both with windows; no training patient contributes an evaluation record. The test split is the primary evaluation set for concept grounding (App.~\ref{app:grounding}): 25 of its 26 patients have monitor heart-rate and respiratory-rate references, and 8 have arterial-pressure windows with a monitor pressure reference. Concept relations are estimated on the 51 validation and test patients with at least 30~min of windows, from 50 individuals (both records of that individual are included and resampled as separate patients; App.~\ref{app:relations}), and graphs are built for all 167 records (App.~\ref{app:graphs}).

\begin{table}[!htbp]
\centering
\caption{Cohorts. Counts or median (interquartile range, IQR). MIMIC-IV-WDB hours and waveform counts refer to the 167 records with windows. BMI: body-mass index; ASA: American Society of Anesthesiologists physical status (\emph{other}: class VI in 11 cases, missing in 79).}
\label{tab:data-cohorts}
\footnotesize
\setlength{\tabcolsep}{4pt}
\begin{tabular}{@{}lll@{}}
\toprule
 & VitalDB & MIMIC-IV-WDB \\
\midrule
Setting & non-cardiac surgery & intensive care \\
Records (individuals) & 3,442 (3,324) & 200 (198) \\
Records with windows & 3,442 & 167 \\
Recorded hours & 12,653 & 7,371 \\
Hours per record & 3.43 (2.36-4.71) & 25.8 (8.2-59.9) \\
Waveforms & ECG, PPG, arterial, capnogram & ECG, PPG, impedance respiration \\
Arterial waveform & 3,236 & 55 \\
Split (train/val/test) & n/a & 140/30/30 \\
\quad with windows & n/a & 114/27/26 \\
Age, years & 60.5 (50-70) & n/a \\
Female & 1,541 (44.8\%) & n/a \\
BMI, kg\,m$^{-2}$ & 23.0 (20.8-25.2) & n/a \\
ASA I-II / III-IV / other & 2,850 / 502 / 90 & n/a \\
Emergency surgery & 447 & n/a \\
Department & general 2,196; thoracic 1,031; & n/a \\
 & gynecology 171; urology 44 & \\
\bottomrule
\end{tabular}
\end{table}

\subsection{Clinician actions}
\label{app:data-actions}

Action streams $u_{p,a}(t)$ come from three sources (Table~\ref{tab:data-actions}). For target-controlled infusion (TCI) of propofol and remifentanil the clinician sets a target concentration and the pump sets its own rate, so only target changes are actions. Propofol and remifentanil are both given by TCI in 1,836 cases (every propofol TCI case) and remifentanil alone in 1,035. Vasoactive infusions are titrated by rate; 208 cases have a vasoactive pump, with 2,125 rate changes in total, of which Table~\ref{tab:data-actions} lists the three drugs used in ClosedLoopBench. Ventilator set rate, set tidal volume, set positive end-expiratory pressure (PEEP) and set inspired oxygen fraction (FiO$_2$) come from the ventilator's set-point records. Adjustments within 10~s of the previous one are merged into one action. For TCI targets, infusion rates and the ventilator set rate the detector starts from zero, so the first setting is an event (a start); for set tidal volume, PEEP and FiO$_2$ the first record is the initial value, not an event. The vaporizer dial is not recorded; volatile-agent steps are therefore derived as changes of at least 0.2 minimum alveolar concentration (MAC) in the 30-s rolling median of the measured end-tidal MAC (values below 0.2 count as off, reset after gaps longer than 120~s). Bolus drugs are recorded only as per-case totals without time stamps (ephedrine in 1,823 cases, phenylephrine boluses in 604, propofol boluses in 1,390) and are not modeled. The counts in Table~\ref{tab:data-actions} include every change in either direction, including starts and stops; the event sets that enter each analysis are defined in App.~\ref{app:benchmark} and App.~\ref{app:relations}.

\begin{table}[!htbp]
\centering
\caption{Recorded action streams in the VitalDB cohort.}
\label{tab:data-actions}
\footnotesize
\begin{tabular}{@{}llrr@{}}
\toprule
Action $a$ & Source & Events & Cases \\
\midrule
Remifentanil TCI target & infusion pump & 53,727 & 2,871 \\
Propofol TCI target & infusion pump & 23,576 & 1,836 \\
Ventilator set rate & ventilator & 24,237 & 3,388 \\
Set tidal volume & ventilator & 10,705 & 2,786 \\
Set FiO$_2$ & ventilator & 11,965 & 3,404 \\
Set PEEP & ventilator & 1,612 & 856 \\
Volatile-agent steps (derived) & measured MAC & 44,138 & 1,998 \\
Phenylephrine rate & infusion pump & 999 & 103 \\
Norepinephrine rate & infusion pump & 668 & 61 \\
Nitroglycerin rate & infusion pump & 131 & 30 \\
\bottomrule
\end{tabular}
\end{table}

\subsection{Data products and resolutions}
\label{app:data-products}

\paragraph{Window numerics ($\Delta = 15$~s).}
Monitor heart rate (HR), pulse rate from the oximeter, mean, systolic and diastolic arterial pressure (MAP, SBP, DBP), oxygen saturation (SpO$_2$), respiratory rate, end-tidal CO$_2$ (EtCO$_2$), drug state and ventilator settings are summarized as medians over 30-s windows with a 15-s stride. Zero values mark a disconnected sensor and are treated as missing. These windows provide the monitor-derived concept states $S_{p,c}(t)$ of the VitalDB graphs (App.~\ref{app:concepts}), the inputs of the hazard model (App.~\ref{app:relations}) and the graph time grid (App.~\ref{app:graphs}).

\paragraph{1-s numerics and the 2-s grid ($\Delta = 2$~s).}
A second extraction reads every numeric track at its native time stamps onto a 1-s grid without forward filling, adding BIS with its signal-quality index, measured MAC, ventilator measured and set values, and TCI target and effect-site concentrations. ClosedLoopBench uses a 2-s grid built from these records: bin $[2i, 2i+2)$~s holds the median of the valid records (MAP 20-200~mmHg, HR 20-250~bpm, EtCO$_2$ 5-100~mmHg, SpO$_2$ 50-100\%, BIS 1-100 with signal-quality index $\ge 50$). The median share of 2-s bins with a value is 0.89 (MAP), 0.94 (HR), 0.91 (EtCO$_2$), 0.95 (SpO$_2$) and 0.89 (BIS); 3,067 cases have BIS in at least half of their bins.

\paragraph{Beat and breath tables.}
At 500~Hz, R-peaks are detected with a Pan-Tompkins-type detector \citep{pan1985real}; each beat $n$ receives its R-R interval $\mathrm{RRI}_n$, the systolic, diastolic and mean pressure of the pulse it ejects, the PPG foot, the pulse arrival time $\mathrm{PAT}_n$ and the respiratory phase from the capnogram. A beat is valid if $\mathrm{RRI}_n \in [300, 2000]$~ms, SBP $\in [40, 250]$, DBP $\in [15, 150]$ and pulse pressure $\in [10, 150]$~mmHg, and if neither its RRI (by more than 20\%) nor its SBP (by more than 30~mmHg) departs from an 11-beat running median; the beat after an invalid interval is also dropped. Breaths are segmented from the capnogram. PAT matching and epoch statistics are described in App.~\ref{app:fast-pat} and App.~\ref{app:fast-epochs}.

\paragraph{Quality gates.}
Table~\ref{tab:data-qc} lists the case counts at each gate. The PAT gates exclude cases whose PPG could not be placed on the case clock, cases whose modal R-to-foot lag falls outside $[0.45, 0.90]$~s (feet matched to a neighboring beat) and cases with PAT on fewer than 80\% of valid beats.

\begin{table}[!htbp]
\centering
\caption{Quality gates for beat-level analyses (VitalDB).}
\label{tab:data-qc}
\footnotesize
\begin{tabular}{@{}lr@{}}
\toprule
Gate & Cases \\
\midrule
Cohort & 3,442 \\
Beat table (ECG, arterial and PPG waveforms present) & 3,227 \\
Valid-beat fraction $\ge 0.5$ & 2,962 \\
$|$beat HR $-$ monitor HR$|$ $\le 5$~bpm & 2,954 \\
$\ge 20$~min of valid beats (beat-level cohort) & 2,954 \\
\quad with $\ge 1$ usable 5-min epoch (105,505 epochs; 94,482 under controlled ventilation) & 2,944 \\
PPG on the case clock & 2,930 \\
Modal R-to-foot lag in $[0.45, 0.90]$~s & 2,917 \\
PAT on $\ge 80\%$ of valid beats (PAT cohort) & 2,893 \\
\quad with a PAT-SBP slope (103,382 epochs) & 2,884 \\
\bottomrule
\end{tabular}
\end{table}

\subsection{Clock alignment and calibration}
\label{app:data-clock}

\paragraph{Clocks.}
The beat, breath, 1-s and PAT tables and the 30-s window numerics are all on the clinical case-start clock; no per-case offset is applied. The origin of the file reader lies a median 0.60~s after case start (maximum 6.46~s), but the 30-s heart-rate medians are reproduced exactly from the case-start 1-s records without a shift (median share of identical windows 1.00), and zero shift fits better than a shift by the reader origin in 121 of 125 checked beat-level cases (the 25 with the largest offsets and 100 random cases with offsets of at least 1~s). Action time stamps agree with the ventilator's native set-point records without a shift: set-rate changes coincide in 3,388 cases, with a median absolute difference of 0.50~s (the 1-s rounding of the native records) and all differences within 1~s. In MIMIC-IV-WDB, monitor numerics are time-stamped in counter ticks (999.56~Hz) and windows are mapped to the record clock through the segment layout; 881 segment headers were cross-checked without mismatch. The mean of the arterial waveform agrees with the monitor's mean arterial pressure with a median absolute error of 0.46~mmHg at zero lag (19 patients, 102,060 windows); the error is smallest at a lag of 6~s, consistent with monitor averaging.

\paragraph{Capnogram gain.}
The capnogram's analog gain is 0.1 in 2,906 cases and 0.075 in 536. With gain 0.1 the waveform reads 1.32-1.34 times the ventilator and monitor EtCO$_2$ numerics and exceeds the simultaneous arterial partial pressure of CO$_2$ (PaCO$_2$) in 11 of 17 blood-gas samples (6 cases), whereas both numerics lie below PaCO$_2$ in 17 of 17. Breath-level CO$_2$ is therefore scaled per case by the ratio of the median ventilator EtCO$_2$ numeric to the median breath end-tidal CO$_2$, over breaths with a numeric above 15~mmHg (at least 30 such breaths). The factor is 0.753 (IQR 0.747-0.758; 2,903 cases) for gain 0.1 and 1.004 (0.997-1.012; 536 cases) for gain 0.075. Raw values are kept.

\paragraph{Arterial offset and heart-rate agreement.}
The arterial waveform and the monitor's arterial numerics differ by a per-case additive offset (monitor minus waveform MAP: median $-0.85$~mmHg, IQR $-3.97$ to $4.10$, range $-9.6$ to $11.0$; 3,221 cases). The offset is recorded and not applied: beat-level analyses use within-case variation, and threshold concepts use the monitor numerics. Beat-derived HR exceeds the monitor HR by a median 0.47~bpm (3,221 cases); of the 2,962 cases with at least half of their beats valid, 6 with an absolute bias above 5~bpm and 2 without an estimate fail the heart-rate gate of Table~\ref{tab:data-qc}.

\subsection{Windows for the concept models}
\label{app:data-windows}

Both datasets pass through one pipeline. Signals are cut into 30-s windows with a 15-s stride. MIMIC-IV-WDB waveforms are used at the record base rate (62.47~Hz); VitalDB waveforms are resampled to 62.5~Hz with an anti-aliased rational resampler that preserves gaps. PPG and respiration are band-pass filtered with a zero-phase fourth-order Butterworth filter (0.5-8~Hz and 0.1-1~Hz; missing spans are interpolated for filtering and restored afterwards); ECG and arterial pressure are not filtered. ECG, PPG and respiration are standardized per segment; arterial pressure is mapped to $(\mathrm{ABP}/\mathrm{mmHg} - 85)/20$, which keeps its absolute level. ECG keeps up to three leads in priority order (II, V and aVR first), so its first channel is lead II when recorded. A window with more than 20\% missing samples is flagged and excluded from evaluation. The respiration channel is an impedance pneumogram in MIMIC-IV-WDB and the capnogram in VitalDB.

\section{Concept vocabulary and node construction}
\label{app:concepts}

This appendix defines the concepts (App.~\ref{app:concepts-vocab}), the waveform concept model and its training (App.~\ref{app:concepts-model}), the construction of monitor labels (App.~\ref{app:concepts-labels}), persistence tiers (App.~\ref{app:concepts-tiers}) and beat concepts (App.~\ref{app:concepts-beat}).

\subsection{Vital concepts}
\label{app:concepts-vocab}

Table~\ref{tab:concept-vocab} lists the vital concepts. A concept is evaluated on each 30-s window, and its state is unobserved where the defining numeric is missing. For surgical records, the numeric is the window median of the anesthesia monitor's records, with zeros treated as missing (App.~\ref{app:data-products}). For intensive-care records, it is the median of at least ten valid 1-Hz monitor samples within the window (App.~\ref{app:concepts-labels}). Heart-rate, respiratory-rate and pressure thresholds are identical for monitor labels, rule labels and graph concepts. Codes are SNOMED CT concept identifiers \citep{donnelly2006snomed}. The capnography concepts, defined on EtCO$_2$ rather than on blood gases, carry local codes.

\begin{table}[h]
\centering
\caption{Vital concepts. Each concept is a finding on one window, not a diagnosis: hypotension and hypertension denote a window with MAP below 65 or above 105\,mmHg, and their codes are the SNOMED CT concepts \emph{low blood pressure} (45007003) and \emph{increased blood pressure} (24184005). Tier cut-offs $\vartheta^{(1)}_c$ / $\vartheta^{(2)}_c$: maximum episode duration (s) of the transient and episodic tiers.}
\label{tab:concept-vocab}
\small
\setlength{\tabcolsep}{3pt}
\begin{tabular}{@{}llllc@{}}
\toprule
Concept & Rule on window median & Waveform & SNOMED CT & Tier cut-offs \\
\midrule
Tachycardia & HR $>100$\,bpm & ECG & 3424008 & 15 / 45 \\
Bradycardia & HR $<60$\,bpm & ECG & 48867003 & 15 / 30 \\
Tachycardia (PPG) & pulse rate (SpO$_2$) $>100$\,bpm & PPG & 3424008 & 30 / 60 \\
Tachypnea & respiratory rate $>20$\,/min & respiration & 271823003 & 15 / 45 \\
Bradypnea & respiratory rate $<12$\,/min & respiration & 86684002 & 30 / 60 \\
Hypotension & MAP $<65$\,mmHg & ABP & 45007003 & 30 / 75 \\
Hypertension & MAP $>105$\,mmHg & ABP & 24184005 & 30 / 75 \\
Hypercapnia & EtCO$_2>45$\,mmHg (valid 15-100) & n/a & local & 60 / 600 \\
Hypocapnia & EtCO$_2<30$\,mmHg (valid 15-100) & n/a & local & 60 / 600 \\
\bottomrule
\end{tabular}
\end{table}

In the surgical records, HR, MAP (arterial line), respiratory rate (capnography) and EtCO$_2$ come from the anesthesia monitor. The relation analyses use the six heart-rate, MAP and respiratory-rate concepts and the two capnography concepts. The pulse-rate concept is used only with waveforms, as a PPG-based counterpart of the ECG heart-rate concepts.

\subsection{Waveform concept model}
\label{app:concepts-model}

\paragraph{Data.}
Waveform concepts are learned on MIMIC-IV-WDB \citep{moody2022mimic4wdb,goldberger2000physiobank} with the training, validation and test split of App.~\ref{app:data-cohorts} (no training patient contributes an evaluation record) and the 30-s windows of App.~\ref{app:data-windows} (15-s stride; 1,874 samples at 62.47\,Hz). ECG uses leads II, V and aVR; PPG, ABP and respiration use one channel each. Windows that fail signal-quality checks are excluded. Self-supervised pretraining uses, per modality, 200,000 training and 40,000 validation windows (295,512 training windows for ABP). Supervised training uses 150,000 windows with at least two modalities present, from 23 training patients, and selects checkpoints on 20,000 windows from 8 validation patients. Evaluation (Sec.~\ref{sec:exp-grounding}, App.~\ref{app:grounding}) uses the 26 test patients with processed windows.

\paragraph{Architecture and training.}
Table~\ref{tab:encoder} gives the architecture. Each modality encoder is pretrained separately. The four encoders, the fusion module and the heads are then fine-tuned jointly. Six models are trained, one per supervision source (rule or monitor labels) and random seed (three seeds, numbered 1-3 in App.~\ref{app:grounding}), with identical settings.

\begin{table}[h]
\centering
\caption{Waveform concept model.}
\label{tab:encoder}
\small
\begin{tabular}{@{}lp{0.72\linewidth}@{}}
\toprule
Component & Setting \\
\midrule
Stem & four blocks of Conv1d (kernels 7, 5, 5, 3; stride 2), normalization and Gaussian error linear unit (GELU); channels 64, 128, 256, 256; 118 tokens per window. BatchNorm, except GroupNorm with one group for ABP, which keeps the absolute pressure level. \\
Encoder & sinusoidal positions; Transformer encoder \citep{vaswani2017attention} with 4 pre-norm layers, 8 heads, width 256, feed-forward width 1,024, dropout 0.1; mean pooling to a 256-d embedding. \\
Pretraining & masked-token reconstruction \citep{he2022masked}: 40\% of stem tokens replaced by a learned mask token; mean squared error between predicted and actual stem features of the masked tokens. AdamW (learning rate $3\times10^{-4}$, weight decay $10^{-4}$), cosine schedule, 50 epochs, batch 4,096; Gaussian noise (s.d.\ 0.01) and amplitude scaling by a factor drawn from the uniform distribution $\mathcal{U}(0.9,1.1)$; checkpoint with the lowest validation loss. \\
Fusion & learned modality embeddings added to the four modality embeddings; 2 self-attention layers (4 heads) over the modality slots, with absent modalities masked from attention; layer normalization; mean over present modalities. \\
Heads & one linear layer and sigmoid per concept on the fused embedding (7 heads). \\
Fine-tuning & all modules jointly; 40 epochs, batch 1,024; AdamW with learning rate $3\times10^{-5}$ (encoders) and $3\times10^{-4}$ (fusion, heads), weight decay $10^{-4}$, cosine schedule, gradient-norm clipping at 1. Loss: binary cross-entropy per concept with positive-class weight $\min(20,n_{\mathrm{neg}}/n_{\mathrm{pos}})$, where $n_{\mathrm{neg}}$ and $n_{\mathrm{pos}}$ are the numbers of negative and positive training labels of the concept, averaged over concepts with a label. Checkpoint with the lowest validation loss on the model's own supervision source. \\
Threshold $\lambda_c$ & monitor-supervised heads: 0.5, fixed. Rule-supervised heads: best F1 score against rule labels on a 0.05-0.95 grid (step 0.05), fitted on validation windows for evaluation (App.~\ref{app:grounding}). Rule-supervised seed 1 also has thresholds fitted on training and validation windows (0.50, 0.45, 0.80, 0.55, 0.50, 0.35 and 0.90 for the first seven concepts of Table~\ref{tab:concept-vocab}); they define the activation episodes for the tier cut-offs and the MIMIC-IV-WDB graphs and are used in the detailed faithfulness criteria (App.~\ref{app:grounding-faith}). \\
\bottomrule
\end{tabular}
\end{table}

\paragraph{Rule labels.}
Rule-supervised heads are trained on labels computed from the same waveform window. Heart and pulse rate are $60f_s/k^\star$, where $f_s$ is the sampling rate and $k^\star$ is the lag (in samples) of the autocorrelation peak of the first channel in the range corresponding to 40-200\,bpm. Respiratory rate is estimated the same way in the range 4-40\,/min. A pressure concept is on if any 10-s sub-window mean (5-s step) of the arterial waveform, restored to mmHg, is below 65 (hypotension) or above 105\,mmHg (hypertension). A window with more than 20\% (rate) or 30\% (pressure) missing samples is labeled off.

\subsection{Monitor labels}
\label{app:concepts-labels}

Monitor-supervised heads are trained on the bedside monitor's own numerics, recorded in MIMIC-IV-WDB alongside the waveforms. These are heart rate from the ECG, pulse rate from the SpO$_2$ sensor, impedance respiratory rate, and arterial mean pressure (the invasive arterial-mean channel ABPm, or ARTm when ABPm is absent). Windows and numerics share the record clock (App.~\ref{app:data-clock}); lag scans against waveform-derived heart rate, MAP and respiratory rate are in App.~\ref{app:grounding-ref}. Samples outside physiological ranges are discarded: HR and pulse rate 20-250\,bpm, respiratory rate 2-80\,/min, arterial mean 20-200\,mmHg. The window value is the median of at least ten valid samples; otherwise it is missing. The label applies the threshold of Table~\ref{tab:concept-vocab}. A concept is labeled only in windows where its waveform modality is present, and windows without a label are ignored by the loss. Table~\ref{tab:label-agreement} compares monitor and rule labels on the training windows. The two sources agree least for the respiratory concepts, where the rule labels 45\% of windows as bradypnea and the monitor 1.3\%.

\begin{table}[h]
\centering
\caption{Monitor versus rule labels on the 150,000 supervised training windows: labeled windows, prevalence and agreement where both exist.}
\label{tab:label-agreement}
\small
\begin{tabular}{@{}lrccc@{}}
\toprule
Concept & Labeled (monitor) & Prevalence, monitor & Prevalence, rule & Agreement \\
\midrule
Tachycardia & 104{,}039 & 0.214 & 0.297 & 0.854 \\
Bradycardia & 104{,}039 & 0.004 & 0.086 & 0.918 \\
Tachycardia (PPG) & 145{,}390 & 0.457 & 0.222 & 0.760 \\
Tachypnea & 145{,}715 & 0.737 & 0.366 & 0.541 \\
Bradypnea & 145{,}715 & 0.013 & 0.455 & 0.550 \\
Hypotension & 17{,}657 & 0.216 & 0.359 & 0.858 \\
Hypertension & 17{,}657 & 0.011 & 0.025 & 0.989 \\
\bottomrule
\end{tabular}
\end{table}

The monitor labels are the output of the monitor's own algorithms. The pressure concepts also have a second reference, the intermittent non-invasive cuff pressure, which is used in App.~\ref{app:supp-pressure}.

\subsection{Persistence tiers}
\label{app:concepts-tiers}

An episode is a maximal run of consecutive windows with $S_{p,c}=1$. With a 15-s stride, an episode of $m$ windows has duration $15m$\,s. It is \emph{transient} if the duration is at most $\vartheta^{(1)}_c$, \emph{episodic} if it is at most $\vartheta^{(2)}_c$, and \emph{persistent} otherwise. The cut-offs are the 33rd and 67th percentiles of episode durations for concept $c$. They were fitted once on the activation episodes of rule-supervised seed 1 in the training and validation patients, with 1,871 (hypertension) to 112,618 (bradypnea) episodes per concept. The same cut-offs are used for concepts computed from monitor numerics. Concepts without calibration data (capnography and beat concepts) use 60\,s and 600\,s. Node $v=(c,\kappa)$ gathers all episodes of concept $c$ in tier $\kappa$. Each episode keeps an identifier, its start and end times, its start and end steps, and a flag that marks a valid onset (the episode starts after two observed steps with the concept off).

\subsection{Beat concepts}
\label{app:concepts-beat}

Beat concepts are defined on the same 30-s windows from the beat tables of App.~\ref{app:data-products}. A beat belongs to a window if its R peak lies in the window. For the rhythm concepts, the R-R interval that starts at that beat must also end inside the window and lie in 300-2000\,ms. A rhythm window is observed if at least 20 such intervals cover at least 24\,s. Table~\ref{tab:beat-concepts} gives the definitions. The two cohort-relative thresholds were fixed from epoch-level summaries before any window was evaluated: $T_{\mathrm{low}}=3.873$\,ms is the 10th percentile of the 5-min-epoch root mean square of successive R-R differences (RMSSD) over the quality-controlled cases, and $T_{\mathrm{pat}}=20.910$\,ms is the 90th percentile of within-case deviations of epoch median PAT from the case median. Beat concepts are signal concepts, not diagnoses. Each beat episode also stores the range of beat-table rows of its windows, the number of beats and the median of its defining statistic.

\begin{table}[h]
\centering
\caption{Beat concepts and their prevalence in quality-controlled surgical cases (prolonged PAT: cases passing PAT quality control).}
\label{tab:beat-concepts}
\scriptsize
\setlength{\tabcolsep}{3pt}
\begin{tabular}{@{}p{0.12\linewidth}p{0.40\linewidth}rp{0.07\linewidth}p{0.07\linewidth}p{0.08\linewidth}@{}}
\toprule
Concept & On in an observed window if and only if & Cases & Obs.\ windows & On / obs. & Cases with episode \\
\midrule
Irregular R-R & after trimming the 2 largest and 2 smallest R-R intervals: RMSSD / mean RRI $\ge0.10$, normalized Shannon entropy of the R-R histogram (16 bins) $\ge0.70$, and the turning-point count within $(2N-4)/3\pm2\sqrt{(16N-29)/90}$ & 2{,}954 & 98.1\% & 5.81\% & 99.5\% \\
Low HRV & RMSSD, after dropping the 2 largest absolute successive differences ($\ge15$ remaining), $<T_{\mathrm{low}}$ & 2{,}954 & 98.1\% & 20.32\% & 95.7\% \\
High PPV & window inside a controlled-ventilation epoch, irregular R-R observed and off, $\ge3$ qualifying breaths; median over breaths of $100(\mathrm{PP}_{\max}-\mathrm{PP}_{\min})/[\tfrac12(\mathrm{PP}_{\max}+\mathrm{PP}_{\min})]>13\%$ \citep{michard2000relation} & 2{,}914 & 61.1\% & 10.85\% & 92.6\% \\
Prolonged PAT & $\ge20$ valid beats with PAT; window median PAT minus the case baseline (median of window medians) $\ge T_{\mathrm{pat}}$ & 2{,}893 & 83.7\% & 11.75\% & 99.1\% \\
\bottomrule
\end{tabular}
\end{table}

Here $N$ is the number of R-R intervals, HRV is heart-rate variability, PPV is pulse pressure variation, $\mathrm{PP}_{\max}$ and $\mathrm{PP}_{\min}$ are the largest and smallest beat pulse pressures within one breath, and a breath qualifies if at least three of its beats, and at least 70\% of them, are valid. A controlled-ventilation epoch is a 5-min epoch with at least 30 breaths, a coefficient of variation of breath duration of at most 0.15, and a breath rate within 2\,/min of the ventilator's set rate. Prolonged PAT is defined within each case, so the device delay contained in PAT cancels. It indicates slower pulse arrival than the patient's usual value, for example with lower arterial pressure or a longer pre-ejection period. Because a PPG foot is matched only within $\pm\min(0.15~\mathrm{s},\,0.4\times\text{median RRI})$ of its expected time (App.~\ref{app:fast-pat}), a PAT rise beyond this margin leaves PAT missing, so the concept can be off or unobserved during the largest rises.

In an exploratory post hoc check against the preoperative 12-lead ECG interpretation, the case share of irregular-R-R windows separated atrial fibrillation (9 cases) from 300 randomly sampled sinus-rhythm cases with an AUROC of 0.93 (Hanley-McNeil 95\% CI [0.82, 1.00]; stratified case bootstrap [0.87, 0.98]). Restricted to windows whose beat-derived heart rate agrees with the monitor's within 5\,bpm and that lie more than 10\,min from the start and end of the recording, the AUROC was 0.92 (stratified case bootstrap [0.84, 0.99]). Restricted instead to windows with at least 80\% valid beats outside the same margins, it was 0.67 [0.48, 0.87]: beat validity excludes beats with an R-R change above 20\%, so under this gate the median share of irregular-R-R windows is 0 in both groups. In 300 random quality-controlled cases, 64.2\% of irregular-R-R windows lie within 10\,min of the start or end of a recording, against 4.8\% of observed windows where the concept is off. Irregular-R-R windows therefore mix rhythm irregularity with artifact near recording edges.

\section{ClosedLoopBench details}
\label{app:benchmark}

This appendix specifies the data, relations, estimators, negative controls, admission rule and metrics of ClosedLoopBench (Sec.~\ref{sec:benchmark}); it gives the full results at the primary 2\,s resolution, the pre-specified predictions, sensitivity analyses and the results at the secondary 30\,s resolution.

\subsection{Data and action events}
\label{app:benchmark-data}

\paragraph{Cohort.}
The benchmark uses all 3,442 cases of our VitalDB cohort (App.~\ref{app:data-cohorts}; \citealp{lee2022vitaldb}), recorded from 3,324 individuals; 96 individuals contribute 214 cases (at most 12 each). The case $p$ is the unit of resampling. No case is excluded.

\paragraph{Vitals.}
Each vital is placed on 2\,s bins $[2i,2i+2)$\,s of the case clock as the median of the valid monitor records in the bin, read at their native time stamps; App.~\ref{app:data-products} gives the valid ranges and the share of bins with a value.

\paragraph{Actions.}
TCI targets (propofol, $\mu$g\,mL$^{-1}$; remifentanil, ng\,mL$^{-1}$), the ventilator set rate (breaths\,min$^{-1}$) and the vasoactive infusion rates (mL\,h$^{-1}$) are taken from the recorded event stream with the value before and after each change. Set tidal volume (L), PEEP (cmH$_2$O) and FiO$_2$ (\%) are taken from the ventilator's set-point records; adjustments within 10\,s of the previous one are merged into one step, and the first record of a case is its initial value, not an event. Volatile exposure steps are changes of at least 0.2\,MAC in the 30\,s rolling median of the measured end-tidal MAC (levels below 0.2\,MAC count as off; the detector resets after a 120\,s gap). A change from or to zero of a TCI target, the ventilator rate, the tidal volume or the MAC level is not a titration event; the VARX receives these changes through a separate channel. After an action's recorded track ends, its dose is treated as missing. Action time stamps agree with the ventilator's own set-point records without a shift (App.~\ref{app:data-clock}).

\paragraph{Event counts.}
Table~\ref{tab:data-actions} counts all recorded events, in both directions and including starts and stops (171,758 in total). Table~\ref{tab:bench-relations} gives the titrations in each relation's direction.

\subsection{Relations}
\label{app:benchmark-relations}

Table~\ref{tab:bench-relations} lists the 29 relations. Estimates are reported in coefficient space: a response relates the dose change to the vital change, and a policy relates the vital's 2-min pre-action trend to the action; $\sigma_r$ is the expected sign of $\partial z/\partial x$ and is the same for the increase and decrease relations of one pair. The domains are gas exchange (R01, R02, R13-R15), drug (R03-R12, R19), volatile exposure (R17, R18), cardiopulmonary (R16) and policy (P01-P10). For FiO$_2{\to}$SpO$_2$ (R15) the event designs use only events and control anchors whose last 30\,s SpO$_2$ level is below 97\%, since higher saturations lie on the plateau of the oxyhemoglobin dissociation curve; the lag-based estimators are unrestricted.

\begin{table}[h]
\caption{The 29 ClosedLoopBench relations (R: response; P: policy). Propofol and remifentanil denote TCI targets. $d$: event direction; $\sigma_r$: expected sign; $h$: horizon (policies use the 2-min pre-action trend); Resp.\ physics: respiratory physics; Titrations: events in direction $d$ (cases) before the pre-window filter.}
\label{tab:bench-relations}
\centering
\footnotesize
\setlength{\tabcolsep}{3.5pt}
\begin{tabular}{@{}lllccclr@{}}
\toprule
ID & Cause $x$ & Effect $z$ & $d$ & $\sigma_r$ & $h$ & Basis & Titrations (cases) \\
\midrule
R01 & Set rate & EtCO$_2$ & $\uparrow$ & $-$ & 60\,s & Resp.\ physics & 11,807 (3,193) \\
R02 & Set rate & EtCO$_2$ & $\downarrow$ & $-$ & 60\,s & Resp.\ physics & 9,042 (2,944) \\
R03 & Propofol & MAP & $\uparrow$ & $-$ & 120\,s & Pharmacology & 8,233 (1,664) \\
R04 & Propofol & MAP & $\downarrow$ & $-$ & 120\,s & Pharmacology & 10,722 (1,823) \\
R05 & Remifentanil & HR & $\uparrow$ & $-$ & 120\,s & Pharmacology & 22,556 (2,781) \\
R06 & Remifentanil & MAP & $\uparrow$ & $-$ & 120\,s & Pharmacology & 22,556 (2,781) \\
R07 & Remifentanil & HR & $\downarrow$ & $-$ & 120\,s & Pharmacology & 21,118 (2,826) \\
R08 & Remifentanil & MAP & $\downarrow$ & $-$ & 120\,s & Pharmacology & 21,118 (2,826) \\
R09 & Phenylephrine & MAP & $\uparrow$ & $+$ & 120\,s & Pharmacology & 501 (103) \\
R10 & Phenylephrine & HR & $\uparrow$ & $-$ & 120\,s & Baroreflex & 501 (103) \\
R11 & Norepinephrine & MAP & $\uparrow$ & $+$ & 120\,s & Pharmacology & 376 (61) \\
R12 & Nitroglycerin & MAP & $\uparrow$ & $-$ & 120\,s & Pharmacology & 78 (30) \\
R13 & Set tidal volume & EtCO$_2$ & $\uparrow$ & $-$ & 60\,s & Resp.\ physics & 5,489 (2,363) \\
R14 & Set tidal volume & EtCO$_2$ & $\downarrow$ & $-$ & 60\,s & Resp.\ physics & 5,216 (2,338) \\
R15 & FiO$_2$ & SpO$_2$ & $\uparrow$ & $+$ & 120\,s & Resp.\ physics & 5,454 (3,346) \\
R16 & PEEP & MAP & $\uparrow$ & $-$ & 60\,s & Physiology & 946 (679) \\
R17 & MAC step & MAP & $\uparrow$ & $-$ & 120\,s & Pharmacology & 16,227 (1,906) \\
R18 & MAC step & BIS & $\uparrow$ & $-$ & 120\,s & Pharmacology & 16,227 (1,906) \\
R19 & Propofol & BIS & $\uparrow$ & $-$ & 120\,s & Pharmacology & 8,233 (1,664) \\
P01 & MAP trend & Remifentanil & $\uparrow$ & $+$ & trend & Practice & 22,556 (2,781) \\
P02 & HR trend & Remifentanil & $\uparrow$ & $+$ & trend & Practice & 22,556 (2,781) \\
P03 & MAP trend & Phenylephrine & $\uparrow$ & $-$ & trend & Practice & 501 (103) \\
P04 & MAP trend & Propofol & $\downarrow$ & $+$ & trend & Practice & 10,722 (1,823) \\
P05 & EtCO$_2$ trend & Set rate & $\uparrow$ & $+$ & trend & Practice & 11,807 (3,193) \\
P06 & MAP trend & Norepinephrine & $\uparrow$ & $-$ & trend & Practice & 376 (61) \\
P07 & SpO$_2$ trend & FiO$_2$ & $\uparrow$ & $-$ & trend & Practice & 5,454 (3,346) \\
P08 & EtCO$_2$ trend & Set tidal volume & $\uparrow$ & $+$ & trend & Practice & 5,489 (2,363) \\
P09 & BIS trend & Propofol & $\uparrow$ & $+$ & trend & Practice & 8,233 (1,664) \\
P10 & BIS trend & MAC step & $\uparrow$ & $+$ & trend & Practice & 16,227 (1,906) \\
\bottomrule
\end{tabular}
\end{table}

The \emph{Basis} column gives the domain knowledge that fixes each expected sign: respiratory physics for ventilation and oxygenation \citep{petersson2014gas}; drug pharmacology for propofol \citep{kazama1999comparison}, remifentanil \citep{komatsu2007remifentanil}, phenylephrine and norepinephrine (including the baroreflex response to phenylephrine; \citealp{ngankee2015randomized}), nitroglycerin \citep{divakaran2017role} and volatile agents \citep{ebert1995cardiovascular,katoh1998electroencephalographic}; cardiopulmonary physiology for PEEP \citep{luecke2005clinical}; and documented anesthetic practice for policies, in which anesthetics are titrated to BIS, arterial pressure and heart rate \citep{lewis2019bispectral}, vasopressors to arterial pressure \citep{ngankee2015randomized}, and ventilation and FiO$_2$ to EtCO$_2$ and SpO$_2$, respectively \citep{hemmes2011rationale}.
The real-data score is therefore concordance with these pre-specified sign expectations, not agreement with causal ground truth, and the bases differ in strength: gas-exchange physics is the strongest; pharmacological and physiological signs depend on context, dose and timing; and policy signs encode documented practice, not an observed decision rule. The cited sources establish the direction of each effect family in their own populations and settings (e.g., an obstetric spinal-anesthesia vasopressor trial and a ventilation trial protocol), not at the 60\,s or 2-min horizons in this cohort, and volatile steps are derived from the end-tidal MAC (App.~\ref{app:benchmark-data}), so their timing differs from that of a commanded vaporizer setting.

\subsection{Estimators}
\label{app:benchmark-estimators}

\paragraph{Notation.}
For relation $r$ with action $a$, vital $j$ and event direction $d_r\in\{+1,-1\}$ (${\uparrow}=+1$, ${\downarrow}=-1$), write $\bar y[t_1,t_2)$ for the mean of the non-missing values of $y_{p,j}$ with time in $[t_1,t_2)$ (the case index $p$ is dropped when clear). Let $L(t)=\bar y[t-30\,\mathrm{s},t)$ be the last level, $T(t)=L(t)-\bar y[t-150\,\mathrm{s},t-120\,\mathrm{s})$ the 2-min trend and $Y_h(t)=\bar y[t+h-3\,\mathrm{s},t+h+3\,\mathrm{s}]$ the level at horizon $h$. The eligible events $\mathcal{E}^r_p\subseteq\mathcal{E}_{p,a}$ are titration steps in direction $d_r$ with no other change of $u_{p,a}$ in $[\tau-150\,\mathrm{s},\tau)$. Events are analyzed as intention-to-treat: later changes of the action do not remove an event. Control anchors $g\in\mathcal{G}^r_p$ lie on a 10\,s grid and have no change of $u_{p,a}$ in $[g-150\,\mathrm{s},g+5\,\mathrm{s}]$; the pool uses information up to a 5\,s decision interval and does not condition on the clinician's later inaction (Proposition~\ref{prop:pools}).

\paragraph{Event designs.}
Each design forms a contrast $\psi_\tau$ for every eligible event, and the pooled estimate is the event-weighted mean $\hat\theta_r=\sum_p\sum_{\tau\in\mathcal{E}^r_p}\psi_\tau\big/\sum_p|\mathcal{E}^r_p|$ over events with a finite contrast.
\begin{itemize}[leftmargin=*,itemsep=1pt,topsep=2pt]
\item \emph{Naive event contrast}: $\psi_\tau=d_r\,\big(Y_h(\tau)-\bar y[\tau-120\,\mathrm{s},\tau)\big)$.
\item \emph{Naive level contrast} (responses only): $\psi_\tau=d_r\,\big(Y_h(\tau)-L(\tau)\big)$.
\item \emph{Matched event design}: $\psi_\tau=d_r\big[\big(Y_h(\tau)-L(\tau)\big)-\frac15\sum_{g\in N(\tau)}\big(Y_h(g)-L(g)\big)-(M_\tau-\bar M_{N(\tau)})^\top\beta_M\big]$, where $N(\tau)$ holds the 5 anchors of the same case nearest to $\tau$ in the features $M=(L,T,u_{p,a})$, standardized over the case's anchors, $\bar M_{N(\tau)}$ is their mean feature vector, and $\beta_M$ is the least-squares coefficient of the control outcome on $M$ over the case's anchors (regression bias correction; \citealp{abadie2011bias}). At least 20 anchors are required.
\item \emph{Interrupted-trend (kink) design}: a line $\eta_\tau+\nu_\tau(t-\tau)$ with intercept $\eta_\tau$ and slope $\nu_\tau$ is fitted to $y$ on $[\tau-60\,\mathrm{s},\tau)$ (at least 21 values) and extrapolated, a design related to regression discontinuity in time \citep{hausman2018regression}; the deviation $D_h(\tau)=Y_h(\tau)-(\eta_\tau+\nu_\tau h)$ replaces $Y_h-L$ in the matched contrast, with features $(\eta,\nu,u_{p,a})$.
\item \emph{Sequential-trial contrast} \citep{hernan2008observational,danaei2013observational}: each event is compared with its 10 nearest anchors from \emph{other} cases in a pooled bank (at most 10 evenly spaced anchors per case and action). The pre-decision state has 19 features: for each vital, $L$, the 10\,s level $\bar y[t-10\,\mathrm{s},t)$ and $T$; the dose in force; elapsed case time; and the numbers of increases and decreases of the other actions in $[t-60\,\mathrm{s},t)$. Features are standardized on the bank, a missing feature is set to the bank mean, and the bias correction uses $\beta_M$ fitted on the pooled bank; at least 5 matched controls with an outcome are required. The outcome is $Y_h-L$.
\end{itemize}
For policies, the outcome is the pre-event trend: the naive contrast uses $d_r\,T(\tau)$; the matched and interrupted-trend designs use $T$ and $120\,\mathrm{s}\times\nu_\tau$, matched on the dose and the elapsed fraction of the case without bias correction; the sequential-trial contrast uses $T$ matched on the dose, elapsed time and co-intervention counts.

\paragraph{Lag-based estimators.}
\begin{itemize}[leftmargin=*,itemsep=1pt,topsep=2pt,beginpenalty=10000]
\item \emph{Lagged cross-correlation.} Per case, $\rho_p(\ell)=\sum_t \mathring u_{p,a}(t)\,\mathring y_{p,j}(t+\ell)\big/\big(\|\mathring u_{p,a}\|\,\|\mathring y_{p,j}\|\big)$ for lags $|\ell|\le1{,}800$\,s, where $\mathring{\cdot}$ denotes the centered series with missing values set to 0; both series need at least 50\% coverage. The curves are averaged over the cases with at least 3 eligible events and at least 50\% coverage of the vital, giving $\bar\rho(\ell)$. The \emph{positive-peak} rule returns $\max_\ell\bar\rho(\ell)$; it is the lagged cross-correlation rule of correlational KG construction (CKG), applied here to dose and vital series. The \emph{signed-peak} rule returns $\bar\rho(\ell^\star)$ with $\ell^\star=\arg\max_\ell|\bar\rho(\ell)|$. The \emph{direction-aware} rule restricts $\ell^\star$ to lags at which the cause leads ($\ell\ge\Delta$ for responses, $\ell\le-\Delta$ for policies). The positive- and signed-peak rules return the same value for a response and a policy on the same action-vital pair and direction, so they have 19 distinct estimates for 29 relations.
\item \emph{Pairwise Granger} \citep{granger1969investigating}. A bivariate vector autoregression of order 30, VAR(30) (60\,s of lags), in the levels of $(u_{p,a},y_{p,j})$, pooled over cases with within-case demeaning (case fixed effects); rows with missing values are dropped and a case needs at least $5\times61$ rows. $\hat\theta_r$ is the sum of the 30 lag coefficients of the cause in the effect's equation. The increase and decrease relations of a pair share one estimate (24 distinct estimates).
\item \emph{PCMCI+} \citep{runge2019detecting,runge2020discovering}. Per case, with partial-correlation tests and significance level 0.01 in the skeleton phase, on 30\,s block means of the dose series (actions with at least one titration) and the vitals, lags 1-8 (up to 240\,s). For each (cause, effect) pair we take the lagged link with the smallest $p$-value; $\hat\theta_r$ is the mean of its partial correlation over eligible cases. PCMCI+ runs on a seeded subsample of 60 cases (cases with vasoactive infusions first), the same in every arm.
\item \emph{VAR-LiNGAM} (vector-autoregressive linear non-Gaussian acyclic model; \citealp{hyvarinen2010estimation}). A pooled fixed-effects VAR(15) (30\,s of lags) of dose and vital levels within five variable sets: \{propofol, remifentanil, set rate, MAP, HR, EtCO$_2$\}, \{phenylephrine, norepinephrine, nitroglycerin, MAP, HR\}, \{set tidal volume, FiO$_2$, PEEP, EtCO$_2$, SpO$_2$\}, \{PEEP, MAP, HR\} and \{MAC, propofol, BIS, MAP\}; a relation is scored in the first set that contains its cause and effect. The DirectLiNGAM algorithm \citep{shimizu2011directlingam} of the \texttt{lingam} package, applied to the pooled VAR residuals (300 per case, at most 20,000), gives the instantaneous matrix $B_0$, the lagged matrices are $B_\ell=(I-B_0)A_\ell$ for the VAR coefficient matrices $A_\ell$ ($I$ the identity), and $\hat\theta_r=\sum_{\ell=0}^{15}B_\ell[z,x]$. It uses 50 bootstrap replicates; increase and decrease relations share one estimate (24 distinct estimates).
\item \emph{VARX}. For each vital, a pooled fixed-effects regression of the first difference $y_{p,j}(t)-y_{p,j}(t-\Delta)$ on 10 lags of the changes of all five vitals (a missing lagged change enters as 0) and, for every action, the signed increments $u_{p,a}(\tau)-u_{p,a}(\tau^-)$ ($\tau^-$: just before the change) split into an increase channel, a decrease channel and a channel for non-titration changes, each summed in lag bins $(0,10]$, $(10,30]$, $(30,60]$, $(60,90]$, $(90,120]$, $(120,180]$ and $(180,240]$\,s. Each equation uses the rows where its outcome is observed. For a response, $\hat\theta_r$ is the cumulative change of vital $j$ at horizon $h$ after a unit increment on the relation's channel, propagated through the fitted system. For a policy, the direction-$d_r$ increment series of action $a$ is regressed on the vitals' 2-min trends and levels at $t-\Delta$ and on the action's own increments in bins up to 120\,s; $\hat\theta_r$ is the coefficient of the vital's trend. Our typed pooled estimator uses these estimates and therefore equals VARX under the admission rule.
\end{itemize}

\paragraph{Uncertainty.}
All estimators use the same 200 case-bootstrap replicates $b=1,\dots,200$ (multinomial case weights, fixed seed; VAR-LiNGAM uses the first 50), and the same weight matrix is applied to every arm, so real and null estimates are paired \citep{efron1994introduction}. Event designs resample case sums of contrasts; lag-based estimators resample case curves (cross-correlation), case estimates (PCMCI+) or case sufficient statistics (Granger, VAR-LiNGAM, VARX). A single-arm call uses the percentile 95\% CI and requires at least 5 contributing cases.

\subsection{Negative controls, admission rule and metrics}
\label{app:benchmark-admission}

\paragraph{Negative controls.}
Negative controls are data in which the relation cannot exist; they detect systematic error \citep{lipsitch2010negative} and calibrate $p$-values for it \citep{schuemie2014interpreting}. For the cross-patient transplant $\tilde u^{(k)}$, each action's users are permuted by a seeded random derangement (no case keeps its own stream); case $p$ receives its partner's event stream for that action, placed on $p$'s clock by elapsed time from case start and truncated to $p$'s duration. Each action has its own derangement, and the $K=8$ draws use fixed seeds. The random-time null places each case's own steps, with their values and order, at sorted uniform random times over the case, with a seed per case.

\paragraph{Admission rule.}
For an estimator and relation $r$, with real-data estimate $\hat\theta_r(u)$ and transplant estimates $\hat\theta_r(\tilde u^{(k)})$,
\begin{equation}
\label{eq:bench-calib}
\tilde\theta_r=\hat\theta_r(u)-\frac1K\sum_{k=1}^{K}\hat\theta_r(\tilde u^{(k)}),\qquad
\widehat{\mathrm{se}}_r^2=\widehat{\mathrm{Var}}_b\Big[\hat\theta^{(b)}_r(u)-\frac1K\sum_{k=1}^{K}\hat\theta^{(b)}_r(\tilde u^{(k)})\Big]+\frac1K\,\widehat{\mathrm{Var}}_k\big[\hat\theta_r(\tilde u^{(k)})\big],
\end{equation}
where $\hat\theta^{(b)}_r$ is the estimate on bootstrap replicate $b$ and $\widehat{\mathrm{Var}}_b$, $\widehat{\mathrm{Var}}_k$ are sample variances over replicates and over draws. The first term is patient sampling; the second is the choice of derangement. The CI is $\tilde\theta_r\pm1.96\,\widehat{\mathrm{se}}_r$ and the two-sided $p$-value is $p_r=2\Phi(-|\tilde\theta_r|/\widehat{\mathrm{se}}_r)$, with $\Phi$ the standard normal distribution function. The relation is admitted if (i) the CI excludes 0; (ii) the Benjamini-Yekutieli-adjusted $p_r$ over the estimator's relations (29; 19 for the naive level contrast) is at most $\alpha=0.05$ \citep{benjamini2001control}; (iii) its stability, the share of draws $k$ whose paired percentile 95\% interval of $\hat\theta^{(b)}_r(u)-\hat\theta^{(b)}_r(\tilde u^{(k)})$ excludes 0 with the sign of $\tilde\theta_r$, is at least 0.75; and (iv) at least 5 cases contribute. An admitted relation is concordant (admitted with the expected sign) if $\operatorname{sign}\tilde\theta_r=\sigma_r$ and wrong-signed otherwise.

\paragraph{Metrics.}
\emph{Uncalibrated null discoveries}: on each transplant draw, the relations whose single-arm CI excludes 0; we report the mean over the 8 draws, over relations and over distinct estimates (relations with identical estimate and CI counted once). A calibrated test at level 0.05 would flag about 5\% of the distinct estimates. \emph{Held-out false admissions}: each draw $k$ in turn is the input, calibrated against the other 7 draws; we report the mean number of admissions per draw. \emph{Random-time false admissions}: the random-time arm is the input, calibrated against the 8 transplant draws.

\subsection{Results at the primary \texorpdfstring{2\,s}{2 s} resolution}
\label{app:benchmark-results}

Table~\ref{tab:bench-leaderboard} gives every count for each estimator, Table~\ref{tab:bench-calls} the call for every estimator and relation, and Table~\ref{tab:bench-estimates} the calibrated estimates cited in Sec.~\ref{sec:exp-benchmark}. Estimators differ in native estimand, resolution and cohort (PCMCI+ runs on 30\,s block means of 60 cases), so these tables compare admission behavior under a common calibration rule, not a like-for-like ranking. With a held-out transplant draw as input, every estimator admitted 0 relations in each of the 8 draws.

\begin{table}[h]
\caption{Counts per estimator at 2\,s. Rel.: relations (distinct estimates). Real: single-arm calls on real data, expected/wrong sign. Transplant: uncalibrated discoveries per draw, mean [range] over relations and mean over distinct estimates. RT: uncalibrated discoveries on the random-time arm (distinct). Cal.: calibrated admissions, expected/wrong sign. FA: false admissions (held-out draw mean; random-time).}
\label{tab:bench-leaderboard}
\centering
\footnotesize
\setlength{\tabcolsep}{2.4pt}
\begin{tabular}{@{}lcccccccc@{}}
\toprule
 & & \multicolumn{4}{c}{Uncalibrated} & & \multicolumn{2}{c}{FA} \\
\cmidrule(lr){3-6}\cmidrule(lr){8-9}
Estimator & Rel. & Real & Transplant & Distinct & RT & Cal. & Held-out & RT \\
\midrule
Naive event contrast & 29 (29) & 17/5 & 18.25 [16-20] & 18.25 & 16 (16) & 16/5 & 0 & 7 \\
Naive level contrast & 19 (19) & 13/2 & 6.38 [5-8] & 6.38 & 5 (5) & 10/2 & 0 & 2 \\
Positive-peak cross-corr. & 29 (19) & 10/19 & 17.62 [15-19] & 11.00 & 18 (13) & 6/11 & 0 & 13 \\
Signed-peak cross-corr. & 29 (19) & 13/10 & 17.75 [14-22] & 11.38 & 17 (11) & 10/7 & 0 & 13 \\
Direction-aware cross-corr. & 29 (29) & 17/6 & 18.12 [16-23] & 18.12 & 17 (17) & 14/6 & 0 & 10 \\
Pairwise Granger & 29 (24) & 22/5 & 19.75 [17-22] & 14.75 & 19 (14) & 20/5 & 0 & 16 \\
PCMCI+ & 29 (28-29) & 11/0 & 3.88 [0-7] & 3.88 & 5 (5) & 3/0 & 0 & 0 \\
VAR-LiNGAM & 29 (24) & 25/1 & 16.00 [14-17] & 12.12 & 15 (11) & 22/1 & 0 & 12 \\
VARX & 29 (29) & 18/5 & 7.88 [6-11] & 7.88 & 6 (6) & 16/4 & 0 & 4 \\
Matched event design & 29 (29) & 18/3 & 2.75 [1-6] & 2.75 & 4 (4) & 14/2 & 0 & 0 \\
Interrupted trend (kink) & 29 (29) & 18/3 & 3.62 [1-6] & 3.62 & 5 (5) & 15/2 & 0 & 0 \\
Sequential-trial contrast & 29 (29) & 19/4 & 9.25 [4-14] & 9.25 & 8 (8) & 16/1 & 0 & 4 \\
\bottomrule
\end{tabular}
\end{table}

\begin{table}[h]
\caption{Calibrated calls at 2\,s. \textbf{R}/\textbf{X}: admitted with the expected/wrong sign; r/x: not admitted, estimate with the expected/wrong sign; n/a: not estimated. NE/NL: naive event/level contrast; XP/XS/XD: positive-peak/signed-peak/direction-aware cross-correlation; GR: Granger; PC: PCMCI+; VL: VAR-LiNGAM; VX: VARX; MA: matched; KI: interrupted trend; ST: sequential trial. Last columns: estimators admitting with the expected/wrong sign.}
\label{tab:bench-calls}
\centering
\footnotesize
\setlength{\tabcolsep}{3.2pt}
\begin{tabular}{@{}lcccccccccccc|cc@{}}
\toprule
 & NE & NL & XP & XS & XD & GR & PC & VL & VX & MA & KI & ST & R & X \\
\midrule
R01 & \textbf{R} & \textbf{R} & \textbf{X} & \textbf{X} & \textbf{X} & \textbf{X} & r & \textbf{R} & \textbf{R} & \textbf{R} & \textbf{R} & \textbf{R} & 7 & 4 \\
R02 & \textbf{R} & \textbf{R} & \textbf{X} & \textbf{X} & \textbf{X} & \textbf{X} & r & \textbf{R} & \textbf{R} & \textbf{R} & \textbf{R} & \textbf{R} & 7 & 4 \\
R03 & \textbf{R} & \textbf{R} & \textbf{X} & \textbf{R} & \textbf{R} & \textbf{R} & r & \textbf{R} & \textbf{R} & \textbf{R} & \textbf{R} & \textbf{R} & 10 & 1 \\
R04 & \textbf{X} & x & \textbf{X} & \textbf{R} & \textbf{R} & \textbf{R} & r & \textbf{R} & r & r & x & \textbf{R} & 5 & 2 \\
R05 & x & \textbf{R} & \textbf{X} & x & r & \textbf{R} & x & \textbf{R} & \textbf{R} & r & r & r & 4 & 1 \\
R06 & \textbf{X} & \textbf{R} & \textbf{X} & \textbf{R} & \textbf{R} & \textbf{R} & r & \textbf{R} & \textbf{R} & r & \textbf{R} & x & 7 & 2 \\
R07 & \textbf{X} & \textbf{X} & \textbf{X} & x & x & \textbf{R} & x & \textbf{R} & \textbf{X} & x & x & x & 2 & 4 \\
R08 & \textbf{X} & \textbf{X} & \textbf{X} & \textbf{R} & \textbf{R} & \textbf{R} & r & \textbf{R} & r & \textbf{R} & \textbf{R} & r & 6 & 3 \\
R09 & r & \textbf{R} & \textbf{R} & r & \textbf{R} & \textbf{R} & x & \textbf{R} & r & r & r & r & 5 & 0 \\
R10 & r & r & x & x & x & x & x & x & x & x & r & r & 0 & 0 \\
R11 & x & r & r & x & r & r & r & r & r & r & r & x & 0 & 0 \\
R12 & r & r & x & x & x & r & r & r & x & r & r & r & 0 & 0 \\
R13 & \textbf{R} & r & x & \textbf{R} & \textbf{R} & \textbf{R} & r & \textbf{R} & \textbf{X} & \textbf{R} & \textbf{R} & \textbf{R} & 8 & 1 \\
R14 & \textbf{R} & \textbf{R} & \textbf{X} & \textbf{R} & \textbf{R} & \textbf{R} & r & \textbf{R} & \textbf{R} & \textbf{R} & \textbf{R} & \textbf{R} & 10 & 1 \\
R15 & r & r & \textbf{R} & x & x & \textbf{X} & r & \textbf{R} & \textbf{R} & r & r & r & 3 & 1 \\
R16 & r & r & r & r & r & r & x & r & r & r & r & x & 0 & 0 \\
R17 & \textbf{R} & \textbf{R} & x & \textbf{R} & \textbf{R} & \textbf{R} & \textbf{R} & \textbf{R} & \textbf{R} & \textbf{R} & \textbf{R} & \textbf{R} & 11 & 0 \\
R18 & \textbf{R} & \textbf{R} & r & \textbf{R} & \textbf{R} & \textbf{R} & r & \textbf{R} & \textbf{R} & \textbf{R} & \textbf{R} & \textbf{R} & 10 & 0 \\
R19 & \textbf{R} & \textbf{R} & r & \textbf{R} & \textbf{R} & \textbf{R} & r & \textbf{R} & \textbf{R} & \textbf{R} & \textbf{R} & \textbf{R} & 10 & 0 \\
\midrule
P01 & \textbf{R} & n/a & \textbf{R} & \textbf{X} & \textbf{R} & \textbf{R} & \textbf{R} & \textbf{R} & \textbf{R} & \textbf{R} & \textbf{R} & \textbf{R} & 10 & 1 \\
P02 & \textbf{R} & n/a & \textbf{R} & r & r & \textbf{R} & r & \textbf{R} & \textbf{R} & \textbf{R} & \textbf{R} & \textbf{R} & 8 & 0 \\
P03 & \textbf{R} & n/a & \textbf{X} & x & \textbf{R} & \textbf{R} & \textbf{R} & \textbf{R} & x & r & r & r & 5 & 1 \\
P04 & \textbf{R} & n/a & \textbf{R} & \textbf{X} & \textbf{X} & \textbf{R} & x & r & \textbf{R} & r & r & \textbf{R} & 5 & 2 \\
P05 & \textbf{R} & n/a & \textbf{R} & \textbf{R} & \textbf{R} & \textbf{X} & r & \textbf{R} & \textbf{R} & \textbf{R} & \textbf{R} & \textbf{R} & 9 & 1 \\
P06 & \textbf{R} & n/a & x & r & \textbf{R} & \textbf{R} & r & \textbf{R} & x & r & r & \textbf{R} & 5 & 0 \\
P07 & \textbf{R} & n/a & \textbf{X} & r & r & \textbf{R} & r & \textbf{R} & \textbf{R} & \textbf{R} & \textbf{R} & \textbf{R} & 7 & 1 \\
P08 & x & n/a & r & \textbf{X} & \textbf{X} & \textbf{R} & r & x & \textbf{X} & \textbf{X} & \textbf{X} & \textbf{X} & 1 & 6 \\
P09 & \textbf{R} & n/a & x & \textbf{X} & \textbf{X} & \textbf{R} & x & \textbf{R} & \textbf{R} & \textbf{R} & \textbf{R} & \textbf{R} & 7 & 2 \\
P10 & \textbf{X} & n/a & x & \textbf{X} & \textbf{X} & \textbf{X} & r & \textbf{X} & \textbf{X} & \textbf{X} & \textbf{X} & x & 0 & 8 \\
\bottomrule
\end{tabular}
\end{table}

\begin{table}[h]
\caption{Selected calibrated estimates at 2\,s (IDs as in Table~\ref{tab:bench-relations}). Event designs: direction-signed change of the vital per event (mmHg, bpm, BIS points or SpO$_2$ percentage points; policies: change over the 2 min before the action). Cross-correlation: correlation; Granger: sum of lag coefficients; VARX: change of the vital at $h$ per unit action change. Stability: draws whose paired interval excludes 0 with the sign of $\tilde\theta_r$, of 8.}
\label{tab:bench-estimates}
\centering
\footnotesize
\setlength{\tabcolsep}{4pt}
\begin{tabular}{@{}llrccl@{}}
\toprule
Relation & Estimator & $\tilde\theta_r$ [95\% CI] & Stability & Cases & Call \\
\midrule
R01 & Matched & $-0.69$ [$-0.73$, $-0.64$] & 8/8 & 3,099 & \textbf{R} \\
 & Sequential trial & $-0.81$ [$-0.86$, $-0.76$] & 8/8 & 3,099 & \textbf{R} \\
 & Cross-corr.\ (3 rules) & $0.053$ [$0.044$, $0.063$] & 8/8 & 1,956 & \textbf{X} \\
 & Granger & $0.0019$ [$0.0013$, $0.0025$] & 8/8 & 3,295 & \textbf{X} \\
R02 & Matched & $-0.76$ [$-0.82$, $-0.70$] & 8/8 & 2,838 & \textbf{R} \\
 & Sequential trial & $-1.07$ [$-1.13$, $-1.01$] & 8/8 & 2,838 & \textbf{R} \\
R13 & VARX & $4.25$ [$1.39$, $7.11$] & 8/8 & 3,426 & \textbf{X} \\
R15 & Naive level & $1.04$ [$0.51$, $1.56$] & 5/8 & 510 & r \\
R03 & Interrupted trend & $-0.92$ [$-1.28$, $-0.57$] & 8/8 & 1,519 & \textbf{R} \\
R19 & Interrupted trend & $-2.99$ [$-3.28$, $-2.70$] & 8/8 & 1,460 & \textbf{R} \\
R07 & Naive event & $0.88$ [$0.74$, $1.01$] & 8/8 & 2,822 & \textbf{X} \\
 & Naive level & $0.73$ [$0.60$, $0.86$] & 8/8 & 2,822 & \textbf{X} \\
 & Positive-peak cross-corr. & $0.057$ [$0.047$, $0.067$] & 8/8 & 2,452 & \textbf{X} \\
 & VARX & $0.22$ [$0.12$, $0.32$] & 8/8 & 3,442 & \textbf{X} \\
 & Matched & $0.08$ [$-0.03$, $0.19$] & 1/8 & 2,803 & x \\
 & Interrupted trend & $0.08$ [$-0.05$, $0.21$] & 1/8 & 2,801 & x \\
 & Sequential trial & $0.05$ [$-0.10$, $0.19$] & 1/8 & 2,803 & x \\
R08 & Matched & $-0.57$ [$-0.78$, $-0.35$] & 8/8 & 2,658 & \textbf{R} \\
 & Interrupted trend & $-0.53$ [$-0.77$, $-0.30$] & 8/8 & 2,664 & \textbf{R} \\
P01 & Matched & $1.96$ [$1.77$, $2.14$] & 8/8 & 2,637 & \textbf{R} \\
P08 & Matched & $-0.40$ [$-0.54$, $-0.27$] & 8/8 & 2,253 & \textbf{X} \\
P10 & Interrupted trend & $-1.03$ [$-1.49$, $-0.56$] & 8/8 & 1,387 & \textbf{X} \\
\bottomrule
\end{tabular}
\end{table}

\paragraph{Gas exchange.}
The three design-based estimators admit all four ventilation$\to$EtCO$_2$ relations (R01, R02, R13, R14) with the physical sign; this was not a pre-specified prediction. All 11 wrong-signed gas-exchange admissions come from the three cross-correlation rules, Granger and VARX (3, 2, 2, 3 and 1 for positive-peak, signed-peak, direction-aware, Granger and VARX); they correspond to 6 distinct estimates. The cross-correlation rules and Granger give the set-rate relations a positive sign, the sign of the policy P05 (EtCO$_2$ rising$\to$set-rate increase). FiO$_2{\to}$SpO$_2$ (R15) has the physical sign in all five event designs (the two naive contrasts and the three design-based estimators), which use only events with pre-event SpO$_2<97\%$ (510 contributing cases for the naive level contrast), and none of them admits it; among the unrestricted lag-based estimators, VAR-LiNGAM, VARX and the positive-peak rule admit it with the expected sign and Granger with the wrong sign.

\paragraph{Drug and cardiopulmonary responses.}
Granger and VAR-LiNGAM return one estimate for remifentanil increases and decreases (R05 and R07 share $-0.0061$ [$-0.0079$, $-0.0042$] for Granger), so their R07 admissions with the expected sign are the R05 estimate. The vasoactive relations R10-R12 (30-103 cases) are admitted by no estimator. PEEP$\to$MAP (R16; 946 increases in 679 cases) has the expected sign for 10 of 12 estimators and is admitted by none.

\paragraph{Policies on tidal-volume and derived volatile steps.}
P08 (EtCO$_2$ rising$\to$set tidal volume increase) is admitted with the opposite sign by 6 estimators: set tidal-volume increases in these data are preceded by falling EtCO$_2$ relative to matched control times (matched design $-0.40$\,mmHg over 2\,min [$-0.54$, $-0.27$]). P10 (BIS rising$\to$MAC step increase) is admitted with the opposite sign by 8 estimators. Derived MAC steps are detected from the measured end-tidal concentration, which follows the vaporizer setting with a wash-in delay; Table~\ref{tab:bench-mac} gives P10 for events after the first 30\,min and without a preceding MAC increase: the matched and interrupted-trend estimates keep the opposite sign, whereas the naive event contrast changes sign.

\subsection{Pre-specified predictions}
\label{app:benchmark-predictions}

Six predictions were fixed before the confirmatory run (App.~\ref{app:prereg}); each is reported as tested.
\begin{itemize}[leftmargin=*,itemsep=1pt,topsep=2pt]
\item \textbf{Q1 (held).} Every estimator averages at most 1 false admission per held-out transplant draw. Observed: 0 for all 12.
\item \textbf{Q2 (held).} The naive event contrast, the positive-peak rule and Granger each average at least 4 uncalibrated null discoveries per transplant draw (over relations). Observed: 18.25, 17.62 and 19.75.
\item \textbf{Q3 (held).} For at least 7 of 12 estimators, the share of the 5 gas-exchange relations admitted with the expected sign exceeds the share of the 11 drug relations. Observed: 8 of 12 (not for signed-peak, direction-aware, Granger and PCMCI+).
\item \textbf{Q4 (failed).} Summed over estimators, at least 80\% of admitted gas-exchange cells have the expected sign, with at least 5 such cells. Observed: 35 concordant and 11 wrong-signed (76\%); counted over distinct estimates (not pre-specified), 32 concordant and 6 wrong-signed (84\%).
\item \textbf{Q5 (failed).} FiO$_2{\to}$SpO$_2$ is admitted with the expected sign by at least one of the naive level contrast, matched, interrupted-trend and sequential-trial designs, and with the wrong sign by none. Observed: none admits it; all four estimates have the expected sign (naive level $1.04$ [$0.51$, $1.56$], stability 5/8; interrupted trend $0.45$ [$0.01$, $0.89$] and sequential trial $0.54$ [$0.03$, $1.04$], adjusted $p$ 0.26 and 0.20, stability 3/8 each; matched $0.20$ [$-0.24$, $0.63$]).
\item \textbf{Q6 (held).} At least one estimator admits a drug relation with the wrong sign. Observed: 13 wrong-signed drug admissions, from the naive event contrast (R04, R06, R07, R08), the naive level contrast (R07, R08), VARX (R07) and the positive-peak rule (R03-R08).
\end{itemize}

\subsection{Sensitivity analyses}
\label{app:benchmark-sensitivity}

\paragraph{Distinct estimates.}
Four estimators return one estimate for several relations (App.~\ref{app:benchmark-estimators}): Granger and VAR-LiNGAM for the increase and decrease relations of a pair (R01/R02, R03/R04, R05/R07, R06/R08, R13/R14), and the positive- and signed-peak rules for the policy and response relations of a pair (10 pairs). Tables~\ref{tab:benchmark} and~\ref{tab:bench-leaderboard} therefore report uncalibrated null discoveries over distinct estimates.

\paragraph{Case phase and volatile wash-in.}
To test whether induction drives the hypnotic and volatile relations, we re-estimated them on the real arm with the event set restricted by phase (not pre-specified); other events stay in the series, so eligibility and control pools are unchanged, and the estimates are not calibrated (Tables~\ref{tab:bench-phase} and~\ref{tab:bench-mac}). The first 30\,min are counted from a case's first change of the action. Signs are preserved for propofol$\to$MAP and propofol$\to$BIS in the matched and interrupted-trend designs; magnitudes are larger in the first 30\,min. In all three designs, P09 has the expected sign in both phases and P04 after the first 30\,min; every P04 interval covers 0 in the first 30\,min. For the volatile relations, the matched and interrupted-trend estimates are stable across phases and after excluding steps preceded by another MAC increase within 300\,s.

\begin{table}[h]
\caption{Real-data estimates by case phase at 2\,s (not calibrated). NE: naive event contrast; MA: matched; KI: interrupted trend. Units as in Table~\ref{tab:bench-estimates}. Cases: all / first 30\,min / after 30\,min.}
\label{tab:bench-phase}
\centering
\footnotesize
\setlength{\tabcolsep}{2.5pt}
\begin{tabular}{@{}llcccc@{}}
\toprule
 & Est. & All events & First 30\,min & After 30\,min & Cases \\
\midrule
R03 & NE & $-0.67$ [$-1.08$, $-0.37$] & $-3.17$ [$-4.82$, $-1.67$] & $-0.17$ [$-0.47$, $0.17$] & 1,552/763/1,411 \\
 & MA & $-0.62$ [$-0.89$, $-0.36$] & $-1.51$ [$-2.54$, $-0.45$] & $-0.46$ [$-0.75$, $-0.16$] & 1,513/628/1,378 \\
 & KI & $-0.89$ [$-1.24$, $-0.56$] & $-2.63$ [$-3.75$, $-1.46$] & $-0.57$ [$-0.85$, $-0.24$] & 1,519/650/1,379 \\
R19 & NE & $-5.30$ [$-5.66$, $-4.91$] & $-13.53$ [$-14.54$, $-12.55$] & $-2.89$ [$-3.15$, $-2.63$] & 1,521/939/1,349 \\
 & MA & $-2.68$ [$-2.95$, $-2.44$] & $-4.84$ [$-5.44$, $-4.19$] & $-2.14$ [$-2.36$, $-1.89$] & 1,459/768/1,299 \\
 & KI & $-3.04$ [$-3.29$, $-2.78$] & $-5.89$ [$-6.79$, $-5.28$] & $-2.29$ [$-2.53$, $-2.06$] & 1,460/790/1,286 \\
P04 & NE & $1.70$ [$1.32$, $2.04$] & $0.09$ [$-1.02$, $1.25$] & $2.06$ [$1.80$, $2.33$] & 1,701/1,093/1,604 \\
 & MA & $0.44$ [$0.16$, $0.78$] & $-0.29$ [$-1.63$, $1.01$] & $0.61$ [$0.39$, $0.81$] & 1,693/1,006/1,589 \\
 & KI & $1.03$ [$0.32$, $1.77$] & $0.94$ [$-1.45$, $3.36$] & $1.05$ [$0.50$, $1.59$] & 1,701/1,115/1,588 \\
P09 & NE & $3.94$ [$3.71$, $4.18$] & $1.64$ [$1.03$, $2.36$] & $4.53$ [$4.29$, $4.80$] & 1,506/862/1,346 \\
 & MA & $1.65$ [$1.51$, $1.84$] & $1.06$ [$0.68$, $1.45$] & $1.80$ [$1.62$, $1.98$] & 1,475/781/1,320 \\
 & KI & $3.97$ [$3.45$, $4.46$] & $3.09$ [$1.94$, $4.34$] & $4.19$ [$3.59$, $4.73$] & 1,478/807/1,309 \\
\bottomrule
\end{tabular}
\end{table}

\begin{table}[h]
\caption{Real-data estimates for derived volatile steps at 2\,s (not calibrated; abbreviations as in Table~\ref{tab:bench-phase}). Late: after the first 30\,min; isolated: without a MAC increase in the preceding 300\,s. Cases: all / late / late, isolated.}
\label{tab:bench-mac}
\centering
\footnotesize
\setlength{\tabcolsep}{2.5pt}
\begin{tabular}{@{}llcccc@{}}
\toprule
 & Est. & All events & Late & Late, isolated & Cases \\
\midrule
R17 & NE & $-5.79$ [$-6.50$, $-5.02$] & $-7.50$ [$-8.26$, $-6.75$] & $-4.05$ [$-4.54$, $-3.51$] & 1,756/1,259/1,202 \\
 & MA & $-1.31$ [$-1.66$, $-1.06$] & $-1.38$ [$-1.71$, $-1.01$] & $-1.40$ [$-1.73$, $-1.05$] & 1,456/1,198/1,179 \\
 & KI & $-1.58$ [$-1.89$, $-1.30$] & $-1.72$ [$-2.07$, $-1.38$] & $-1.74$ [$-2.12$, $-1.37$] & 1,455/1,199/1,180 \\
R18 & NE & $-9.21$ [$-9.70$, $-8.81$] & $-3.12$ [$-3.62$, $-2.72$] & $-1.91$ [$-2.20$, $-1.66$] & 1,666/1,188/1,131 \\
 & MA & $-1.74$ [$-1.89$, $-1.59$] & $-1.73$ [$-1.94$, $-1.56$] & $-1.73$ [$-1.93$, $-1.54$] & 1,378/1,124/1,106 \\
 & KI & $-1.91$ [$-2.11$, $-1.77$] & $-1.88$ [$-2.10$, $-1.69$] & $-1.91$ [$-2.14$, $-1.66$] & 1,370/1,116/1,099 \\
P10 & NE & $-8.53$ [$-9.09$, $-7.92$] & $0.15$ [$-0.14$, $0.40$] & $0.44$ [$0.24$, $0.62$] & 1,659/1,197/1,141 \\
 & MA & $-0.28$ [$-0.40$, $-0.16$] & $-0.28$ [$-0.45$, $-0.15$] & $-0.21$ [$-0.35$, $-0.05$] & 1,393/1,139/1,120 \\
 & KI & $-1.02$ [$-1.39$, $-0.67$] & $-0.83$ [$-1.22$, $-0.41$] & $-0.81$ [$-1.24$, $-0.42$] & 1,387/1,131/1,113 \\
\bottomrule
\end{tabular}
\end{table}

\paragraph{Event definition, control pool, co-interventions and bias correction.}
These pre-specified variants of the design-based estimators were computed at the secondary 30\,s resolution (App.~\ref{app:benchmark-30s}) on its 12 response relations, with the same admission rule (Table~\ref{tab:bench-sens30}). \emph{Per-protocol} keeps an event only if the action does not change again before $\tau+h$ plus the post-window half-width; \emph{future-inaction pool} uses per-protocol events and anchors with no change of the action up to $g+135$\,s, without bias correction; \emph{co-intervention exclusion} drops events and anchors with any change of another action in $[t-60\,\mathrm{s},t)$; \emph{co-intervention adjustment} adds the counts of those changes to the matching features; \emph{no co-intervention features} removes these counts from the sequential-trial state; \emph{pre-trend window 60\,s} fits the interrupted-trend line on 60\,s instead of 180\,s. No variant had a false admission under the held-out transplant draws or the random-time null. The set-rate relations R01 and R02 were admitted with the physical sign by every variant; the table lists the admitted drug relations.

\begin{table}[h]
\caption{Design-based variants at 30\,s: calibrated admissions among the 12 response relations, expected/wrong sign (admitted drug relations; wrong-signed in bold). Primary: intention-to-treat events, pre-decision control pool, bias correction, 180\,s pre-trend.}
\label{tab:bench-sens30}
\centering
\footnotesize
\setlength{\tabcolsep}{4pt}
\begin{tabular}{@{}lccc@{}}
\toprule
Variant & Matched & Interrupted trend & Sequential trial \\
\midrule
Primary & 3/0 (R08) & 2/1 (\textbf{R06}) & 2/2 (\textbf{R05}, \textbf{R06}) \\
Per-protocol events & 3/0 (R08) & 2/0 & n/a \\
Future-inaction pool & 2/0 & 2/0 & n/a \\
No bias correction & 5/0 (R03, R06, R08) & 3/0 (R08) & 3/1 (R03, \textbf{R05}) \\
Co-intervention exclusion & 3/0 (R03) & 2/1 (\textbf{R06}) & n/a \\
Co-intervention adjustment & 3/0 (R03) & 3/1 (R07, \textbf{R06}) & n/a \\
No co-intervention features & n/a & n/a & 2/2 (\textbf{R05}, \textbf{R06}) \\
Pre-trend window 60\,s & n/a & 3/0 (R03) & n/a \\
\bottomrule
\end{tabular}
\end{table}

\subsection{Secondary resolution: \texorpdfstring{30\,s}{30 s} windows}
\label{app:benchmark-30s}

\paragraph{Setting.}
The secondary resolution uses 30\,s windows at a 15\,s stride of MAP, HR and EtCO$_2$ in all 3,442 cases and scores the 18 relations defined on these vitals and on the drug and ventilator-rate actions (R01-R12, P01-P06). The estimators and the admission rule are those of the primary resolution, with these settings: event designs use a 180\,s pre-window, a post-window $[t+h-15\,\mathrm{s},t+h+15\,\mathrm{s}]$, anchors every 15\,s with a 7.5\,s decision interval and a 180\,s pre-trend for the interrupted-trend design; Granger, VARX and VAR-LiNGAM use 4 lags of 30\,s (VARX lag bins from $(0,30]$\,s; VAR-LiNGAM with 60 residuals per case, at most 50,000); PCMCI+ runs on a seeded subsample of 1,200 cases (1,144-1,178 completed per arm) with lags up to 240\,s; the sequential-trial bank keeps at most 40 anchors per case and action, without the 10\,s level.

\begin{table}[h]
\caption{ClosedLoopBench at 30\,s (3,442 cases, 18 relations). Columns as in Table~\ref{tab:benchmark}; Held-out: mean false admissions per held-out transplant draw.}
\label{tab:bench-30s}
\centering
\footnotesize
\setlength{\tabcolsep}{4pt}
\begin{tabular}{@{}lcccccc@{}}
\toprule
 & Null & \multicolumn{3}{c}{Calibrated, expected/wrong} & \multicolumn{2}{c}{False admissions} \\
\cmidrule(lr){3-5}\cmidrule(lr){6-7}
Estimator & (uncal.) & Gas (2) & Drug (10) & Policy (6) & Held-out & RT \\
\midrule
Naive event contrast & 10.12/18 & 2/0 & 1/5 & 6/0 & 0 & 6 \\
Naive level contrast & 5.25/12 & 2/0 & 4/2 & n/a & 0 & 1 \\
Positive-peak cross-corr. & 10.00/12 & 0/2 & 0/5 & 4/0 & 0 & 9 \\
Signed-peak cross-corr. & 7.75/12 & 0/2 & 4/0 & 1/2 & 0 & 2 \\
Direction-aware cross-corr. & 12.50/18 & 0/2 & 5/0 & 4/1 & 0 & 5 \\
Pairwise Granger & 6.50/14 & 2/0 & 6/0 & 5/1 & 0.125 & 6 \\
PCMCI+ & 8.62/18 & 2/0 & 6/0 & 5/0 & 0 & 0 \\
VAR-LiNGAM & 5.00/14 & 2/0 & 4/0 & 6/0 & 0.25 & 4 \\
VARX & 4.25/18 & 2/0 & 3/2 & 2/1 & 0 & 3 \\
Matched event design & 1.50/18 & 2/0 & 1/0 & 3/0 & 0 & 0 \\
Interrupted trend (kink) & 2.12/18 & 2/0 & 0/1 & 3/0 & 0 & 0 \\
Sequential-trial contrast & 1.25/18 & 2/0 & 0/2 & 4/0 & 0 & 0 \\
\bottomrule
\end{tabular}
\end{table}

\paragraph{Results.}
Counted over relations, the uncalibrated null discoveries per transplant draw were 15.62 of 18 for the positive-peak rule, 12.50 for the direction-aware rule, 11.50 for the signed-peak rule, 10.12 for the naive event contrast, 9.50 for Granger, 8.62 for PCMCI+, 7.38 for VAR-LiNGAM, 5.25 of 12 for the naive level contrast, 4.25 for VARX and 1.25-2.12 for the design-based estimators (Table~\ref{tab:bench-30s} gives distinct estimates). Held-out false admissions were at most 0.25 per draw (VAR-LiNGAM; Granger 0.125); with the random-time streams as input the rule admitted up to 9 of 18 relations (positive-peak rule). VARX admitted both remifentanil-decrease relations with the wrong sign, with the same call in 8 of 8 draws: $+0.34$\,bpm per ng\,mL$^{-1}$ [$0.24$, $0.43$] for HR and $+0.94$\,mmHg per ng\,mL$^{-1}$ [$0.63$, $1.25$] for MAP. MAP rising$\to$remifentanil increase and HR rising$\to$remifentanil increase were admitted with the expected sign by all 9 direction-resolving estimators (naive event contrast, direction-aware rule, Granger, VARX, PCMCI+, VAR-LiNGAM and the three design-based estimators), and the signed-peak rule admitted the MAP policy with the wrong sign; the vasopressor and propofol policies were admitted by at most 6 of the 11 estimators that estimate policies. The positive-peak rule's 4 concordant admissions (P01, P02, P04, P05) are the same estimates as 4 of its wrong-signed response admissions (R06, R05, R04, R01). The set-rate relations were admitted with the physical sign by every estimator except the three cross-correlation rules, which admitted both with the opposite sign; PCMCI+'s set-rate admissions come from the transplant offset (real-data mean partial correlation for R01 $-0.0044$ [$-0.016$, $0.006$], transplant mean $+0.029$). In the clinician-in-the-loop simulator (App.~\ref{app:theory-results}), the sequential-trial contrast gave a CI excluding 0 for an inert drug in 6 of 24 cells at 2\,s and 19 of 24 at 30\,s. The split-half test of patient-specific responses (App.~\ref{app:fast-spec}, Table~\ref{tab:slow-spec}) declared no relation patient-specific at either resolution.

\section{Theory and simulation}
\label{app:theory}

\ifdefined\corollary\else\newtheorem{corollary}{Corollary}\fi
\newenvironment{formalprop}[1]{\par\medskip\noindent\textbf{Proposition~\ref{#1} (formal).}\itshape\ }{\par\medskip}

This appendix states the closed-loop model (App.~\ref{app:theory-model}), proves Propositions~\ref{prop:reversal} and~\ref{prop:pools} (App.~\ref{app:theory-reversal}, App.~\ref{app:theory-pools}), describes the simulator (App.~\ref{app:theory-sim}) and reports the simulations (App.~\ref{app:theory-results}). The physiology is linear, so we state results for one vital and one action; other vitals are treated alike, and responses to further actions add. Proposition~\ref{prop:reversal} restates the drug titration paradox \citep{schnider2021drug,schnider2022drug,schamberg2021drug} and the direct-method bias of closed-loop identification \citep{forssell1999closed} for the contrasts used to build temporal relations. Proposition~\ref{prop:pools} restates bias from conditioning on future treatment \citep{suissa2008immortal,hernan2016specifying}. The closed-loop-specific content is the horizon dependence (Corollaries~\ref{cor:crossover} and~\ref{cor:pathway}) and the policy-determined sign of the control-pool bias (Proposition~\ref{prop:pools}(c)).

\subsection{Closed-loop model}
\label{app:theory-model}

Time $t\in\mathbb{Z}$ has step $\Delta$ (1\,s in the simulator), and $L$ is the lag operator. Symbols defined here are local to this appendix; subscripted $\tau$ symbols (time constants such as $\tau_{\mathrm{pk}}$, $\tau_s$ and $\tau_d$, and the estimand $\tau_{\mathrm{ITT}}$) are unrelated to event times $\tau$.

\begin{assumption}[Linear physiology with a latent driver]\label{as:physiology}
The $m$ vitals $\mathbf{y}_t\in\R^m$ satisfy $A(L)\mathbf{y}_t=\mathbf{b}(L)C_t+\mathbf{g}(L)\zeta_t+\boldsymbol{\varepsilon}_t$. Here $C_t$ is the effect-site concentration of the drug; $\zeta_t$ is an unrecorded driver (surgical stimulation, insufflation); $\boldsymbol{\varepsilon}_t$ are i.i.d.\ Gaussian innovations independent of $\zeta$; $A(L)=I-\sum_{k\ge1}A_kL^k$ with $\det A(\lambda)\neq0$ for complex $|\lambda|\le1$ (a stable physiological loop, e.g., baroreflex and cardiac-output couplings); and $\mathbf{b},\mathbf{g}$ are causal, stable filters. The drug enters through a first-order pharmacokinetic (PK) filter,
\begin{equation}\label{eq:pk}
C_t=\frac{1-\varrho}{1-\varrho L}\,u_t,\qquad \varrho=e^{-\Delta/\tau_{\mathrm{pk}}},
\end{equation}
with effect-site time constant $\tau_{\mathrm{pk}}$. The driver $\zeta$ is stationary, or stationary plus isolated events.
\end{assumption}

Solving the loop, each vital is
\begin{equation}\label{eq:decomp}
y_t=\sum_{i}\delta_i\,\Gamma(t-t_i)+f_t+\nu_t,
\end{equation}
where the sum runs over all changes $\delta_i$ of $u$ at times $t_i$ (a constant baseline, removed by every contrast below, is omitted). $\Gamma(k)$ is the loop-closed step response of the vital to a unit change of $u$ ($\Gamma(k)=0$ for $k<0$), with static gain $g_u=\lim_{k\to\infty}\Gamma(k)$. The driver footprint $f_t$ and the noise $\nu_t$ are the corresponding components of $A(L)^{-1}\mathbf{g}(L)\zeta_t$ and $A(L)^{-1}\boldsymbol{\varepsilon}_t$.

\begin{assumption}[Event-triggered clinician]\label{as:clinician}
The input $u$ is piecewise constant and changes only at checks. Routine checks occur at times independent of $(\zeta,\boldsymbol{\varepsilon})$. An alarm check occurs a reaction delay after a displayed statistic crosses an alarm level. At a check at $t$ outside a refractory period after the previous change, the clinician changes $u$ by $+\delta_u$ if $s_t>\eta$ and by $-\delta_u$ if $s_t<-\eta$ (within dose bounds), where
\begin{equation}\label{eq:policy}
s_t=w_y\,s^{\mathrm{rec}}_t+w_{\zeta}\,\zeta_t+w_{\mathrm{ant}}\,Z^{(H_{\mathrm{ant}})}_t .
\end{equation}
$s^{\mathrm{rec}}_t$ is the recorded decision statistic, a linear functional of the displayed vitals (trailing means over $W_d$ steps); $Z^{(H_{\mathrm{ant}})}_t$ is the sum of driver events in $(t,t+H_{\mathrm{ant}}]$, i.e., the surgical plan up to a horizon $H_{\mathrm{ant}}$. Signal-independent (exogenous) changes also occur at a constant rate. The clinician is \emph{recorded-reactive} if $w_{\zeta}=w_{\mathrm{ant}}=0$, \emph{field-reactive} if $w_{\zeta}\neq0$, and \emph{anticipatory} if $w_{\mathrm{ant}}\neq0$. The policy sign is $\sigma_\pi=\sign w_y$: $\sigma_\pi=+1$ if the clinician raises $u$ when the recorded statistic is high.
\end{assumption}

\begin{assumption}[Sampling]\label{as:sampling}
Patients are i.i.d.\ records of bounded length. For a set of event times $\mathcal{T}$ with $\E|\mathcal{T}_p|>0$, $\E[\,\cdot\mid t\in\mathcal{T}]$ denotes the event-weighted expectation $\E[\sum_{t\in\mathcal{T}_p}(\cdot)]/\E|\mathcal{T}_p|$, the limit of averages over all events of $P$ patients as $P\to\infty$.
\end{assumption}

\paragraph{Contrasts and estimand.}
Let $\mathcal{T}_h$ be the times $t$ at which $u$ increases by $\delta_u$, with no other change of $u$ in the pre-window $[t-w_{\mathrm{pre}},t)$ and with $[t-w_{\mathrm{pre}},t+h]$ inside the record. A contrast $\Psi_h$ is a linear functional of a series $\chi$ on $[t-w_{\mathrm{pre}},t+h]$ whose weights sum to zero. The \emph{naive} contrast is $\Psi^{\mathrm{N}}_h\chi_t=\tfrac12(\chi_{t+h-1}+\chi_{t+h})-w_{\mathrm{pre}}^{-1}\sum_{k=1}^{w_{\mathrm{pre}}}\chi_{t-k}$. The \emph{pre-trend} contrast subtracts the extrapolation of the least-squares line fitted on the pre-window. The estimator is the event average $\hat\theta_r(h)$ of $\Psi_hy_t$ over $t\in\mathcal{T}_h$. Its open-loop estimand is $\theta_r(h)=\delta_u\,\Psi_h\Gamma(\cdot-t)$, the contrast of the response to a sustained step of size $\delta_u$ at $t$ with no other change.

\subsection{Sign reversal}
\label{app:theory-reversal}

\begin{formalprop}{prop:reversal}
Under Assumptions~\ref{as:physiology}-\ref{as:sampling}, $\hat\theta_r(h)\to\theta_r(h)+B_r(h)$ almost surely as $P\to\infty$, with $B_r(h)=b_f(h)+b_\nu(h)+b_-(h)+b_+(h)$ and
\begin{align*}
b_f(h)&=\E[\Psi_hf_t\mid t\in\mathcal{T}_h], &
b_-(h)&=\E\Bigl[\textstyle\sum_{i:\,t_i<t-w_{\mathrm{pre}}}\delta_i\,\Psi_h\Gamma(\cdot-t_i)\Bigm| t\in\mathcal{T}_h\Bigr],\\
b_\nu(h)&=\E[\Psi_h\nu_t\mid t\in\mathcal{T}_h], &
b_+(h)&=\E\Bigl[\textstyle\sum_{i:\,t<t_i\le t+h}\delta_i\,\Psi_h\Gamma(\cdot-t_i)\Bigm| t\in\mathcal{T}_h\Bigr].
\end{align*}
The terms $b_f$ and $b_\nu$ select the driver footprint and noise through the clinician's decisions. The term $b_-$ is the still-evolving response to earlier changes of $u$, and $b_+$ is the response to later changes within the horizon. (i) If $\zeta$ is stationary and decisions are independent of $(\zeta,\boldsymbol{\varepsilon})$, then $b_f=b_\nu=0$. (ii) The limit has the sign opposite to $\theta_r(h)$ if and only if $\sign B_r(h)=-\sign\theta_r(h)$ and $|B_r(h)|>|\theta_r(h)|$.
\end{formalprop}

\begin{proof}
Apply the linear functional $\Psi_h$ to Eq.~\ref{eq:decomp} at $t\in\mathcal{T}_h$. The focal change occurs at $t$, after the pre-window, and contributes $\delta_u\Psi_h\Gamma(\cdot-t)=\theta_r(h)$. By definition of $\mathcal{T}_h$, no change occurs in $[t-w_{\mathrm{pre}},t)$. Changes before $t-w_{\mathrm{pre}}$ and in $(t,t+h]$ contribute the summands of $b_-$ and $b_+$, and changes after $t+h$ do not enter the window. The remaining terms are $\Psi_hf_t$ and $\Psi_h\nu_t$. $\hat\theta_r(h)$ is a ratio of sums of i.i.d.\ per-patient quantities, which are integrable because records and doses are bounded and $\zeta$ and $\boldsymbol{\varepsilon}$ have finite means; by the strong law of large numbers it converges to the ratio of their expectations, i.e., to the event-weighted expectation of the sum above. (i) If event times are independent of $(\zeta,\boldsymbol{\varepsilon})$, then $\E[\Psi_hf_t\mid t\in\mathcal{T}_h]=\E[\Psi_hf_t]=0$, because $f$ is stationary and the weights of $\Psi_h$ sum to zero; the same holds for $\nu$. (ii) The limit is $\theta_r+B_r$, whose sign differs from that of $\theta_r$ exactly when $B_r$ has the opposite sign and larger magnitude.
\end{proof}

\begin{corollary}[Recorded-reactive clinician: crossover horizon]\label{cor:crossover}
Let $y^0=f+\nu$ be stationary Gaussian with variance $\sigma^2$ and autocorrelation $\rho(k)$. Suppose that at a check at $t$ the clinician increases $u$ iff $s_t>\eta$, where $s_t=W_d^{-1}\sum_{k=0}^{W_d-1}y^0_{t-k}$ is the displayed mean ($\sigma_\pi=+1$); that the check time and the eligibility of $t$ are independent of $y^0$; and that no other change of $u$ affects $[t-w_{\mathrm{pre}},t+h]$. Then, for the naive contrast,
\begin{equation}\label{eq:crossover}
B_r(h)=\frac{\Cov(\Psi^{\mathrm{N}}_hy^0_t,\,s_t)}{\Var(s_t)}\;\E\bigl[s_t-\E s_t\bigm| s_t>\eta\bigr],
\qquad
\Cov(\Psi^{\mathrm{N}}_hy^0_t,\,s_t)=\sigma^2\bigl(\bar\rho_{\mathrm{post}}(h)-\bar\rho_{\mathrm{pre}}\bigr),
\end{equation}
with
\begin{equation*}
\bar\rho_{\mathrm{post}}(h)=\frac{1}{2W_d}\sum_{q\in\{h-1,h\}}\sum_{k=0}^{W_d-1}\rho(q+k),\qquad
\bar\rho_{\mathrm{pre}}=\frac{1}{w_{\mathrm{pre}}W_d}\sum_{i=1}^{w_{\mathrm{pre}}}\sum_{k=0}^{W_d-1}\rho(i-k).
\end{equation*}
The conditional expectation is positive, so $\sign B_r(h)=\sign(\bar\rho_{\mathrm{post}}(h)-\bar\rho_{\mathrm{pre}})$. Let $\rho(k)=e^{-|k|/T}$ with finite persistence time $T$, and read the formula for $\bar\rho_{\mathrm{post}}(h)$ for real $h\ge1$. Then $\bar\rho_{\mathrm{post}}(h)$ decreases strictly toward 0 while $\bar\rho_{\mathrm{pre}}>0$ is fixed, so the bias changes sign at most once. (a) If $\bar\rho_{\mathrm{post}}(1)>\bar\rho_{\mathrm{pre}}$, a crossover exists: the bias is positive below the crossover horizon $h^\star(T)$, the root of $\bar\rho_{\mathrm{post}}(h)=\bar\rho_{\mathrm{pre}}$, and negative above it. This holds for every finite $T$ in Table~\ref{tab:hstar}. (b) Otherwise no crossover exists and the bias is negative for every $h>1$; for example, $w_{\mathrm{pre}}=1$, $W_d=10$ and $T=10$ give $\bar\rho_{\mathrm{post}}(1)=0.63<\bar\rho_{\mathrm{pre}}=0.71$. Let
\begin{equation}\label{eq:hstar-limit}
h^\star_\infty=\E|U_1-U_2|-\E U_2+\tfrac12,\qquad U_1\sim\mathcal{U}\{1,\dots,w_{\mathrm{pre}}\},\; U_2\sim\mathcal{U}\{0,\dots,W_d-1\}\ \text{independent}.
\end{equation}
If $h^\star_\infty>1$, case (a) holds for all large $T$, and $h^\star(T)\to h^\star_\infty$ as $T\to\infty$.
\end{corollary}

\begin{proof}
For jointly Gaussian $(X,s)$, $\E[X\mid s]=\E X+\Cov(X,s)\Var(s)^{-1}(s-\E s)$. Taking the expectation over $\{s>\eta\}$ with $X=\Psi^{\mathrm{N}}_hy^0_t$ and $\E X=0$ gives the first identity; the events select on $s_t$ only, by the independence assumptions. The covariance follows by bilinearity from $\Cov(y^0_{t+q},y^0_{t-k})=\sigma^2\rho(q+k)$ and $\Cov(y^0_{t-i},y^0_{t-k})=\sigma^2\rho(i-k)$. $\E[s-\E s\mid s>\eta]>0$ for a non-degenerate Gaussian. For $h\ge1$, every lag $q+k$ in $\bar\rho_{\mathrm{post}}(h)$ is nonnegative, so each term $e^{-(q+k)/T}$ decreases strictly in $h$ and tends to 0. Hence $D_T(h)=\bar\rho_{\mathrm{post}}(h)-\bar\rho_{\mathrm{pre}}$ is continuous and decreases strictly to $-\bar\rho_{\mathrm{pre}}<0$: it has exactly one root in $h>1$ if $D_T(1)>0$ (case (a)) and is negative for all $h>1$ otherwise (case (b)). For the limit, $\rho(k)=1-|k|/T+O(T^{-2})$ on bounded $k$ gives $D_T(h)=T^{-1}\bigl(\E|U_1-U_2|-(h-\tfrac12+\E U_2)\bigr)+O(T^{-2})=T^{-1}(h^\star_\infty-h)+O(T^{-2})$, uniformly on bounded $h$ (Eq.~\ref{eq:hstar-limit}). If $h^\star_\infty>1$, then $T\,D_T(1)\to h^\star_\infty-1>0$, so case (a) holds for large $T$; and for $0<\epsilon<h^\star_\infty-1$, $T\,D_T(h^\star_\infty\mp\epsilon)\to\pm\epsilon$, so by monotonicity the root lies within $\epsilon$ of $h^\star_\infty$ for large $T$.
\end{proof}

For a drug that lowers the vital ($\theta_r<0$) and a clinician who raises the dose when the vital is high, the naive contrast, when a crossover exists, is biased against the effect below $h^\star$ and exaggerates it above; without a crossover it exaggerates the effect at every horizon $h>1$. Table~\ref{tab:hstar} evaluates $h^\star(T)$ for the simulator's display ($W_d=10$\,s) and a 120-s pre-window. For persistence times $T$ from 5\,s to $10^4$\,s the crossover lies between 11 and 53\,s; it is set mostly by the pre-window. A Monte Carlo simulation of this selection at $T=30$\,s gives 36\,s (formula: 35\,s). The simulator departs from these assumptions (for example, alarm-triggered checks, a statistic combining HR and MAP, and episodic drivers), and its crossover is later: with an inert drug and the recorded clinician, the bias of the 2-s naive contrast is $+0.96$ [0.67, 1.25]\,bpm at 60\,s and $-0.34$ [$-0.61$, $-0.06$]\,bpm at 120\,s.

\begin{table}[h]
\centering
\caption{Crossover horizon $h^\star(T)$ of the naive contrast (Corollary~\ref{cor:crossover}; $w_{\mathrm{pre}}=120$\,s, $W_d=10$\,s): first horizon with negative bias on the 1-s grid; $T=\infty$: Eq.~\ref{eq:hstar-limit}.}
\label{tab:hstar}
\small
\begin{tabular}{@{}lcccccccccc@{}}
\toprule
$T$ (s) & 5 & 10 & 20 & 30 & 60 & 120 & 300 & 1{,}000 & 10{,}000 & $\infty$ \\
$h^\star$ (s) & 11 & 19 & 29 & 35 & 43 & 48 & 51 & 52 & 53 & 52.2 \\
\bottomrule
\end{tabular}
\end{table}

\begin{corollary}[Field-reactive and anticipatory clinicians]\label{cor:pathway}
Let a driver event of size $A_{\zeta}$ at time $t_{\zeta}$ add $\gamma_{\zeta}A_{\zeta}\varphi(t-t_{\zeta})$ to the vital, with loop-closed gain $\gamma_{\zeta}$ and $\varphi(t)=1-e^{-t/\tau_s}$ for $t\ge0$ and 0 otherwise (a first-order sympathetic pathway with time constant $\tau_s$). Let the clinician respond at $t=t_{\zeta}+\ell$ with $+\delta_u$ (field-reactive, reaction delay $\ell$), and ignore $\nu$ and other changes of $u$. For the contrast $\Psi_h\chi_t=\chi_{t+h}-\chi_t$, $b_f(h)=\gamma_{\zeta}A_{\zeta}[\varphi(\ell+h)-\varphi(\ell)]$. Suppose the drug opposes the driver ($\sign g_u=-\sign\gamma_{\zeta}A_{\zeta}$) and its response is dominated by the PK filter, so that $\delta_u\Gamma(h)\approx\delta_ug_uh/\tau_{\mathrm{pk}}$ for $h\ll\tau_{\mathrm{pk}}$. Then, to first order in $h$, the estimate is wrong-signed at short horizons iff
\begin{equation}\label{eq:pathway}
\frac{\tau_{\mathrm{pk}}}{\tau_s}\,e^{-\ell/\tau_s}>\frac{\delta_u|g_u|}{|\gamma_{\zeta}A_{\zeta}|}.
\end{equation}
For an anticipatory clinician who acts a fixed lead $\ell_{\mathrm{ant}}$ before $t_{\zeta}$, $b_f(h)=\gamma_{\zeta}A_{\zeta}\,\varphi((h-\ell_{\mathrm{ant}})_+)$, with $(\cdot)_+=\max(\cdot,0)$: it is zero for $h\le\ell_{\mathrm{ant}}$ and, for $h>\ell_{\mathrm{ant}}$, equals the field-reactive term with $\ell=0$ at horizon $h-\ell_{\mathrm{ant}}$.
\end{corollary}

\begin{proof}
Substitute the footprint into $b_f$. $\varphi(\ell+h)-\varphi(\ell)=h\,\tau_s^{-1}e^{-\ell/\tau_s}+O(h^2)$, and $\Gamma(h)=g_u(1-e^{-h/\tau_{\mathrm{pk}}})=g_uh/\tau_{\mathrm{pk}}+O(h^2)$. Comparing the two first-order terms gives Eq.~\ref{eq:pathway}. In the anticipatory case the footprint starts $\ell_{\mathrm{ant}}$ after the decision.
\end{proof}

The left side of Eq.~\ref{eq:pathway} is the pathway-speed ratio: a driver that acts faster than the drug's effect site produces reversal at short horizons. With the simulator's loop-closed gains, $\delta_u|g_u|=0.5\times1.79=0.89$\,bpm for a 0.5\,ng\,mL$^{-1}$ remifentanil step, and $\gamma_{\zeta}A_{\zeta}=3.57$\,bpm for a unit stimulation episode, so the right side is 0.25. With $\tau_{\mathrm{pk}}=100$\,s and $\tau_s,\ell\in\{10,20,30\}$\,s, the left side ranges from 0.50 to 3.68. Reversal is therefore predicted at short horizons in all nine combinations.

\paragraph{Lagged regressions.}
If $\hat\theta_r$ is a least-squares coefficient on lags of $u$ (pairwise Granger, VARX, or an autoregressive model with exogenous input, ARX), its limit is $\theta_r$ plus the projection of $f+\nu$ on the lagged inputs given the other regressors. This term vanishes when $u$ is uncorrelated with the disturbance given the regressors. In closed loop, that requires the regressors to contain the clinician's decision statistic at its native resolution and no unrecorded information to drive decisions. This is the direct-method bias of closed-loop identification \citep{forssell1999closed}; App.~\ref{app:theory-results} shows it for a 30-s ARX with an inert drug.

\subsection{Control pools}
\label{app:theory-pools}

For every time $t'$, let $X_{t'}\in\{0,1\}$ indicate a change of $u$ in the event direction $d\in\{-1,+1\}$ ($+1$ for an increase) at $t'$. $E_{t'}\in\{0,1\}$ is an eligibility indicator measurable with respect to information before $t'$ (no change in the preceding refractory period, dose not at its bound in direction $d$), and $M_{t'}$ is a matching statistic measurable with respect to the recorded information before $t'$. $Y_{t'}(1)$ and $Y_{t'}(0)$ are the contrast $\Psi_hy_{t'}$ when the action at $t'$ is taken or withheld and the clinician follows the policy afterwards (usual care), and $Y_{t'}=Y_{t'}(X_{t'})$. The per-decision intention-to-treat (ITT) estimand is $\tau_{\mathrm{ITT}}(h)=\E[Y_t(1)-Y_t(0)\mid X_t=1,E_t=1]$. A matched estimator compares each event with controls of the same $M$ from a pool $\mathcal{Q}$; its limit is $\theta^{\mathcal{Q}}=\E[Y_t-\mu_{\mathcal{Q}}(M_t)\mid X_t=1,E_t=1]$ (event-weighted, Assumption~\ref{as:sampling}) with $\mu_{\mathcal{Q}}(m)=\E[Y_{t'}\mid t'\in\mathcal{Q},M_{t'}=m]$. The \emph{eligible} pool is $\mathcal{Q}_E=\{t':X_{t'}=0,E_{t'}=1\}$. The \emph{future-inaction} pool is $\mathcal{Q}_F=\mathcal{Q}_E\cap\{F_{t'}=0\}$, where $F_{t'}=\mathbf{1}[u\text{ changes in }(t',t'+h]]$. Proposition~\ref{prop:pools} and its proof hold unchanged if $(t',t'+h]$ is replaced by a fixed window $(t',t'+\bar h]$ with $\bar h\ge h$ in $F_{t'}$, (c1) and (c3). Write $\pi_F(m)=\Pr(F_{t'}=1\mid X_{t'}=0,E_{t'}=1,M_{t'}=m)$, $\mu_1(m)=\E[Y_{t'}\mid X_{t'}=0,E_{t'}=1,F_{t'}=1,M_{t'}=m]$, and $\mu_0(m)$ for the same expectation with $F_{t'}=0$.

\begin{formalprop}{prop:pools}
Assume (i) sequential exchangeability, $Y_{t'}(0)\perp X_{t'}\mid M_{t'},E_{t'}=1$, and (ii) overlap, $0<\Pr(X_{t'}=1\mid M_{t'},E_{t'}=1)<1$ on the support of $M$ among events. Then:
(a) $\theta^{\mathcal{Q}_E}=\tau_{\mathrm{ITT}}(h)$;
(b) if $\pi_F<1$ on the support of $M$ among events, $\theta^{\mathcal{Q}_F}=\tau_{\mathrm{ITT}}(h)+b_{\mathrm{FI}}$ with $b_{\mathrm{FI}}=\E\bigl[\pi_F(M_t)\{\mu_1(M_t)-\mu_0(M_t)\}\mid X_t=1,E_t=1\bigr]$;
(c) suppose further that, within the matched strata, (c1) later changes of $u$ within $(t',t'+h]$ can only be in direction $d$; (c2) at a check time $t_c$ the clinician makes such a change if and only if $d\,\sigma_\pi\,s^{\mathrm{rec}}_{t_c}>\eta$; and (c3) conditional on the stratum and on the check times, $(Y_{t'},s^{\mathrm{rec}}_{t'+1},\dots,s^{\mathrm{rec}}_{t'+h})$ is associated, and $\E[Y_{t'}\mid\text{stratum, check times}]$ does not depend on the check times. Then $d\cdot b_{\mathrm{FI}}$ has the sign of $\sigma_\pi$ (or is zero). For a corrective policy, $\sigma_\pi=-\sign(\text{effect per unit increase})$, so the future-inaction bias opposes the true effect.
\end{formalprop}

A random vector $V$ is \emph{associated} if $\Cov(g_1(V),g_2(V))\ge0$ for all coordinatewise non-decreasing functions $g_1,g_2$. The condition on $\E[Y_{t'}\mid\cdot]$ in (c3) holds exactly for an inert drug when all checks in the window are routine. Alarm checks depend on the display, and with an active drug later actions triggered at checks also change $Y_{t'}$; the simulations include both and test the sign for inert and active drugs.

\begin{proof}
(a) On $\{X_{t'}=0\}$, $Y_{t'}=Y_{t'}(0)$. By (i), $\mu_{\mathcal{Q}_E}(m)=\E[Y_{t'}(0)\mid X_{t'}=0,E_{t'}=1,M_{t'}=m]=\E[Y_t(0)\mid X_t=1,E_t=1,M_t=m]$, which is defined on the support of events by (ii). Averaging $Y_t(1)-Y_t(0)$ over events gives $\tau_{\mathrm{ITT}}(h)$.
(b) $\mu_{\mathcal{Q}_E}(m)=\pi_F(m)\mu_1(m)+\{1-\pi_F(m)\}\mu_0(m)$ and $\mu_{\mathcal{Q}_F}(m)=\mu_0(m)$, so $\mu_{\mathcal{Q}_E}-\mu_{\mathcal{Q}_F}=\pi_F(\mu_1-\mu_0)$; average over the events' $M$.
(c) By (c1)-(c2), $F_{t'}=\max_{t_c}\mathbf{1}[d\sigma_\pi s^{\mathrm{rec}}_{t_c}>\eta]$ over the check times $t_c\in(t',t'+h]$. This is a non-decreasing function of $(s^{\mathrm{rec}}_{t_c})_{t_c}$ if $d\sigma_\pi=+1$ and a non-increasing function if $d\sigma_\pi=-1$. By association, conditional on the stratum and the check times, $\Cov(Y_{t'},F_{t'})$ has the sign of $d\sigma_\pi$ or is zero. Because $\E[Y_{t'}\mid\text{stratum, check times}]$ does not depend on the check times, the law of total covariance preserves this sign unconditionally on the check times. $\Cov(Y_{t'},F_{t'}\mid\text{stratum})=\pi_F(1-\pi_F)(\mu_1-\mu_0)$, so $\mu_1-\mu_0$ and hence $b_{\mathrm{FI}}$ have the sign of $d\sigma_\pi$, and $d\,b_{\mathrm{FI}}$ that of $\sigma_\pi$ ($d^2=1$). A corrective clinician raises a vital-lowering drug when the vital is high ($\sigma_\pi=+1$, effect per unit increase $<0$), and lowers a vital-raising drug when the vital is high ($\sigma_\pi=-1$, effect $>0$).
\end{proof}

\paragraph{Remarks.}
Excluding controls that are later treated conditions eligibility on information after baseline. This is the error behind immortal-time bias \citep{suissa2008immortal,hernan2016specifying}; a related violation arises with event-dependent exposures in self-controlled case series \citep{farrington2009case}. Selecting treated events without a later change before $t+h$ (``clean'' events) makes the same error on the treated arm; the ITT contrast keeps all events. What is specific to closed-loop data is part (c): the policy fixes the sign, and for a corrective policy the bias reinforces the titration paradox instead of diluting it. When later actions in both directions are possible (two-sided strata), their contributions have opposite signs and can cancel. Alarm-triggered checks violate overlap (ii), because the action is then a deterministic function of the display; events on that path must be excluded or analyzed with an external source of variation.

\paragraph{Sequential-trial contrast.}
Every eligible time is a trial of ``act now'' against ``not now, usual care afterwards'' \citep{hernan2008observational,danaei2013observational,hernan2016using}, and Proposition~\ref{prop:pools}(a) is its matched form. For a recorded-reactive clinician, among eligible times that are not alarm-triggered checks and at which the recorded decision statistic is beyond the threshold in the event direction (\emph{at-risk} times), whether a routine check occurs, and hence whether the clinician acts, is independent of the physiology. The pooled sequential-trial contrast compares the mean outcome of at-risk times with a decision with that of at-risk times without one, summing over all patients. Matching each event only to controls of the same patient is not equivalent: the refractory period after an event makes the following times ineligible, so same-patient controls similar in $M$ tend to precede the event (App.~\ref{app:theory-results}).

\subsection{Simulator}
\label{app:theory-sim}

The simulator implements Assumptions~\ref{as:physiology}-\ref{as:clinician} at a 1-s step for heart rate, MAP and EtCO$_2$ (Table~\ref{tab:sim-params}). Each simulated patient contributes 2\,h after a 10-min burn-in. Patient-specific gains are multiplied by log-normal factors (log-scale s.d.\ 0.2) with their signs kept. The clinician watches 10-s trailing means of the vitals. For remifentanil, the recorded statistic is $s^{\mathrm{rec}}_t=\tfrac12(\mathrm{HR}^{\mathrm{disp}}_t/8+\mathrm{MAP}^{\mathrm{disp}}_t/10)$, where the superscript denotes the displayed deviation from baseline; for the ventilator it is $\mathrm{EtCO}^{\mathrm{disp}}_{2,t}/4$. The clinician types are \emph{recorded} ($w_y=1$, $w_{\zeta}=w_{\mathrm{ant}}=0$), \emph{reversed} ($w_y=-1$, raising the dose when the vitals are low), \emph{recorded+field} ($w_{\zeta}=1$), and \emph{recorded+field+plan} ($w_{\zeta}=1$, $w_{\mathrm{ant}}=1.5$, $H_{\mathrm{ant}}=60$\,s); a policy gain (1 unless stated) multiplies $w_y$, $w_\zeta$ and $w_{\mathrm{ant}}$. An \emph{inert} drug has zero effect on HR and MAP, so the truth is 0 and any systematic deviation from 0 is bias. The simulator provides three observation scales: 1\,s, 2-s monitor numerics (trailing means) and 30-s windows at a 15-s stride (window medians). Two truths are available. The \emph{sustained-step} truth applies each estimator's contrast to a noise-free, driver-free, open-loop step response ($\theta_r(h)$ above). The per-decision \emph{ITT} truth comes from paired simulations that share every random draw and veto a single decision (the first three changes per patient in the event direction). CIs are 95\% patient-cluster bootstrap intervals (1,000 resamples). The simulation hypotheses were fixed before the confirmatory runs (Table~\ref{tab:sim-hyp}; App.~\ref{app:prereg}).

\begin{table}[h]
\centering
\caption{Simulator parameters (population values).}
\label{tab:sim-params}
\small
\begin{tabular}{@{}lp{0.66\linewidth}@{}}
\toprule
Component & Value \\
\midrule
Vital dynamics & first order; $\tau_{\mathrm{HR}}=2$\,s, $\tau_{\mathrm{MAP}}=3$\,s, $\tau_{\mathrm{CO_2}}=60$\,s \\
Baroreflex (MAP$\to$HR) & $-0.4$\,bpm/mmHg, latency 2\,s \\
Cardiac output (HR$\to$MAP) & $+0.3$\,mmHg/bpm, latency 1\,s \\
Remifentanil & $\tau_{\mathrm{pk}}=100$\,s; $-4$\,bpm and $-5$\,mmHg per ng\,mL$^{-1}$; loop-closed HR gain $-1.79$\,bpm per ng\,mL$^{-1}$ \\
Ventilator rate$\to$EtCO$_2$ & $-2.5$\,mmHg per breath/min; insufflation $+8$\,mmHg per unit \\
Stimulation driver & Ornstein-Uhlenbeck process (correlation time $\tau_d=200$\,s, s.d.\ 0.8) plus episodes (4/h, gamma sizes, mean 1) decaying with $\tau_d$ \\
Insufflation driver & 15-40-min periods at 0.75/h; filter time constant $\max(30\,\mathrm{s},0.5\tau_d)$; plus an Ornstein-Uhlenbeck component (correlation time $\tau_d$, s.d.\ 0.15) \\
Sympathetic pathway & first order, $\tau_s=\max(1\,\mathrm{s},0.1\tau_d)$; $+8$\,bpm and $+10$\,mmHg per unit \\
Noise & innovations 1.2\,bpm, 2.0\,mmHg, 0.15\,mmHg; measurement 0.3, 0.5, 0.2 \\
Clinician & routine checks every 120\,s on average; threshold $\eta=1.5$; alarm when $|s^{\mathrm{rec}}|>2$ for either action, check 10\,s later; refractory 120\,s after any change of the same input; remifentanil step 0.5\,ng\,mL$^{-1}$ in [1, 8]; ventilator step 2\,breaths/min in [8, 24]; exogenous changes of random sign 1/h (remifentanil), 0.3/h (ventilator) \\
\bottomrule
\end{tabular}
\end{table}

\paragraph{Estimators.}
All estimators target remifentanil$\uparrow\to$HR or ventilator rate$\uparrow\to$EtCO$_2$.
\begin{itemize}[leftmargin=*,itemsep=1pt,topsep=2pt]
\item \emph{Naive}: $\Psi^{\mathrm{N}}_h$ with a 120-s pre-window.
\item \emph{Pre-trend}: the mean of the last three samples up to $t+h$ minus the extrapolated least-squares line of the 60-s pre-window.
\item \emph{Matched} (30\,s): the change from the mean of the last three pre-event samples to the post mean, at the event minus at the nearest eligible control (matched on pre-level, pre-trend and dose).
\item \emph{Kink} (interrupted trend): the pre-trend deviation minus the mean deviation of the five eligible controls nearest in pre-trend level and slope.
\item \emph{Kink with decision matching}: the same, additionally matched on the displayed means of every vital over the last 10\,s and on the dose.
\item \emph{Exogenous-only}: the naive contrast restricted to exogenous changes.
\item \emph{Direct ARX} (30\,s): a per-patient regression of HR on two lags each of HR, dose and MAP; the estimate is the median over patients of the fitted response at $h$ to a sustained 0.5\,ng\,mL$^{-1}$ step.
\item \emph{Decision-statistic matching} (1\,s; Tables~\ref{tab:sim-pools} and~\ref{tab:sim-policy}): $\Psi^{\mathrm{N}}_h$ with a 60-s pre-window at each routine-path event, minus its mean over the five eligible controls of the same patient nearest in $s^{\mathrm{rec}}$ (sampled every third second; eligibility as in App.~\ref{app:theory-pools}, alarm-triggered checks excluded). The future-inaction pool excludes controls with a change in the next 111\,s at every horizon. Horizons stop at 110\,s, inside the 120-s refractory period, so no event is followed by another change of the same input before $t+h$.
\end{itemize}

\subsection{Simulation results}
\label{app:theory-results}

\paragraph{Horizon dependence of sign reversal.}
Table~\ref{tab:sim-reversal} shows the pattern of Proposition~\ref{prop:reversal} for the recorded clinician (200 patients, active drug). The naive contrasts report an HR increase after remifentanil increases at horizons up to 60\,s, and the 2-s naive contrast exceeds the truth in magnitude at 120\,s. The same contrast restricted to exogenous changes tracks the truth. For ventilator increases, the 2-s naive contrast is wrong-signed only at 10\,s and the 30-s naive contrast only at 30\,s. A correct sign does not imply an unbiased estimate: the pre-trend contrast is never wrong-signed, but it is 10 to 14 times the truth, because it measures the mean reversion of the deviation that triggered the decision.

\begin{table}[h]
\centering
\caption{Estimates [95\% CI] per 0.5\,ng\,mL$^{-1}$ remifentanil increase (HR, bpm) or 2\,breaths/min ventilator-rate increase (EtCO$_2$, mmHg); recorded clinician, active drug, 200 patients. Truth: sustained-step estimand of the estimator above (the exogenous-only contrast shares the 2-s naive truth).}
\label{tab:sim-reversal}
\footnotesize
\setlength{\tabcolsep}{3pt}
\resizebox{\ifdim\width>\linewidth\linewidth\else\width\fi}{!}{%
\begin{tabular}{@{}lcccc@{}}
\toprule
Estimator & $h=10$\,s & 30\,s & 60\,s & 120\,s \\
\midrule
\multicolumn{5}{@{}l}{\emph{Remifentanil$\uparrow\to$HR}} \\
naive, 2\,s & $+1.68$ [1.36, 2.00] & $+1.27$ [0.95, 1.60] & $+0.52$ [0.20, 0.88] & $-1.03$ [$-1.34$, $-0.65$] \\
\quad truth & $-0.11$ & $-0.25$ & $-0.42$ & $-0.63$ \\
exogenous-only, 2\,s & $-0.19$ [$-0.54$, 0.14] & $-0.29$ [$-0.73$, 0.16] & $-0.43$ [$-0.90$, 0.03] & $-0.73$ [$-1.28$, $-0.15$] \\
naive, 30\,s & n/a & $+1.91$ [1.57, 2.25] & $+1.29$ [0.96, 1.65] & $-0.30$ [$-0.60$, 0.04] \\
\quad truth & n/a & $-0.14$ & $-0.34$ & $-0.59$ \\
pre-trend, 2\,s & $-1.09$ [$-1.30$, $-0.88$] & $-2.47$ [$-2.84$, $-2.06$] & $-4.59$ [$-5.27$, $-3.86$] & $-9.04$ [$-10.34$, $-7.72$] \\
\quad truth & $-0.10$ & $-0.24$ & $-0.41$ & $-0.63$ \\
\addlinespace
\multicolumn{5}{@{}l}{\emph{Ventilator rate$\uparrow\to$EtCO$_2$}} \\
naive, 2\,s & $+0.36$ [0.21, 0.50] & $-0.85$ [$-1.03$, $-0.68$] & $-2.00$ [$-2.21$, $-1.80$] & $-3.05$ [$-3.32$, $-2.81$] \\
\quad truth & $-0.77$ & $-1.97$ & $-3.16$ & $-4.32$ \\
naive, 30\,s & n/a & $+0.26$ [0.09, 0.42] & $-1.22$ [$-1.44$, $-1.02$] & $-2.65$ [$-2.91$, $-2.41$] \\
\quad truth & n/a & $-1.10$ & $-2.63$ & $-4.13$ \\
\bottomrule
\end{tabular}}
\end{table}

Across a grid of 60 closed-loop cohorts of field-reactive clinicians ($w_y=w_\zeta=1$, $w_{\mathrm{ant}}=1.5$; policy gain $\{0.5,1,2,4\}$ $\times$ plan horizon $H_{\mathrm{ant}}\in\{0,60,180\}$\,s $\times$ driver time scale $\tau_d\in\{20,60,200,600,2000\}$\,s; 60 patients each), Table~\ref{tab:sim-phase} gives the share of cohorts in which the estimate of remifentanil$\uparrow\to$HR is wrong-signed (CI excludes 0 on the wrong side). For the naive and matched contrasts the share falls with the horizon, in line with the short-horizon reversal of Corollaries~\ref{cor:crossover} and~\ref{cor:pathway}. It is lowest for designs that condition on the pre-event trajectory or on lagged vitals; a share of 0 does not imply an unbiased estimate (pre-trend above, direct ARX below). Ventilator$\to$EtCO$_2$ is wrong-signed only for the 30-s naive contrast at 30\,s (0.18 of cohorts).

\begin{table}[h]
\centering
\caption{Share of the 60 closed-loop simulator cohorts in which the remifentanil$\uparrow\to$HR estimate is wrong-signed.}
\label{tab:sim-phase}
\small
\begin{tabular}{@{}lccc@{}}
\toprule
Estimator (eligible control pool where applicable) & $h=30$\,s & 60\,s & 120\,s \\
\midrule
Naive, 30\,s & 0.87 & 0.43 & 0.05 \\
Matched, 30\,s & 0.67 & 0.23 & 0.02 \\
Naive, 2\,s & 0.43 & 0.15 & 0.00 \\
Kink, 2\,s & 0.05 & 0.07 & 0.03 \\
Kink with decision matching, 2\,s & 0.00 & 0.00 & 0.00 \\
Pre-trend, 2\,s & 0.00 & 0.00 & 0.00 \\
Direct ARX, 30\,s & 0.00 & 0.00 & 0.00 \\
Exogenous-only, 2\,s & 0.02 & 0.02 & 0.02 \\
\bottomrule
\end{tabular}
\end{table}

\paragraph{Direct-method regression with an inert drug.}
With an inert drug (200 patients), the 30-s direct ARX gives estimates at $h=30$, 60 and 120\,s of $-0.08$ [$-0.16$, 0.02], $-0.06$ [$-0.12$, $-0.02$] and $-0.06$ [$-0.10$, $-0.03$]\,bpm for the recorded clinician. For recorded+field they are $+0.10$ [0.02, 0.16], $-0.01$ [$-0.04$, 0.02] and $-0.07$ [$-0.09$, $-0.04$]. For recorded+field+plan they are $+0.25$ [0.16, 0.34], $+0.08$ [0.05, 0.09] and $-0.05$ [$-0.07$, $-0.03$]. At 30\,s, the recorded+field+plan estimate exceeds in magnitude the active drug's sustained-step effect at the 30-s scale ($-0.14$\,bpm; Table~\ref{tab:sim-reversal}). A second replicate for the recorded and recorded+field+plan clinicians gives the same CI decisions except for the latter at 60\,s (recorded: $-0.11$ [$-0.23$, 0.01], $-0.07$ [$-0.12$, $-0.01$], $-0.06$ [$-0.09$, $-0.02$]; recorded+field+plan: $+0.19$ [0.13, 0.26], $+0.025$ [$-0.002$, 0.049], $-0.05$ [$-0.06$, $-0.03$]). In open loop, with exogenous changes only, every CI covers 0 in two replicates ($-0.03$ [$-0.08$, 0.03], $+0.01$ [$-0.02$, 0.03], $+0.01$ [$-0.01$, 0.03]; $+0.03$ [$-0.03$, 0.08], $+0.01$ [$-0.01$, 0.03], $-0.01$ [$-0.03$, 0.02]). For every clinician the decision statistic, a 10-s display of 1-s data, is not among the regressors of the 30-s model, so the condition of the lagged-regression paragraph in App.~\ref{app:theory-reversal} is not met.

\paragraph{Control pools.}
Table~\ref{tab:sim-pools} tests Proposition~\ref{prop:pools}(a,b) with decision-statistic matching (1,000 patients per replicate, recorded clinician, routine path). For the inert drug, the future-inaction pool is biased upwards in all three replicates, by $+0.49$ to $+0.93$\,bpm with every CI above 0. The eligible pool gives $-0.03$ to $-0.25$\,bpm, with CIs excluding 0 at two horizons of one replicate. For the active drug, the future-inaction pool turns the correctly signed effect into a wrong-signed one at 10-60\,s in both replicates ($+0.22$ to $+0.56$, every CI above 0). The eligible pool minus the paired ITT truth is $-0.18$ [$-0.32$, $-0.07$], $-0.02$ [$-0.18$, 0.14], $-0.05$ [$-0.21$, 0.12] and $+0.01$ [$-0.16$, 0.20] in the first replicate, and $+0.02$ [$-0.11$, 0.17], $+0.13$ [$-0.02$, 0.29], $+0.05$ [$-0.11$, 0.23] and $+0.09$ [$-0.09$, 0.29] in the second. In two further exploratory replicates with the active drug, the future-inaction bias relative to the ITT truth at 10\,s is $+0.64$ [0.35, 0.99] and $+0.62$ [0.35, 0.92] on the clinician's own decisions, and $+0.06$ [$-0.11$, 0.22] and $-0.06$ [$-0.23$, 0.10] on exogenous changes, as $b_{\mathrm{FI}}$ predicts: $\pi_F$ is large only near the policy threshold.

\begin{table}[h]
\centering
\caption{Matched contrast of HR (bpm) after a 0.5\,ng\,mL$^{-1}$ remifentanil increase, by control pool; recorded clinician, routine path; one row per independent replicate of 1,000 patients; ITT truth: replicate 1 / replicate 2.}
\label{tab:sim-pools}
\scriptsize
\setlength{\tabcolsep}{3.5pt}
\resizebox{\linewidth}{!}{%
\begin{tabular}{@{}llcccc@{}}
\toprule
Drug & Pool & $h=10$\,s & 30\,s & 60\,s & 110\,s \\
\midrule
Inert (truth 0) & eligible & $-0.05$ [$-0.17$, 0.07] & $-0.03$ [$-0.17$, 0.10] & $-0.11$ [$-0.28$, 0.05] & $-0.04$ [$-0.21$, 0.13] \\
 & & $-0.08$ [$-0.20$, 0.03] & $-0.17$ [$-0.32$, $-0.04$] & $-0.25$ [$-0.42$, $-0.09$] & $-0.18$ [$-0.35$, 0.01] \\
 & & $-0.07$ [$-0.19$, 0.06] & $-0.10$ [$-0.24$, 0.05] & $-0.17$ [$-0.34$, 0.00] & $-0.17$ [$-0.36$, 0.02] \\
 & future-inaction & $+0.58$ [0.43, 0.74] & $+0.89$ [0.72, 1.07] & $+0.93$ [0.72, 1.14] & $+0.82$ [0.59, 1.09] \\
 & & $+0.50$ [0.38, 0.63] & $+0.63$ [0.47, 0.78] & $+0.64$ [0.45, 0.81] & $+0.50$ [0.29, 0.73] \\
 & & $+0.49$ [0.34, 0.63] & $+0.67$ [0.51, 0.84] & $+0.72$ [0.52, 0.93] & $+0.60$ [0.36, 0.82] \\
\addlinespace
Active & ITT truth & $-0.11$ / $-0.11$ & $-0.24$ / $-0.24$ & $-0.37$ / $-0.38$ & $-0.50$ / $-0.52$ \\
 & eligible & $-0.26$ [$-0.39$, $-0.14$] & $-0.27$ [$-0.41$, $-0.11$] & $-0.44$ [$-0.59$, $-0.27$] & $-0.51$ [$-0.69$, $-0.33$] \\
 & & $-0.08$ [$-0.20$, 0.05] & $-0.13$ [$-0.28$, 0.01] & $-0.34$ [$-0.51$, $-0.19$] & $-0.46$ [$-0.66$, $-0.27$] \\
 & future-inaction & $+0.24$ [0.10, 0.37] & $+0.36$ [0.20, 0.52] & $+0.22$ [0.04, 0.42] & $-0.05$ [$-0.27$, 0.16] \\
 & & $+0.43$ [0.30, 0.57] & $+0.56$ [0.39, 0.73] & $+0.42$ [0.23, 0.62] & $+0.07$ [$-0.17$, 0.32] \\
\bottomrule
\end{tabular}}
\end{table}

\paragraph{Sign set by the policy.}
Table~\ref{tab:sim-policy} tests Proposition~\ref{prop:pools}(c) with the inert drug. It reports the paired difference between the future-inaction and eligible pools, multiplied by the event direction $d$. As predicted, the difference is positive for the standard policy of the recorded clinician ($\sigma_\pi=+1$) and negative for the reversed policy ($\sigma_\pi=-1$), in both event directions. The exceptions are three cells whose CIs cover 0: reversed-policy increases at 60 and 110\,s and standard-policy decreases at 110\,s. Among reversed-policy increases, 90\% are exogenous changes, and among the clinician's own decisions the predicted sign holds for both policies in both directions.

\begin{table}[h]
\centering
\caption{Future-inaction minus eligible pool, times event direction (bpm); inert drug, routine path, 1,000 patients per clinician.}
\label{tab:sim-policy}
\scriptsize
\setlength{\tabcolsep}{3.5pt}
\resizebox{\linewidth}{!}{%
\begin{tabular}{@{}llcccc@{}}
\toprule
Policy & Event & $h=10$\,s & 30\,s & 60\,s & 110\,s \\
\midrule
Standard ($\sigma_\pi=+1$) & increase & $+0.63$ [0.53, 0.74] & $+0.93$ [0.80, 1.06] & $+1.04$ [0.88, 1.19] & $+0.86$ [0.70, 1.05] \\
 & decrease & $+0.11$ [0.06, 0.17] & $+0.19$ [0.12, 0.26] & $+0.24$ [0.14, 0.35] & $+0.09$ [$-0.03$, 0.22] \\
Reversed ($\sigma_\pi=-1$) & increase & $-0.07$ [$-0.14$, $-0.01$] & $-0.15$ [$-0.24$, $-0.07$] & $-0.09$ [$-0.22$, 0.01] & $+0.05$ [$-0.08$, 0.18] \\
 & decrease & $-0.44$ [$-0.53$, $-0.36$] & $-0.54$ [$-0.66$, $-0.43$] & $-0.58$ [$-0.72$, $-0.46$] & $-0.43$ [$-0.59$, $-0.26$] \\
\bottomrule
\end{tabular}}
\end{table}

\paragraph{Unrecorded information and the alarm path.}
When the clinician also reacts to the surgical field or plan, assumption (i) of Proposition~\ref{prop:pools} fails, and the eligible pool is biased with the inert drug. The estimates are $+0.30$ to $+0.38$\,bpm (recorded+field) and $+0.35$ to $+0.68$ (recorded+field+plan) across the four horizons, all CIs above 0. The future-inaction pool adds a further $+0.34$ to $+0.83$ (paired differences, all CIs above 0). On the alarm path, where overlap fails, the eligible pool gives $-1.72$ [$-1.97$, $-1.48$] to $-2.48$ [$-2.86$, $-2.10$] with the inert drug, with a median matching distance of 0.13 against 0.006 on the routine path. The future-inaction pool gives $+1.07$ to $+1.95$ there: on this path, excluding later-treated controls adds a positive bias ($+3.37$ to $+4.13$, paired differences) to the negative overlap bias.

\paragraph{Sequential-trial contrast.}
Table~\ref{tab:sim-seqtrial} compares the pooled sequential-trial contrast with the same contrast computed within patients. Same-patient controls matched on $s^{\mathrm{rec}}$ precede their event in 67-79\% of matches (median offset $-7$ to $-10$\,s), which is the mechanism described in App.~\ref{app:theory-pools}. Across both policies and both directions with the inert drug (16 cells), the pooled contrast's CI covers 0 in 16 cells and the within-patient contrast's in 11. With the active drug, matching each event to eligible controls of \emph{other} patients gives estimates whose difference from the paired ITT truth has a CI covering 0 in 16 of 16 cells. The grid implementation of this contrast in ClosedLoopBench approximates the decision statistic with 2-s or 30-s numerics. It did not pass its pre-specified simulator check (relations R01, R02 and R05-R08 of App.~\ref{app:benchmark} at their primary horizons; recorded and reversed clinicians; two replicates of 800 patients; 24 cells per resolution). With remifentanil and ventilator inert, at most 10\% of cells were to have a CI excluding 0; 6 of 24 did at 2\,s and 19 of 24 at 30\,s. With both active, no cell was to have a wrong-signed CI; 2 did at 30\,s.

\begin{table}[h]
\centering
\caption{Sequential-trial contrast of HR (bpm); recorded clinician, 1,000 patients. Inert drug (truth 0): pooled over patients versus within patient. Active drug: other-patient matched estimate minus paired ITT truth.}
\label{tab:sim-seqtrial}
\scriptsize
\setlength{\tabcolsep}{3.5pt}
\resizebox{\linewidth}{!}{%
\begin{tabular}{@{}lllcccc@{}}
\toprule
Drug & Event & Contrast & $h=10$\,s & 30\,s & 60\,s & 110\,s \\
\midrule
Inert & increase & pooled & $+0.01$ [$-0.18$, 0.20] & $-0.06$ [$-0.31$, 0.17] & $-0.19$ [$-0.44$, 0.08] & $-0.12$ [$-0.42$, 0.17] \\
 & & within patient & $-0.20$ [$-0.40$, 0.00] & $-0.28$ [$-0.52$, $-0.04$] & $-0.39$ [$-0.66$, $-0.12$] & $-0.29$ [$-0.56$, 0.01] \\
 & decrease & pooled & $+0.05$ [$-0.23$, 0.32] & $+0.19$ [$-0.18$, 0.56] & $+0.26$ [$-0.21$, 0.70] & $-0.28$ [$-0.77$, 0.21] \\
 & & within patient & $+0.15$ [$-0.15$, 0.45] & $+0.40$ [0.02, 0.81] & $+0.47$ [0.01, 0.89] & $-0.19$ [$-0.64$, 0.27] \\
\addlinespace
Active & increase & other patients $-$ ITT & $+0.03$ [$-0.23$, 0.29] & $+0.17$ [$-0.11$, 0.46] & $+0.23$ [$-0.05$, 0.51] & $+0.24$ [$-0.09$, 0.59] \\
 & decrease & other patients $-$ ITT & $-0.17$ [$-0.43$, 0.07] & $-0.08$ [$-0.40$, 0.26] & $-0.05$ [$-0.43$, 0.35] & $-0.10$ [$-0.52$, 0.34] \\
\bottomrule
\end{tabular}}
\end{table}

\paragraph{Beat-level couplings.}
For a physiological coupling such as the baroreflex (MAP$\to$HR, planted gain $g_{hm}$), the residualized ARX removes slow components before fitting: it subtracts a centered moving average of width $W$ from every series, regresses HR on its own lags and on MAP lags (six at 1\,s, standing in for beat series; four at 2\,s), and reports the median over patients of the static gain. The slow driver and the policy leave a residual bias that hardly depends on $g_{hm}$ (Table~\ref{tab:sim-beta}). The sign is therefore identified only when $|g_{hm}|$ exceeds the bias. Across both scales and five widths ($W$ from 15 to 242\,s), predicting the sign as $\sign(g_{hm}+\hat B)$, with $\hat B$ the bias of the $g_{hm}=-0.4$ cell with the same $\tau_d$, scale and $W$, is correct in 121 of the 134 non-ambiguous cells ($|g_{hm}+\hat B|\ge0.02$; 90.3\%).

\begin{table}[h]
\centering
\caption{Bias (median estimate minus the median patient-level planted gain, bpm/mmHg) of the residualized ARX at 1\,s ($W=61$\,s; recorded+field+plan clinician; 60 patients per cell), by driver and sympathetic time constants $(\tau_d,\tau_s)$. R/X: the estimate's CI excludes 0 with the correct/wrong sign.}
\label{tab:sim-beta}
\small
\begin{tabular}{@{}lccccc@{}}
\toprule
$g_{hm}$ & (20, 2)\,s & (60, 6)\,s & (200, 20)\,s & (600, 60)\,s & (2000, 200)\,s \\
\midrule
$-0.40$ & $+0.270$ R & $+0.258$ R & $+0.095$ R & $+0.018$ R & $+0.016$ R \\
$-0.20$ & $+0.302$ X & $+0.255$ X & $+0.099$ R & $+0.014$ R & $+0.005$ R \\
$-0.10$ & $+0.279$ X & $+0.257$ X & $+0.093$ R & $+0.017$ R & $-0.002$ R \\
$-0.05$ & $+0.281$ X & $+0.255$ X & $+0.087$ X & $+0.011$ R & $-0.005$ R \\
\bottomrule
\end{tabular}
\end{table}

\paragraph{Pre-specified hypotheses.}
Table~\ref{tab:sim-hyp} lists the 14 pre-specified hypotheses of the simulation study; nine hold.

\begin{table}[h]
\centering
\caption{Pre-specified simulation hypotheses (1-s decision-statistic matching unless stated; inert drug unless stated).}
\label{tab:sim-hyp}
\scriptsize
\begin{tabular}{@{}lp{0.58\linewidth}p{0.3\linewidth}@{}}
\toprule
ID & Hypothesis & Outcome \\
\midrule
L1 & Eligible pool, recorded clinician, increases: CI covers 0 at all horizons & holds \\
L2 & Future-inaction pool, same cells: CI $>0$ at 30, 60, 110\,s & holds \\
L3 & Pool difference $\times d$: $>0$ for standard and $<0$ for reversed policy, increases, 30-110\,s & fails at 60, 110\,s (reversed) \\
L3b & Same for decreases & fails at 110\,s (standard) \\
L4 & Active drug: eligible pool minus ITT covers 0 at all horizons; future-inaction bias $>0$ at $\ge2$ horizons & fails at 10\,s ($-0.18$ [$-0.32$, $-0.07$]) \\
L5 & Field or plan information: eligible pool CI $>0$ and pool difference $>0$ at $\ge2$ of 4 horizons & holds (4/4) \\
L6 & Alarm path: eligible pool CI excludes 0 at all horizons & holds \\
E1 & 2-s and 1-s kink with eligible pool, recorded clinician: CI covers 0 at every horizon & fails in 3 of 8 cells (e.g., $+0.32$ [0.15, 0.48] at 10\,s, 2\,s) \\
E2 & Kink: future-inaction minus eligible $>0$ at 10-60\,s for all three clinicians & holds (18/18) \\
E3 & 30-s direct ARX: CI excludes 0 at $\ge1$ horizon for each clinician & holds \\
B1 & Beat/2-s ARX bias range across $g_{hm}$ $\le0.06$ at $W\approx61$\,s in every $(\tau_d,\text{scale})$ cell & fails (0.063 in 1 of 10) \\
B2 & $\sign(g_{hm}+\hat B)$ predicts the sign in $\ge90\%$ of non-ambiguous cells & holds (121/134) \\
D1 & An HR display latency of 12\,s relative to MAP turns the residualized 2-s ARX(4) MAP$\to$HR estimate from negative to positive & holds \\
D2 & The estimate increases monotonically with the latency (Spearman $\ge0.9$) & holds \\
\bottomrule
\end{tabular}
\end{table}

D1 and D2 hold for four lags; with eight lags the estimate stays negative at latencies of 8 and 12\,s.

\section{Typed concept relations}
\label{app:relations}

This appendix specifies the slow concept-level relations of our graphs (Sec.~\ref{sec:method-slow}) and gives the full evidence for the concept-level results of Sec.~\ref{sec:exp-graphs}. Vital-level relations come from the typed pooled estimator of App.~\ref{app:benchmark}. Beat-level couplings are described in App.~\ref{app:fast}.

\subsection{Data, concept states and actions}
\label{app:relations-data}

We use all 3{,}442 VitalDB cases \citep{lee2022vitaldb}. Monitor numerics are summarized in 30\,s windows at a 15\,s stride ($\Delta=15$\,s), which gives 3{,}031{,}582 steps. The analysis uses eight concepts $c$ on four vitals, defined in Table~\ref{tab:concept-vocab}: tachycardia and bradycardia (HR), hypotension and hypertension (MAP), tachypnea and bradypnea (respiratory rate), and hypercapnia and hypocapnia (EtCO$_2$). The state $S_{p,c}(t)$ of concept $c$ in patient $p$ at step $t$ is 1 (on) or 0 (off). We write $S_{p,c}(t)=\ast$ when the defining vital is unobserved.

\textbf{Onsets and risk sets.} Concept $z$ has an onset at step $t$ if $S_{p,z}(t)=1$ and $S_{p,z}(t-1)=S_{p,z}(t-2)=0$. Step $t$ is at risk for an onset of $z$ if $S_{p,z}(t)\neq\ast$ and $S_{p,z}(t-1)=S_{p,z}(t-2)=0$.

\textbf{Actions.} The analysis uses four actions $a$: the propofol and remifentanil target-controlled-infusion targets, the phenylephrine infusion rate, and the set ventilator respiratory rate (vent.\ rate in tables). With the direction $d\in\{\uparrow,\downarrow\}$ of a change, they give eight action-directions $(a,d)$. A change $\tau\in\mathcal{E}_{p,a}$ is assigned to step $t$ if it occurs in $[15, 30)$\,s after the start of window $t$, that is, after the end of window $t-1$. Changes from or to zero (starts and stops) are excluded for the infusion targets and the ventilator rate and included for phenylephrine. Dose covariates are $\log(1+\text{value})$ of the propofol and remifentanil targets and of the phenylephrine and norepinephrine rates, and the set ventilator rate divided by 10. A missing value is coded 0, with a not-recorded indicator for the two targets and the ventilator rate, and doses are never carried forward past the end of a track.

\subsection{Pooled hazard model}
\label{app:relations-model}

Let $Y_p(t)\in\{0,1\}$ be the outcome of a relation family at step $t$ and $\mathcal{R}_p$ its at-risk steps. The outcome is an onset of a target concept $z$ (association and response families) or a change of action $a$ in direction $d$ (policy family). For each outcome, one pooled discrete-time logistic hazard is fitted across patients:
\begin{equation}
\operatorname{logit}\Pr\bigl(Y_p(t)=1 \mid t\in\mathcal{R}_p\bigr)=\alpha_p+\sum_{x\in\mathcal{X}}\beta_x^\top e_{p,x}(t)+\gamma^\top w_p(t).
\label{eq:hazard-full}
\end{equation}
Here $\alpha_p$ is a patient fixed effect, $\mathcal{X}$ is the set of candidate sources of the family (all entered jointly), $e_{p,x}(t)\in\{0,1\}^{m_x}$ is the vector of $m_x$ lag-bin exposure indicators of source $x$ (four for a concept source, two for an action source), $\beta_x$ are their coefficients (for $r=(x\to z)$, the $\beta_r$ of Eq.~\ref{eq:hazard}), and $w_p(t)$ are adjustment covariates with coefficients $\gamma$. The relation $r=(x\to z)$ is summarized by the average conditional log odds ratio over exposed rows,
\begin{equation}
\omega_r=\bar e_x^\top\beta_x ,
\label{eq:omega}
\end{equation}
where $\bar e_x$ is the (sampling-weighted) mean of $e_{p,x}(t)$ over at-risk rows with the source exposed. A concept source is exposed when it is on at the source lag; an action source is exposed when it has a change in the response window. Thus $\exp(\omega_r)$ is the conditional odds ratio of the discrete-time hazard while the source is present. With a single source ($\mathcal{X}=\{x\}$), Eq.~\ref{eq:hazard-full} is Eq.~\ref{eq:hazard}; if $e_{p,x}$ is a single indicator, $\omega_r=\beta_x$. Tables report the estimate $\hat\omega_r$. Table~\ref{tab:rel-families} lists the three families; there and in App.~\ref{app:graphs}, $A$ and $B$ denote the source $x$ and the target $z$.

\begin{table}[!htbp]
\centering
\caption{Relation families of the pooled hazard (Eq.~\ref{eq:hazard-full}). Lags in 15\,s steps.}
\label{tab:rel-families}
\scriptsize
\setlength{\tabcolsep}{3pt}
\begin{tabular}{@{}p{0.16\linewidth}p{0.2\linewidth}p{0.29\linewidth}p{0.29\linewidth}@{}}
\toprule
Family & Outcome; at risk & Exposure $e_{p,x}(t)$ per source & Adjustments $w_p(t)$ \\
\midrule
Association $A\to B$; 48 ordered pairs on different vitals & Onset of $B$; risk set of $B$ & $A$ on at lag 2 (windows do not overlap); $A$ onset in lags $[2,4]$, $[5,12]$, $[13,40]$ (0-1, 1-3, 3-10\,min). Sources: the six concepts on other vitals & Phase; $B$ history; doses at lag 1; a change of each action-direction in the last 10\,min; source unobserved at lag 2 \\
\addlinespace
Policy $A\to(a,d)$; 8 concepts $\times$ 8 action-directions & Change of $a$ in direction $d$; action running at $t-1$\textsuperscript{a} & $A$ on at lag 1; $A$ onset in lags $[1,4]$, $[5,12]$, $[13,40]$. Sources: all eight concepts & Phase; time since the last change of $a$ (never, $\le$1, 1-5, 5-15\,min); doses at lag 1; source unobserved at lag 1 \\
\addlinespace
Response $(a,d)\to B$; 8 action-directions $\times$ 8 concepts & Onset of $B$; risk set of $B$ & Change of $(a,d)$ in lags $[5,12]$ and $[13,40]$ (1-3 and 3-10\,min, at or beyond pharmacological latency). Sources: all eight action-directions & Phase; $B$ history; at lag 41 (pre-treatment): doses, state and unobserved indicator of every concept, $z$-scored value of $B$'s vital and its missing indicator \\
\bottomrule
\end{tabular}
\\[2pt]
\raggedright\scriptsize \textsuperscript{a}Infusion target or ventilator rate $>0$; phenylephrine$\uparrow$: every step; phenylephrine$\downarrow$: rate $>0$. Phase: time since anesthesia start (first step with a propofol or remifentanil target or ventilator rate $>0$; before, 0-15\,min, 15-60\,min, 1-3\,h), time since recording start (0-10, 10-30, 30-60\,min, 1-2, 2-4\,h), time to recording end ($\le$15, 15-30, 30-60\,min). $B$ history: time since $B$ was last on (never, $\le$2, 2-10, 10-30\,min).
\end{table}

The estimates $\hat\omega_r$ are conditional associations, not causal effects, and they are retrospective: the phase covariates include the time to recording end (Table~\ref{tab:rel-families}, footnote), which is not known at decision time.

\textbf{Estimation.} Eq.~\ref{eq:hazard-full} is fitted per outcome by Newton-Raphson with the patient intercepts profiled out. Patients without both an event and a non-event carry no within-patient information and are dropped from that fit. A ridge penalty of 1 is applied to the shared coefficients. Every event row is kept; non-event rows are subsampled at rate 0.2 and weighted by 5.

\textbf{Uncertainty and test.} Uncertainty comes from a patient-cluster score bootstrap with 500 multinomial patient resamples. Each replicate is a one-step Newton update from the full-sample fit with the Hessian held fixed. $\widehat{\mathrm{se}}(\hat\omega_r)$ is the bootstrap standard deviation of $\hat\omega_r$, and 95\% CIs are bootstrap percentiles. With the Wald statistic $W_r=\hat\omega_r/\widehat{\mathrm{se}}(\hat\omega_r)$, the primary test is two-sided, $p_{\omega}=2\Phi(-|W_r|)$, where $\Phi$ is the standard normal distribution function. A relation is testable only if it has at least 20 exposed events from at least 10 patients; otherwise $p_\omega=1$.

\subsection{Negative controls and admission}
\label{app:relations-admission}

For each null kind there are 19 replicates $\tilde u^{(k)}$, $k=1,\dots,19$:
\begin{itemize}[leftmargin=*,itemsep=1pt,topsep=2pt]
\item \emph{Cross-patient transplant}: a derangement of patients. The partner supplies the source series at the same elapsed step; beyond the partner's end, concept states are unobserved and action changes absent.
\item \emph{Within-patient circular shift}: the source series is circularly shifted by $k s_p$ steps, with $k$ uniform on $\{1,\dots,\lfloor N^{\mathrm{rec}}_p/s_p\rfloor-1\}$ and $s_p=\min\{\max(41,\lceil\max_c \tau^{\mathrm{int}}_{p,c}\rceil),\lfloor N^{\mathrm{rec}}_p/2\rfloor\}$. Here 41 exceeds the longest lag (40 steps), $\tau^{\mathrm{int}}_{p,c}$ is the integrated autocorrelation time of the on-indicator of concept $c$, and $N^{\mathrm{rec}}_p$ is the record length in steps.
\end{itemize}
In the association and policy families the null replaces all concept source series; in the response family it replaces the action-change series. A relation is \emph{admitted} if (i) it passes Benjamini-Hochberg (BH) control at level 0.05 within its family \citep{benjamini1995controlling}, and (ii) its $W_r^2$ exceeds the maximum $W_r^2$ of the same relation over all 38 null replicates. Its type is its family. When a null replicate itself is scored, rule (ii) uses the other 37 replicates.

\subsection{Admitted relations}
\label{app:relations-admitted}

Table~\ref{tab:rel-counts} gives the counts. Tables~\ref{tab:rel-list-ar} and~\ref{tab:rel-list-p} list every admitted relation with its sign status (App.~\ref{app:graphs-status}). At the population level, of the relations admitted by rule (i) in the fit on the first half of each record (H1), 12 of 28 association, 23 of 33 policy and 9 of 21 response relations are admitted by rule (i) with the same sign in the fit on the second half (H2). The shares of candidates admitted in H2, the rates expected by chance, are 0.479, 0.422 and 0.219.

\textbf{MIMIC-IV-WDB.} An exploratory fit of the association family uses the 51 MIMIC-IV-WDB patients with at least 30\,min of windows (App.~\ref{app:data}). Concept states on the 15\,s grid are the thresholded activations of the concept heads that give the MIMIC-IV-WDB graph nodes (App.~\ref{app:graphs-size}); the capnography concepts are unobserved, and there are no action or dose covariates. Of the 36 candidate relations with a non-capnography target, seven pass rule (i) and six are admitted (Table~\ref{tab:rel-list-ar}). Of the 8 relations admitted by rule (i) in H1, 4 are admitted with the same sign in H2 (chance 0.194).

\begin{table}[!htbp]
\centering
\caption{Concept-level relations on VitalDB. Sign status: sign-stable / sign-unresolved / unconfirmed (App.~\ref{app:graphs-status}).}
\label{tab:rel-counts}
\small
\begin{tabular}{@{}lcccccc@{}}
\toprule
Family & Candidates & Rule (i) & Admitted & $\hat\omega_r>0$ & $\hat\omega_r<0$ & Sign status \\
\midrule
Association & 48 & 31 & 28 & 21 & 7 & 12 / 0 / 16 \\
Policy & 64 & 36 & 33 & 22 & 11 & 23 / 0 / 10 \\
Response & 64 & 25 & 20 & 4 & 16 & 0 / 1 / 19 \\
\bottomrule
\end{tabular}
\end{table}

\begin{table}[!htbp]
\centering
\caption{Admitted association and response relations: $\hat\omega_r$ [95\% CI], exposed events / contributing patients, sign status (App.~\ref{app:graphs-status}).}
\label{tab:rel-list-ar}
\scriptsize
\begin{tabular}{@{}lccc@{}}
\toprule
Relation & $\hat\omega_r$ [95\% CI] & Events / patients & Sign status \\
\midrule
\multicolumn{4}{@{}l}{\emph{Association, VitalDB}}\\
bradycardia $\to$ bradypnea & $-$0.13 [$-$0.21, $-$0.04] & 2{,}910 / 1{,}240 & unconfirmed \\
bradycardia $\to$ hypercapnia & $-$0.70 [$-$0.96, $-$0.38] & 168 / 104 & sign-stable \\
bradycardia $\to$ hypertension & $-$0.34 [$-$0.40, $-$0.29] & 4{,}183 / 1{,}335 & sign-stable \\
bradycardia $\to$ hypocapnia & 0.37 [0.30, 0.45] & 2{,}825 / 1{,}059 & sign-stable \\
bradycardia $\to$ hypotension & 0.43 [0.36, 0.50] & 3{,}958 / 1{,}211 & sign-stable \\
hypercapnia $\to$ bradycardia & $-$0.38 [$-$0.63, $-$0.14] & 124 / 65 & unconfirmed \\
hypercapnia $\to$ bradypnea & 0.40 [0.18, 0.60] & 483 / 374 & unconfirmed \\
hypertension $\to$ bradycardia & $-$0.47 [$-$0.55, $-$0.40] & 1{,}782 / 831 & sign-stable \\
hypertension $\to$ bradypnea & 0.27 [0.19, 0.34] & 2{,}777 / 1{,}521 & unconfirmed \\
hypertension $\to$ hypercapnia & 0.64 [0.44, 0.84] & 714 / 496 & sign-stable \\
hypertension $\to$ tachycardia & 0.44 [0.34, 0.55] & 1{,}448 / 870 & unconfirmed \\
hypertension $\to$ tachypnea & 0.55 [0.40, 0.72] & 1{,}204 / 807 & unconfirmed \\
hypocapnia $\to$ bradycardia & 0.21 [0.12, 0.31] & 1{,}344 / 775 & unconfirmed \\
hypocapnia $\to$ bradypnea & 1.29 [1.21, 1.37] & 4{,}224 / 2{,}020 & sign-stable \\
hypocapnia $\to$ hypertension & 0.10 [0.03, 0.17] & 1{,}476 / 1{,}004 & unconfirmed \\
hypocapnia $\to$ hypotension & 0.20 [0.12, 0.29] & 1{,}149 / 668 & sign-stable \\
hypocapnia $\to$ tachycardia & 0.38 [0.25, 0.51] & 886 / 586 & unconfirmed \\
hypocapnia $\to$ tachypnea & 0.92 [0.77, 1.06] & 1{,}608 / 1{,}140 & sign-stable \\
hypotension $\to$ bradycardia & 0.20 [0.12, 0.28] & 2{,}126 / 1{,}082 & sign-stable \\
hypotension $\to$ bradypnea & 0.17 [0.08, 0.25] & 1{,}914 / 1{,}280 & sign-stable \\
hypotension $\to$ hypocapnia & 0.54 [0.47, 0.62] & 1{,}913 / 1{,}087 & sign-stable \\
hypotension $\to$ tachypnea & 0.27 [0.10, 0.48] & 641 / 546 & unconfirmed \\
tachycardia $\to$ bradypnea & 0.23 [0.09, 0.36] & 1{,}219 / 742 & unconfirmed \\
tachycardia $\to$ tachypnea & 0.28 [0.06, 0.49] & 625 / 384 & unconfirmed \\
tachypnea $\to$ bradycardia & $-$0.53 [$-$0.76, $-$0.25] & 157 / 98 & unconfirmed \\
tachypnea $\to$ hypertension & 0.35 [0.18, 0.50] & 529 / 416 & unconfirmed \\
tachypnea $\to$ hypotension & $-$0.33 [$-$0.53, $-$0.13] & 259 / 135 & unconfirmed \\
tachypnea $\to$ tachycardia & 0.67 [0.48, 0.88] & 437 / 321 & unconfirmed \\
\addlinespace
\multicolumn{4}{@{}l}{\emph{Response, VitalDB}}\\
propofol$\downarrow$ $\to$ bradypnea & $-$0.28 [$-$0.35, $-$0.20] & 1{,}505 / 863 & unconfirmed \\
propofol$\downarrow$ $\to$ hypocapnia & $-$0.19 [$-$0.27, $-$0.12] & 1{,}401 / 748 & unconfirmed \\
propofol$\downarrow$ $\to$ tachypnea & $-$0.37 [$-$0.58, $-$0.16] & 180 / 141 & unconfirmed \\
propofol$\uparrow$ $\to$ bradycardia & $-$0.19 [$-$0.27, $-$0.12] & 1{,}954 / 724 & unconfirmed \\
propofol$\uparrow$ $\to$ tachycardia & 0.26 [0.17, 0.35] & 1{,}130 / 508 & unconfirmed \\
remifentanil$\downarrow$ $\to$ bradycardia & $-$0.07 [$-$0.12, $-$0.03] & 4{,}784 / 1{,}480 & unconfirmed \\
remifentanil$\downarrow$ $\to$ bradypnea & $-$0.13 [$-$0.20, $-$0.08] & 2{,}501 / 1{,}400 & unconfirmed \\
remifentanil$\downarrow$ $\to$ hypotension & 0.11 [0.06, 0.15] & 5{,}206 / 1{,}790 & unconfirmed \\
remifentanil$\downarrow$ $\to$ tachycardia & $-$0.15 [$-$0.22, $-$0.07] & 1{,}835 / 769 & unconfirmed \\
remifentanil$\downarrow$ $\to$ tachypnea & $-$0.78 [$-$0.95, $-$0.62] & 281 / 236 & unconfirmed \\
remifentanil$\uparrow$ $\to$ bradypnea & $-$0.08 [$-$0.15, $-$0.01] & 2{,}090 / 1{,}125 & unconfirmed \\
remifentanil$\uparrow$ $\to$ hypertension & 0.07 [0.03, 0.11] & 5{,}759 / 1{,}933 & unconfirmed \\
remifentanil$\uparrow$ $\to$ hypocapnia & $-$0.15 [$-$0.21, $-$0.09] & 1{,}923 / 973 & unconfirmed \\
remifentanil$\uparrow$ $\to$ hypotension & $-$0.21 [$-$0.26, $-$0.16] & 3{,}512 / 1{,}401 & sign-unresolved \\
remifentanil$\uparrow$ $\to$ tachypnea & $-$0.32 [$-$0.48, $-$0.15] & 297 / 227 & unconfirmed \\
vent.\ rate$\downarrow$ $\to$ bradypnea & 0.25 [0.17, 0.34] & 1{,}674 / 1{,}005 & unconfirmed \\
vent.\ rate$\downarrow$ $\to$ hypocapnia & $-$0.36 [$-$0.43, $-$0.28] & 1{,}874 / 962 & unconfirmed \\
vent.\ rate$\downarrow$ $\to$ tachycardia & $-$0.13 [$-$0.21, $-$0.04] & 1{,}067 / 567 & unconfirmed \\
vent.\ rate$\uparrow$ $\to$ bradycardia & $-$0.17 [$-$0.23, $-$0.12] & 3{,}005 / 1{,}224 & unconfirmed \\
vent.\ rate$\uparrow$ $\to$ bradypnea & $-$0.72 [$-$0.79, $-$0.63] & 1{,}622 / 948 & unconfirmed \\
\addlinespace
\multicolumn{4}{@{}l}{\emph{Association, MIMIC-IV-WDB (51 patients)}}\\
bradycardia $\to$ hypertension & 0.64 [0.44, 0.85] & 3{,}662 / 43 & sign-stable \\
bradypnea $\to$ tachycardia & 0.34 [0.21, 0.48] & 6{,}565 / 49 & unconfirmed \\
hypotension $\to$ tachycardia & $-$0.25 [$-$0.41, $-$0.10] & 3{,}914 / 48 & sign-stable \\
tachycardia $\to$ hypertension & 0.58 [0.39, 0.75] & 3{,}777 / 45 & sign-stable \\
tachycardia $\to$ hypotension & $-$0.22 [$-$0.35, $-$0.09] & 5{,}436 / 49 & unconfirmed \\
tachypnea $\to$ tachycardia & 0.37 [0.23, 0.52] & 2{,}845 / 46 & sign-stable \\
\bottomrule
\end{tabular}
\end{table}

\begin{table}[!htbp]
\centering
\caption{Admitted policy relations on VitalDB; columns as in Table~\ref{tab:rel-list-ar}.}
\label{tab:rel-list-p}
\scriptsize
\begin{tabular}{@{}lccc@{}}
\toprule
Relation & $\hat\omega_r$ [95\% CI] & Events / patients & Sign status \\
\midrule
bradycardia $\to$ propofol$\uparrow$ & $-$0.23 [$-$0.33, $-$0.14] & 1{,}265 / 565 & unconfirmed \\
bradycardia $\to$ remifentanil$\downarrow$ & 0.23 [0.16, 0.29] & 4{,}321 / 1{,}283 & sign-stable \\
bradycardia $\to$ remifentanil$\uparrow$ & $-$0.44 [$-$0.50, $-$0.38] & 3{,}293 / 1{,}065 & sign-stable \\
bradycardia $\to$ vent.\ rate$\downarrow$ & 0.26 [0.18, 0.35] & 2{,}004 / 1{,}041 & sign-stable \\
bradycardia $\to$ vent.\ rate$\uparrow$ & $-$0.49 [$-$0.57, $-$0.41] & 1{,}885 / 926 & sign-stable \\
bradypnea $\to$ propofol$\downarrow$ & 0.24 [0.15, 0.34] & 1{,}299 / 801 & unconfirmed \\
bradypnea $\to$ remifentanil$\uparrow$ & 0.10 [0.03, 0.18] & 3{,}458 / 1{,}092 & unconfirmed \\
bradypnea $\to$ vent.\ rate$\downarrow$ & $-$0.31 [$-$0.41, $-$0.19] & 813 / 615 & sign-stable \\
bradypnea $\to$ vent.\ rate$\uparrow$ & 0.33 [0.24, 0.41] & 2{,}960 / 1{,}690 & sign-stable \\
hypercapnia $\to$ vent.\ rate$\uparrow$ & 1.73 [1.53, 1.94] & 350 / 219 & sign-stable \\
hypertension $\to$ phenylephrine$\downarrow$ & 1.44 [1.10, 1.79] & 75 / 35 & sign-stable \\
hypertension $\to$ propofol$\downarrow$ & $-$0.16 [$-$0.26, $-$0.06] & 760 / 506 & unconfirmed \\
hypertension $\to$ propofol$\uparrow$ & 0.94 [0.84, 1.03] & 1{,}251 / 627 & sign-stable \\
hypertension $\to$ remifentanil$\downarrow$ & $-$0.57 [$-$0.64, $-$0.50] & 1{,}397 / 876 & sign-stable \\
hypertension $\to$ remifentanil$\uparrow$ & 1.49 [1.43, 1.54] & 5{,}462 / 1{,}575 & sign-stable \\
hypertension $\to$ vent.\ rate$\downarrow$ & $-$0.26 [$-$0.36, $-$0.15] & 534 / 440 & unconfirmed \\
hypertension $\to$ vent.\ rate$\uparrow$ & 0.46 [0.39, 0.53] & 1{,}412 / 930 & sign-stable \\
hypocapnia $\to$ propofol$\downarrow$ & 0.34 [0.24, 0.45] & 664 / 506 & sign-stable \\
hypocapnia $\to$ remifentanil$\downarrow$ & 0.23 [0.13, 0.33] & 1{,}111 / 773 & unconfirmed \\
hypocapnia $\to$ vent.\ rate$\downarrow$ & 2.48 [2.39, 2.58] & 2{,}416 / 1{,}376 & sign-stable \\
hypocapnia $\to$ vent.\ rate$\uparrow$ & $-$0.40 [$-$0.53, $-$0.25] & 363 / 304 & sign-stable \\
hypotension $\to$ phenylephrine$\uparrow$ & 1.47 [1.21, 1.74] & 162 / 67 & sign-stable \\
hypotension $\to$ propofol$\downarrow$ & 0.74 [0.65, 0.82] & 1{,}704 / 911 & sign-stable \\
hypotension $\to$ propofol$\uparrow$ & $-$0.42 [$-$0.53, $-$0.32] & 612 / 435 & sign-stable \\
hypotension $\to$ remifentanil$\downarrow$ & 1.16 [1.10, 1.23] & 3{,}543 / 1{,}557 & sign-stable \\
hypotension $\to$ remifentanil$\uparrow$ & $-$0.95 [$-$1.04, $-$0.86] & 804 / 581 & sign-stable \\
hypotension $\to$ vent.\ rate$\downarrow$ & 0.38 [0.28, 0.46] & 1{,}193 / 816 & sign-stable \\
hypotension $\to$ vent.\ rate$\uparrow$ & $-$0.28 [$-$0.37, $-$0.18] & 744 / 570 & sign-stable \\
tachycardia $\to$ propofol$\downarrow$ & 0.18 [0.06, 0.31] & 518 / 270 & unconfirmed \\
tachycardia $\to$ propofol$\uparrow$ & 0.38 [0.25, 0.50] & 487 / 265 & unconfirmed \\
tachycardia $\to$ remifentanil$\uparrow$ & 0.55 [0.45, 0.64] & 1{,}461 / 554 & sign-stable \\
tachycardia $\to$ vent.\ rate$\uparrow$ & 0.36 [0.23, 0.47] & 790 / 420 & unconfirmed \\
tachypnea $\to$ remifentanil$\downarrow$ & 0.49 [0.27, 0.70] & 211 / 160 & unconfirmed \\
\bottomrule
\end{tabular}
\end{table}

\subsection{Pre-specified sign checks}
\label{app:relations-signs}

Table~\ref{tab:rel-signs} lists the sign checks that were fixed before estimation. A check is \emph{correct} if the 95\% CI of $\omega_r$ excludes 0 in the expected direction, \emph{wrong} if it excludes 0 in the opposite direction, and \emph{inconclusive} otherwise. All five primary policy checks are correct, and so are the two secondary policy checks, which were added after the exploratory pilots (App.~\ref{app:prereg}). None of the six response checks is correct: two are wrong and four are inconclusive. Because the response family fails these checks, none of its relations is marked sign-stable (App.~\ref{app:graphs-status}; Table~\ref{tab:rel-counts}). App.~\ref{app:theory} analyzes a mechanism by which clinician feedback can reverse or mask response signs.

\begin{table}[!htbp]
\centering
\caption{Pre-specified sign checks for concept-level relations on VitalDB. Expected: expected sign of $\omega_r$.}
\label{tab:rel-signs}
\scriptsize
\begin{tabular}{@{}lccccc@{}}
\toprule
Relation & Expected & $\exp(\hat\omega_r)$ & $\hat\omega_r$ [95\% CI] & Call & Admitted \\
\midrule
\multicolumn{6}{@{}l}{\emph{Response checks}}\\
propofol$\uparrow$ $\to$ hypotension & $+$ & 1.03 & 0.025 [$-$0.036, 0.086] & inconclusive & no \\
remifentanil$\uparrow$ $\to$ bradycardia & $+$ & 1.02 & 0.020 [$-$0.032, 0.071] & inconclusive & no \\
remifentanil$\uparrow$ $\to$ hypotension & $+$ & 0.81 & $-$0.208 [$-$0.259, $-$0.158] & wrong & yes \\
vent.\ rate$\uparrow$ $\to$ hypocapnia & $+$ & 0.90 & $-$0.103 [$-$0.169, $-$0.030] & wrong & no \\
vent.\ rate$\uparrow$ $\to$ hypercapnia & $-$ & 0.89 & $-$0.111 [$-$0.294, 0.090] & inconclusive & no \\
phenylephrine$\uparrow$ $\to$ hypotension & $-$ & 1.00 & $-$0.004 [$-$0.206, 0.204] & inconclusive & no \\
\addlinespace
\multicolumn{6}{@{}l}{\emph{Policy checks}}\\
hypertension $\to$ remifentanil$\uparrow$ & $+$ & 4.42 & 1.487 [1.433, 1.540] & correct & yes \\
tachycardia $\to$ remifentanil$\uparrow$ & $+$ & 1.73 & 0.550 [0.451, 0.637] & correct & yes \\
hypotension $\to$ phenylephrine$\uparrow$ & $+$ & 4.35 & 1.470 [1.214, 1.739] & correct & yes \\
hypotension $\to$ propofol$\downarrow$ & $+$ & 2.10 & 0.741 [0.653, 0.819] & correct & yes \\
hypercapnia $\to$ vent.\ rate$\uparrow$ & $+$ & 5.67 & 1.735 [1.529, 1.939] & correct & yes \\
\addlinespace
\multicolumn{6}{@{}l}{\emph{Secondary policy checks}}\\
hypocapnia $\to$ vent.\ rate$\downarrow$ & $+$ & 11.94 & 2.480 [2.390, 2.580] & correct & yes \\
hypertension $\to$ propofol$\uparrow$ & $+$ & 2.57 & 0.943 [0.838, 1.034] & correct & yes \\
\bottomrule
\end{tabular}
\end{table}

\subsection{Null admissions at the population level}
\label{app:relations-nulls}

Table~\ref{tab:rel-nulls} reports, for each family and null kind, the mean number of null relations admitted per replicate and the share of testable null $p_\omega$ below 0.05. Summed over the three families (176 candidate relations), each replicate admits on average 0.42 null relations under the transplant null and 0.63 under the shift null. At least one null relation is admitted in 7 of 19 transplant and 9 of 19 shift replicates. Testable null $p_\omega$ fall below 0.05 in 5.7-7.9\% of tests per family and null kind, above the nominal 5\%. App.~\ref{app:graphs-fe} turns these admissions into false relation instances per patient graph.

\begin{table}[!htbp]
\centering
\caption{Null admissions of concept-level relations on VitalDB (19 replicates per null kind). Last column: share of testable null $p_\omega<0.05$ (number testable).}
\label{tab:rel-nulls}
\footnotesize
\begin{tabular}{@{}llccc@{}}
\toprule
Family & Null & Admitted per replicate & Replicates with $\ge$1 & $p_\omega<0.05$ \\
\midrule
Association & transplant & 0.21 & 3 & 0.059 (892) \\
Association & shift & 0.11 & 2 & 0.067 (912) \\
Policy & transplant & 0.05 & 1 & 0.060 (1{,}036) \\
Policy & shift & 0.26 & 5 & 0.079 (1{,}136) \\
Response & transplant & 0.16 & 3 & 0.057 (1{,}121) \\
Response & shift & 0.26 & 5 & 0.073 (1{,}160) \\
\bottomrule
\end{tabular}
\end{table}

\subsection{Within-patient replication against presence-matched chance}
\label{app:relations-replication}

Each record is split into its first half H1 and second half H2. For patient $p$, half $H\in\{1,2\}$ and a given method, let $G_{p,H}$ be the set of relations in $p$'s half-$H$ graph. Let $\mathcal{P}_{p,H}$ be the set of candidate pairs whose endpoints are present in half $H$. For CKG and the per-patient baselines, both concepts must have at least one episode. For our graphs, the requirement depends on the family: association needs a source episode and a target onset; policy needs a source episode and a change of $(a,d)$; response needs a change of $(a,d)$ and a target onset. Replication and presence-matched chance are
\begin{equation}
\mathrm{Rep}=\frac{\sum_p |G_{p,1}\cap G_{p,2}|}{\sum_p |G_{p,1}|},\qquad
\mathrm{Rep}_0=\frac{\sum_p |G_{p,1}|\,|G_{p,2}|\,|\mathcal{P}_{p,1}\cap\mathcal{P}_{p,2}|\,/\,(|\mathcal{P}_{p,1}|\,|\mathcal{P}_{p,2}|)}{\sum_p |G_{p,1}|}.
\label{eq:replication}
\end{equation}
Thus $\mathrm{Rep}_0$ is the expected replication if each half's relations were a uniformly random subset, of the same size, of that half's present pairs. Direction must match, and sign is ignored. For our graphs, $G_{p,H}$ contains the relations admitted by rule (i) in the pooled half-$H$ fit and instantiated in $p$ because both endpoints are present in that half. Rule (ii) is not applied, because there are no half-specific null replicates. Replication, chance and their difference have 95\% CIs from a patient bootstrap with 1{,}000 draws.

The per-patient baselines are applied to the same concept series and the same 48 cross-vital pairs. \emph{Pairwise Granger + BH} is an ordinary least-squares $F$-test on the 15\,s grid of the state of $B$ on its own lags 1 and 2 and its sums over lag bins 3-4, 5-12 and 13-40, plus lag 2 and the same lag-bin sums of $A$, with Benjamini-Hochberg control at 0.05 over the patient's 48 pairs. \emph{PCMCI+} \citep{runge2020discovering} uses partial-correlation tests on 1-min block means with lags 1-10 and significance level 0.01; its per-pair $p$-value is the Bonferroni-over-lags minimum, followed by BH at 0.05 over the 48 pairs. It is run on a random subset of 400 cases. Table~\ref{tab:rel-replication} gives the results.

\begin{table}[!htbp]
\centering
\caption{Within-patient H1$\to$H2 replication ($\mathrm{Rep}$) and presence-matched chance ($\mathrm{Rep}_0$). $n_1$: relations in H1 graphs (per patient).}
\label{tab:rel-replication}
\scriptsize
\setlength{\tabcolsep}{3pt}
\begin{tabular}{@{}lcccc@{}}
\toprule
Method & $n_1$ & $\mathrm{Rep}$ [95\% CI] & $\mathrm{Rep}_0$ [95\% CI] & $\mathrm{Rep}-\mathrm{Rep}_0$ [95\% CI] \\
\midrule
\multicolumn{5}{@{}l}{\emph{VitalDB (3{,}442 patients)}}\\
CKG cross-correlation & 26{,}988 (7.84) & 0.185 [0.179, 0.191] & 0.184 [0.180, 0.188] & 0.001 [$-$0.003, 0.006] \\
CKG Granger-precedes & 34{,}499 (10.02) & 0.275 [0.267, 0.282] & 0.262 [0.256, 0.268] & 0.013 [0.009, 0.017] \\
CKG (union) & 47{,}423 (13.78) & 0.352 [0.345, 0.360] & 0.347 [0.340, 0.354] & 0.005 [0.002, 0.008] \\
Pairwise Granger + BH & 12{,}837 (3.73) & 0.177 [0.168, 0.186] & 0.161 [0.155, 0.168] & 0.016 [0.010, 0.021] \\
PCMCI+ (400 cases) & 397 (0.99) & 0.065 [0.040, 0.089] & 0.063 [0.054, 0.072] & 0.002 [$-$0.021, 0.024] \\
Ours, association & 42{,}880 (12.46) & 0.341 [0.334, 0.349] & 0.317 [0.312, 0.323] & 0.024 [0.020, 0.027] \\
Ours, policy & 58{,}416 (16.97) & 0.430 [0.424, 0.437] & 0.344 [0.339, 0.350] & 0.086 [0.084, 0.088] \\
Ours, response & 34{,}209 (9.94) & 0.244 [0.239, 0.249] & 0.160 [0.157, 0.163] & 0.084 [0.081, 0.087] \\
\addlinespace
\multicolumn{5}{@{}l}{\emph{MIMIC-IV-WDB (51 patients)}}\\
CKG cross-correlation & 116 (2.27) & 0.216 [0.092, 0.322] & 0.179 [0.098, 0.246] & n/a \\
CKG (union) & 713 (13.98) & 0.683 [0.619, 0.743] & 0.638 [0.578, 0.695] & n/a \\
Ours, association & 405 (7.94) & 0.491 [0.481, 0.499] & 0.286 [0.278, 0.292] & n/a \\
\bottomrule
\end{tabular}
\end{table}

\subsection{Held-out log-likelihood of concept onsets}
\label{app:relations-ll}

We evaluate the association family on the H2 rows of every patient. All models are fitted on H1 as pooled logistic hazards without patient fixed effects, with the same subsampling and adjustment covariates. The models are:
\begin{itemize}[leftmargin=*,itemsep=1pt,topsep=2pt]
\item \emph{base}: no source terms;
\item \emph{pooled}: all six sources;
\item \emph{typed}: only the sources admitted by rule (i) on H1;
\item \emph{CKG-gated}: the source exposure features multiplied by an indicator that the patient's H1 CKG graph contains the edge;
\item \emph{shrunk-patient}: pooled plus the per-patient deviations below, shrunk by empirical Bayes;
\item \emph{per-patient}: pooled plus the unshrunk deviations $\delta_p$ (App.~\ref{app:relations-deviations}), estimated on H1 by maximum likelihood with a normal penalty of standard deviation 3.
\end{itemize}
The metric is the change in held-out log-likelihood relative to base, in nats per H2 onset (52{,}482 onsets on VitalDB), with 95\% CIs from a patient bootstrap with 2{,}000 draws.

Each source also carries an indicator that it is unobserved at the source lag. This indicator is an adjustment, not an edge exposure, and we report three codings of it. In the primary coding it enters with its source (all sources in pooled and CKG-gated, admitted sources in typed) but not in base. In coding (a) the indicators enter every model, including base. In coding (b) they are dropped from every model. Table~\ref{tab:rel-ll} gives the results. On VitalDB, the shrunk-patient model equals the pooled model in every coding, because the DerSimonian-Laird between-patient variance estimate \citep{dersimonian1986meta} is 0 for every association relation: the penalized per-patient estimates are less dispersed than their standard errors imply (Cochran's heterogeneity statistic $Q$ is below its degrees of freedom, the number of patient-level estimates minus one). On MIMIC-IV-WDB (51 patients, primary coding), the pooled, typed, CKG-gated and shrunk-patient models do not differ from base (all CIs cover 0), and typed minus CKG-gated is $-0.005$ [$-0.016$, 0.007]. The per-patient model gives $-0.074$ [$-0.123$, $-0.037$], 0.075 [0.047, 0.117] below the shrunk-patient model. On VitalDB, source terms fitted on the first half improve the second-half likelihood over base only when the indicators are dropped (coding (b), typed $+0.019$); association relations therefore carry little predictive information beyond the base hazard. Typed minus CKG-gated and typed minus pooled are positive in every coding, and the per-patient model lies 0.257-0.258 nats per onset below the pooled model in every coding.

\begin{table}[!htbp]
\centering
\caption{Held-out log-likelihood relative to base (nats per onset, VitalDB) under three codings of unobserved sources.}
\label{tab:rel-ll}
\scriptsize
\begin{tabular}{@{}lccc@{}}
\toprule
Model & Primary & (a) indicators in base & (b) indicators dropped \\
\midrule
Pooled (= shrunk-patient) & $-$0.039 [$-$0.045, $-$0.033] & $-$0.023 [$-$0.030, $-$0.017] & 0.003 [$-$0.001, 0.007] \\
Typed & $-$0.029 [$-$0.035, $-$0.024] & $-$0.014 [$-$0.020, $-$0.008] & 0.019 [0.015, 0.022] \\
CKG-gated & $-$0.074 [$-$0.080, $-$0.067] & $-$0.058 [$-$0.065, $-$0.051] & $-$0.020 [$-$0.024, $-$0.015] \\
Per-patient & $-$0.296 [$-$0.314, $-$0.279] & $-$0.281 [$-$0.299, $-$0.263] & $-$0.255 [$-$0.272, $-$0.239] \\
\addlinespace
Typed $-$ CKG-gated & 0.045 [0.040, 0.049] & 0.044 [0.040, 0.049] & 0.038 [0.034, 0.043] \\
Typed $-$ pooled & 0.010 [0.008, 0.011] & 0.010 [0.008, 0.011] & 0.016 [0.014, 0.017] \\
Shrunk-patient $-$ per-patient & 0.257 [0.241, 0.274] & n/a & n/a \\
\bottomrule
\end{tabular}
\end{table}

\subsection{Patient-level deviations}
\label{app:relations-deviations}

A patient-level relation is a deviation $\delta_p$ added to the pooled logit hazard of patient $p$ at rows where the source is exposed. It is tested by a score test, with the pooled linear predictor as an offset and the patient intercept profiled out, using all at-risk rows. Multiplicity is controlled by Benjamini-Yekutieli (BY) at level 0.05 over all patient $\times$ relation tests of the three families \citep{benjamini2001control}. On real data this admits 2{,}434 deviations in 1{,}471 patients out of 350{,}496 tests. On the first two replicates of each null kind it admits 2{,}100 and 1{,}981 deviations (transplant; 336{,}192 and 335{,}663 tests) and 1{,}760 and 1{,}812 deviations (shift; 357{,}433 and 357{,}653 tests). In a post hoc re-test, empirical $p$-values from the patient-level statistics of the same relation in these four null replicates, followed by BY, admit no deviation on real data and none on any null replicate scored against the other three. Empirical $p<0.05$ occurs in 7.4\% of real tests and 5.0\% of null tests. Adding penalized per-patient deviations to the pooled association model lowers the held-out log-likelihood by 0.257 [0.241, 0.274] nats per onset (Table~\ref{tab:rel-ll}). Our graphs therefore store no patient-specific relation weight; every instantiated relation carries the population estimate.

\section{Beat-level couplings and patient specificity}
\label{app:fast}

This appendix specifies the beat pipeline, the delay-aware pulse arrival time, the epoch statistics and the nulls, reports the full results behind Sec.~\ref{sec:exp-specificity}, and adds construct-validity checks and a measurement of the monitor's display latency. Unless MIMIC-IV-WDB is named, analyses use VitalDB \citep{lee2022vitaldb}: ECG lead II, invasive arterial pressure and the photoplethysmogram (PPG) at 500~Hz, the capnogram at 62.5~Hz, and 1-s monitor numerics, all on a common clock in seconds from case start. Analyses follow protocols fixed before the outcomes were computed (App.~\ref{app:prereg}); items marked \emph{post hoc} or \emph{exploratory} are outside the protocols. Unless stated otherwise, CIs are 95\% percentile intervals from 2,000 case-bootstrap replicates \citep{efron1994introduction}.

\subsection{Beat extraction and quality control}
\label{app:fast-beats}

For beat $n$ we write $t^{\mathrm R}_n$ for the R-peak time and $\mathrm{RRI}_n=t^{\mathrm R}_{n+1}-t^{\mathrm R}_n$ (ms). $\mathrm{SBP}_n$ is the systolic peak of the pulse ejected by beat $n$; it normally occurs inside $\mathrm{RRI}_n$, so $\mathrm{SBP}_n\to\mathrm{RRI}_n$ is causally admissible and $\mathrm{RRI}_n\to\mathrm{SBP}_n$ is not. $\mathrm{DBP}_n$ is the end-diastolic minimum before that pulse, $\mathrm{PP}_n=\mathrm{SBP}_n-\mathrm{DBP}_n$ the pulse pressure, and $\phi_n$ the respiratory phase at $t^{\mathrm R}_n$.

\paragraph{Detectors.}
R peaks come from a Pan-Tompkins-type detector \citep{pan1985real}: 5-18~Hz band-pass, derivative, squaring, 120-ms moving-window integration, an adaptive block-wise threshold, a 250-ms refractory period, search-back in R-R gaps longer than 1.6 times the running median, rejection of T waves (a candidate within 360~ms of the previous peak with less than half its integrated energy), and refinement to the R apex with the lead polarity estimated per segment. Arterial systolic peaks are taken on the 16-Hz low-passed pressure with prominence at least 5~mmHg and at least 0.4 times the running median prominence, which rejects dicrotic notches; the pulse foot is the minimum between consecutive systolic peaks. $\mathrm{SBP}_n$ is the first systolic peak at least 50~ms after $t^{\mathrm R}_n$ whose delay lies within 150~ms of the segment's median R-to-systole delay and that occurs no later than 100~ms after $t^{\mathrm R}_{n+1}$. PPG pulses are detected on the 0.4-10~Hz band-passed signal; the foot is located by the intersecting-tangent method, i.e., the intersection of the tangent at the point of maximum upslope with the horizontal line through the preceding minimum. The respiratory phase is the Hilbert phase of the capnogram band-passed around its dominant frequency (Welch estimate within 0.08-0.8~Hz); phase wraps delimit breaths.

\paragraph{Beat gating.}
A beat is valid iff $\mathrm{RRI}_n\in[300,2000]$~ms and lies within 20\% of the 11-beat running median (the beat after an invalid interval is also dropped), a matched arterial pulse exists, $\mathrm{SBP}_n\in[40,250]$, $\mathrm{DBP}_n\in[15,150]$ and $\mathrm{PP}_n\in[10,150]$~mmHg, and $\mathrm{SBP}_n$ lies within 30~mmHg of its 11-beat running median.

\paragraph{Case quality control.}
Of the 3,442 cases, 3,227 have ECG and arterial waveforms, 2,962 have at least 50\% valid beats, and 2,954 have a case-level bias of beat-derived HR against the monitor HR within $\pm$5~bpm; all 2,954 have at least 20~min of valid beats. Of these, 2,944 have at least one usable epoch (105,505 usable epochs; 94,482 under controlled ventilation).

\subsection{Delay-aware pulse arrival time}
\label{app:fast-pat}

$\mathrm{PAT}_n$ is the time from $t^{\mathrm R}_n$ to the foot of the corresponding PPG pulse. In VitalDB the PPG foot follows the R peak by about 0.65~s (the physiological arrival time plus a device delay of the PPG output), so a fixed search window of 80-600~ms after the R peak retains a median of 3.1\% of beats. We therefore match feet in a delay-aware way. PPG feet are detected on the 500-Hz waveform in 10-min chunks. For each case $p$, $D_p$ is the mode (10-ms bins) of all foot-minus-R lags in $(0,1.5]$~s over valid beats. Each R peak is assigned the foot nearest to $t^{\mathrm R}_n+D_p$ if it lies within $\pm\min(0.15~\mathrm{s},\,0.4\times\text{median RRI})$; otherwise $\mathrm{PAT}_n$ is missing. $\mathrm{PAT}_n$ therefore includes a per-case device delay that is constant within the case; only within-case variation is interpreted. Across the 2,954 cases, $D_p$ has median 0.655~s (5th-95th percentile 0.605-0.695~s).

PAT quality control excludes cases whose PPG clock alignment cannot be verified (24 cases; 2,930 remain), whose $D_p$ lies outside $[0.45, 0.90]$~s (2,917 remain) or whose PAT coverage of valid beats is below 0.80 (2,893 remain; median coverage 0.990 over all 2,954 cases). Of these, 2,884 have at least one epoch with an estimable PAT-SBP slope (103,382 of 103,439 usable epochs). The positive-control bar and the null readings below were registered before any full-cohort delay-aware PAT outcome was computed; App.~\ref{app:prereg}, deviation (v), lists the results known at that time.

\subsection{Epochs and statistics}
\label{app:fast-epochs}

Blocks are maximal runs of valid beats with gaps of at most 60~s, kept if at least 600~s long, and are cut into consecutive 5-min epochs $e$. An epoch is usable iff it has at least 100 beats, at least 80\% valid beats and at least 150 valid beats; statistics require at least 60 valid beats. An epoch is under controlled ventilation iff it contains at least 30 breaths, the coefficient of variation of breath duration is at most 0.15, the ventilator set rate is recorded, and the capnogram rate is within 2~min$^{-1}$ of it. For a beat series $x_n$, $x^{\mathrm{dt}}_n$ denotes $x_n$ minus its centered $W$-beat running mean over valid beats ($W=31$ unless stated), defined where at least $\lfloor W/2\rfloor$ beats in the window are valid.
\begin{itemize}[leftmargin=*,itemsep=1pt,topsep=2pt]
\item \textbf{PAT-SBP slope} $b_{p,e}$: ordinary least-squares slope of $\mathrm{PAT}^{\mathrm{dt}}_n$ on $\mathrm{SBP}^{\mathrm{dt}}_n$ over valid beats with both values (at least 30 pairs), in ms/mmHg; the Pearson correlation of the same pairs is reported alongside. Expected sign: negative \citep{mukkamala2015toward}.
\item \textbf{Pulse pressure variation (PPV) and systolic pressure variation (SPV)} (phase available for more than 80\% of beats; population summaries use controlled-ventilation epochs only): per breath with at least 3 valid beats and at least 70\% valid beats, $100\,(\mathrm{PP}_{\max}-\mathrm{PP}_{\min})/\big((\mathrm{PP}_{\max}+\mathrm{PP}_{\min})/2\big)$ \citep{michard2000relation} and $\mathrm{SBP}_{\max}-\mathrm{SBP}_{\min}$; the epoch value is the median over breaths.
\item \textbf{RSA amplitude} (phase and population-summary conditions as for PPV): least-squares fit $\mathrm{RRI}^{\mathrm{dt}}_n=a_0+a_1\cos\phi_n+a_2\sin\phi_n+\varepsilon_n$ with a 15-beat detrend and at least 60 beats; amplitude $2\sqrt{a_1^2+a_2^2}$ (peak to trough, ms).
\item \textbf{Sequence baroreflex sensitivity}, following \citet{parati1988evaluation}: runs of at least 3 beats in which SBP changes by at least 1~mmHg per beat and $\mathrm{RRI}_{n+\ell}$ by at least 5~ms per beat in the same direction, with correlation at least 0.85, at beat lags $\ell\in\{0,1,2\}$; the statistic is the mean RRI-on-SBP slope over these sequences pooled across lags (ms/mmHg).
\item \textbf{Causal beat ARX baroreflex} (autoregressive model with exogenous input), after \citet{porta2013model}: $\mathrm{RRI}^{\mathrm{dt}}_n=c_0+\sum_{i=1}^{2}c_i\mathrm{RRI}^{\mathrm{dt}}_{n-i}+\sum_{\ell=0}^{2}\beta_\ell\mathrm{SBP}^{\mathrm{dt}}_{n-\ell}+\varepsilon_n$ by least squares; the statistic is $\sum_\ell\beta_\ell$ (ms/mmHg). Expected sign: positive.
\item \textbf{RMSSD}, the root mean square of successive differences of valid R-R intervals (a signal property, no coupling), and epoch \textbf{mean HR} ($60{,}000$ divided by the mean valid RRI) and \textbf{mean SBP}.
\end{itemize}
Positive-control rule: the population sign of a coupling is accepted iff the 95\% CI of the share of cases whose median epoch statistic has the expected sign has lower bound at least 0.90.

\subsection{Nulls and signal-property surrogates}
\label{app:fast-nulls}

\paragraph{Within-case nulls for the population sign.}
For every epoch with an estimable slope we draw five null data sets and recompute $b_{p,e}$ with the same estimator. (N-a) SBP and its validity mask are circularly shifted against PAT by $\ell_{\mathrm{sh}}\sim\mathcal U(60,240)$~s within the 300-s epoch, so the shift is at least 60~s in both directions. (N-b) PAT of epoch $e$ is paired, beat by beat from each epoch's first beat, with SBP of another usable epoch of the same case. Both nulls break beat alignment. We report the null case share of negative medians (mean over draws) and the paired bootstrap CI of the observed minus null share. The sign is read as not an artifact of the estimator iff the CI of the difference lies above 0 under both nulls.

\paragraph{Signal-property surrogates.}
For the patient-specificity analysis each statistic has a surrogate computed on the same epoch that keeps its marginal signal properties but destroys the coupling: for PAT-SBP, SBP circularly shifted against PAT by $\ell_{\mathrm{sh}}\sim\mathcal U\{30,\dots,n_e-30\}$ beats ($n_e$ beats in the epoch; mean over 5 shifts); for PPV, PPV recomputed after subtracting from $\mathrm{PP}_n$ its component locked to the respiratory phase (least-squares fit on two harmonics of $\phi_n$); for RSA, the amplitude recomputed after adding to the phase of each breath its own independent offset $\delta\sim\mathcal U(0,2\pi)$ (mean over 5 draws); for sequence baroreflex sensitivity, SBP circularly shifted by $\ell_{\mathrm{sh}}\sim\mathcal U\{20,\dots,n_e-20\}$ beats (mean over 5 shifts). RMSSD has no surrogate.

\paragraph{Per-patient tests.}
The per-patient surrogate tests that instantiate mechanistic edges in the graphs, and their evaluation on null data, are specified and reported in App.~\ref{app:graphs-fast}. The baroreflex edge is not tested per patient because its population sign does not meet the positive-control rule.

\subsection{Population results}
\label{app:fast-pop}

Table~\ref{tab:fast-pop} gives the population signs and distributions, and Table~\ref{tab:fast-signnulls} the within-case nulls. The PAT-SBP sign is unchanged when PAT quality control is relaxed (93.6-93.9\%). Because the PPG foot could shift with pulse amplitude, which co-varies with SBP, we also regressed PAT on detrended SBP and detrended PPG amplitude (post hoc); the median slope decreases by 15\% and 92.4\% of cases stay negative. Within case, epoch-level PAT shortens with MAP (CV7w in Table~\ref{tab:fast-cv}; $-0.56$~ms/mmHg after adjustment for HR, post hoc). In 50 MIMIC-IV-WDB intensive-care patients \citep{moody2022mimic4wdb} (exploratory), where the PPG foot follows the R peak by a median of 439~ms, so the fixed 80-600~ms window applies, the detrended PAT-SBP correlation is negative in all 50 patients and in 87\% [83, 91] of 1,536 epochs.

\begin{table}[h]
\centering
\caption{Population sign and distribution of beat-level statistics. Case share: cases whose median epoch statistic has the expected sign; epoch share with case-cluster CI; last column: median of case medians. PPV, SPV, RSA: controlled-ventilation epochs.}
\label{tab:fast-pop}
\resizebox{\ifdim\width>\linewidth\linewidth\else\width\fi}{!}{%
\begin{tabular}{lcrrccc}
\toprule
Statistic & Sign & Cases & Epochs & Case share (\%) & Epoch share (\%) & Median [95\% CI] \\
\midrule
PAT-SBP slope (ms/mmHg) & $-$ & 2,884 & 103,382 & 94.3 [93.4, 95.1] & 86.6 [85.9, 87.2] & $-0.58$ [$-0.59$, $-0.56$] \\
PAT-SBP correlation & $-$ & 2,884 & 103,382 & 94.3 [93.4, 95.1] & 86.6 [85.9, 87.2] & $-0.18$ [$-0.19$, $-0.17$] \\
\quad no PAT quality control & $-$ & 2,942 & 105,022 & 93.6 [92.6, 94.5] & 86.1 [85.4, 86.7] & $-0.57$ [$-0.59$, $-0.56$] \\
\quad clock exclusion only & $-$ & 2,920 & 104,473 & 93.9 [93.0, 94.8] & 86.3 [85.6, 86.9] & $-0.57$ [$-0.59$, $-0.56$] \\
\quad PPG-amplitude adjusted (post hoc) & $-$ & 2,884 & n/a & 92.4 [91.4, 93.4] & n/a & $-0.49$ [$-0.51$, $-0.47$] \\
Beat ARX baroreflex (ms/mmHg) & $+$ & 2,944 & 105,505 & 63.3 [61.5, 64.9] & 59.3 [58.4, 60.2] & 0.13 [0.11, 0.15] \\
\midrule
PPV (\%) & & 2,915 & 94,476 & & & 5.60 [5.50, 5.72] \\
SPV (mmHg) & & 2,915 & 94,476 & & & 4.18 [4.10, 4.26] \\
RSA amplitude (ms) & & 2,915 & 94,476 & & & 5.95 [5.76, 6.14] \\
Sequence baroreflex sensitivity (ms/mmHg) & & 2,938 & 82,266 & & & 4.77 [4.69, 4.85] \\
Mean HR (bpm) & & 2,944 & 105,505 & & & 70.46 [69.87, 70.98] \\
Mean SBP (mmHg) & & 2,944 & 105,505 & & & 116.20 [115.60, 116.74] \\
\bottomrule
\end{tabular}}
\end{table}

\begin{table}[h]
\centering
\caption{Within-case nulls for the PAT-SBP sign (N-a, N-b: mean over 5 draws; surrogate: slope averaged over 5 shifts per epoch). Observed case share negative: 94.3\%; pp: percentage points.}
\label{tab:fast-signnulls}
\resizebox{\ifdim\width>\linewidth\linewidth\else\width\fi}{!}{%
\begin{tabular}{lrccc}
\toprule
Null & Cases & Null case share (\%) & Null epoch share (\%) & Observed $-$ null (pp) \\
\midrule
N-a: circular shift of SBP by 60-240~s & 2,884 & 49.1 [48.3, 50.0] & 49.8 & $+45.2$ [44.0, 46.4] \\
N-b: pairing with another epoch of the case & 2,880 & 52.3 [51.5, 53.2] & 50.4 & $+42.0$ [40.8, 43.2] \\
Signal-property surrogate & 2,884 & 48.5 [46.7, 50.3] & 49.9 & $+45.8$ [43.8, 47.8] \\
\bottomrule
\end{tabular}}
\end{table}

\subsection{Patient specificity}
\label{app:fast-spec}

\paragraph{Protocol.}
The notation follows Sec.~\ref{sec:exp-specificity}. The early and late halves come from the same anesthetic record, so the test measures within-record stability of a case-level value, not a trait that persists across admissions. Each statistic is winsorized at the 1st and 99th percentiles of its values over usable epochs. The primary split (S1) orders a case's finite epoch values in time and takes the first half (rounded down) as the early half; a case is eligible if each half has at least 3 values. From the early halves we compute the pooled epoch-weighted mean $\mu$, the pooled within-case variance $\sigma^2_{\mathrm w}$, the method-of-moments between-case variance $\sigma^2_{\mathrm b}=\max\{0,\widehat{\mathrm{Var}}_p(\hat\theta^{\mathrm E}_p)-\overline{\sigma^2_{\mathrm w}/N_p}\}$ \citep[cf.][]{dersimonian1986meta}, the shrinkage weight $\xi_p=\sigma^2_{\mathrm b}/(\sigma^2_{\mathrm b}+\sigma^2_{\mathrm w}/N_p)$ and $\hat m_p=\mu+\xi_p(\hat\theta^{\mathrm E}_p-\mu)$. The MSE of each prediction against the late-half epoch values is averaged within and then across cases; the difference is $\Delta\mathrm{MSE}=\mathrm{MSE}(\hat m)-\mathrm{MSE}(\mu)$ and the relative difference is $\mathrm{MSE}(\hat m)/\mathrm{MSE}(\mu)-1$. All quantities, including $\mu$, $\sigma^2_{\mathrm w}$ and $\sigma^2_{\mathrm b}$, are refitted in each of 2,000 case-bootstrap replicates. $r_{\mathrm{EL}}$ is the Pearson correlation of early and late half means. A cross-patient pairing null permutes late halves across cases (1,000 permutations) and gives one-sided $p$-values for $r_{\mathrm{EL}}$; Holm's procedure is applied over the five primary statistics (PAT-SBP slope, PPV, RSA, sequence baroreflex sensitivity, RMSSD). Adjustments are residuals from a least-squares regression, fitted once on early-half epochs, on epoch mean HR and SBP (adjusted, primary), additionally on the signal-property surrogate (surrogate-adjusted), or additionally on remifentanil and propofol effect-site concentrations (0 if not recorded) and a propofol-recorded indicator (drug-adjusted). Decision rule: a statistic is patient-specific iff the CI of $\mathrm{MSE}(\hat m)-\mathrm{MSE}(\mu)$ lies below 0 both unadjusted and adjusted. Secondary splits: S2 interleaves epochs on the 5-min grid (indices $\equiv 0$ versus $\equiv 2 \bmod 4$, so halves are at least 5~min apart); S3 shuffles epoch order within case before splitting (200 draws; median reported), a ceiling for a case-specific component that does not drift; S4 uses only the first 3 epochs of each S1 half (15~min each). Persistence relative to the operating point is the paired case-bootstrap difference between $r_{\mathrm{EL}}$ of the adjusted statistic and $r_{\mathrm{EL}}$ of unadjusted mean HR or mean SBP, on common cases.

\paragraph{Results.}
All five primary statistics are patient-specific under the rule (Table~\ref{tab:fast-spec}); permutation $p=0.001$ for each (the smallest attainable with 1,000 permutations), Holm-adjusted $0.005$. The same holds after drug adjustment ($r_{\mathrm{EL}}$ 0.63-0.72) and, for PPV and RSA, within controlled-ventilation epochs only (adjusted $r_{\mathrm{EL}}=0.69$ [0.67, 0.71] and 0.66 [0.63, 0.69]; 2,847 cases). For PAT-SBP, the signal-property surrogate itself has $r_{\mathrm{EL}}=0.03$ [$-0.03$, 0.08], so the surrogate adjustment does not probe signal properties; adjusting instead for epoch PAT level or the case's modal lag $D_p$ (post hoc) gives $r_{\mathrm{EL}}=0.64$ [0.61, 0.67] and 0.62 [0.59, 0.66]. Table~\ref{tab:fast-splits} gives the split variants, and Table~\ref{tab:fast-spec-op} compares persistence with the operating point: all five statistics are less persistent than mean HR, and their differences from mean SBP lie between $-0.05$ and 0.03. S1 lies below the S3 ceiling (0.80-0.92) for every statistic. With three epochs per half (S4), $r_{\mathrm{EL}}$ is 0.23-0.43, and the MSE criterion holds for PAT-SBP, sequence baroreflex sensitivity and RMSSD but not for PPV or RSA. A separately registered test of the same five statistics and decision rule on 50 MIMIC-IV-WDB patients (41-47 eligible per statistic) and 24 VitalDB cases (7-24 eligible) declared none patient-specific in any of the 10 statistic-cohort cells (MIMIC-IV-WDB adjusted $r_{\mathrm{EL}}$ 0.43-0.71), so the within-record result was not replicated in these two smaller cohorts.

\begin{table}[h]
\centering
\caption{Patient specificity of beat-level statistics (split S1). Adjusted: residual after epoch mean HR and SBP. Rel.\ $\Delta$MSE: $\mathrm{MSE}(\hat m)/\mathrm{MSE}(\mu)-1$ (adjusted).}
\label{tab:fast-spec}
\footnotesize
\setlength{\tabcolsep}{3pt}
\resizebox{\ifdim\width>\linewidth\linewidth\else\width\fi}{!}{%
\begin{tabular}{@{}lrcccc@{}}
\toprule
 & & \multicolumn{3}{c}{$r_{\mathrm{EL}}$} & \\
\cmidrule(lr){3-5}
Statistic & Cases & unadjusted & adjusted & surrogate-adjusted & Rel.\ $\Delta$MSE \\
\midrule
PAT-SBP slope & 2,861 & 0.66 [0.63, 0.69] & 0.67 [0.63, 0.70] & 0.67 [0.63, 0.70]$^{a}$ & $-0.18$ [$-0.21$, $-0.16$] \\
PPV & 2,921 & 0.68 [0.65, 0.70] & 0.69 [0.66, 0.71] & 0.71 [0.69, 0.74] & $-0.27$ [$-0.30$, $-0.24$] \\
RSA amplitude & 2,921 & 0.67 [0.64, 0.70] & 0.66 [0.63, 0.69] & 0.65 [0.62, 0.68] & $-0.21$ [$-0.25$, $-0.18$] \\
Sequence baroreflex sens. & 2,782 & 0.63 [0.60, 0.67] & 0.63 [0.59, 0.67] & 0.16 [0.08, 0.25]$^{b}$ & $-0.17$ [$-0.20$, $-0.14$] \\
RMSSD & 2,921 & 0.73 [0.69, 0.76] & 0.72 [0.68, 0.75] & n/a & $-0.28$ [$-0.32$, $-0.24$] \\
\midrule
Mean HR & 2,921 & 0.85 [0.84, 0.86] & & & \\
Mean SBP & 2,921 & 0.68 [0.66, 0.70] & & & \\
SPV & 2,921 & 0.73 [0.71, 0.75] & 0.72 [0.70, 0.74] & & \\
Beat ARX baroreflex & 2,921 & 0.62 [0.58, 0.65] & 0.61 [0.58, 0.65] & & \\
PAT-SBP correlation & 2,861 & 0.74 [0.72, 0.76] & 0.75 [0.73, 0.76] & & \\
\bottomrule
\multicolumn{6}{@{}l}{\scriptsize $^{a}$Surrogate $r_{\mathrm{EL}}=0.03$; see text. $^{b}$2,770 cases; MSE criterion not met ($\Delta$MSE CI above 0).}
\end{tabular}}
\end{table}

\begin{table}[h]
\centering
\caption{Persistence relative to the operating point: adjusted $r_{\mathrm{EL}}$ minus $r_{\mathrm{EL}}$ of mean HR or mean SBP (paired case bootstrap on common cases).}
\label{tab:fast-spec-op}
\footnotesize
\begin{tabular}{@{}lcc@{}}
\toprule
Statistic & $\Delta r$ vs.\ mean HR & $\Delta r$ vs.\ mean SBP \\
\midrule
PAT-SBP slope & $-0.18$ [$-0.22$, $-0.15$] & $-0.02$ [$-0.05$, 0.02] \\
PPV & $-0.16$ [$-0.19$, $-0.14$] & 0.00 [$-0.03$, 0.03] \\
RSA amplitude & $-0.19$ [$-0.22$, $-0.16$] & $-0.03$ [$-0.06$, 0.01] \\
Sequence baroreflex sens. & $-0.20$ [$-0.24$, $-0.16$] & $-0.05$ [$-0.09$, $-0.00$] \\
RMSSD & $-0.13$ [$-0.17$, $-0.10$] & 0.03 [$-0.01$, 0.08] \\
\bottomrule
\end{tabular}
\end{table}

\begin{table}[h]
\centering
\caption{Split variants of the patient-specificity test ($r_{\mathrm{EL}}$, adjusted). S2: interleaved halves; S3: within-case shuffle (ceiling); S4: three epochs per half. MSE: whether the $\Delta$MSE CI lies below 0. S2 has 2,389-2,751 eligible cases; S3 and S4 use the S1 cases.}
\label{tab:fast-splits}
\resizebox{\ifdim\width>\linewidth\linewidth\else\width\fi}{!}{%
\begin{tabular}{lccccc}
\toprule
Statistic & S2 & S3 & S4 & MSE (S2) & MSE (S4) \\
\midrule
PAT-SBP slope & 0.83 [0.81, 0.85] & 0.86 [0.85, 0.87] & 0.23 [0.19, 0.28] & yes & yes \\
PPV & 0.93 [0.92, 0.93] & 0.92 [0.91, 0.93] & 0.31 [0.27, 0.35] & yes & no \\
RSA amplitude & 0.90 [0.88, 0.91] & 0.90 [0.89, 0.91] & 0.36 [0.32, 0.40] & yes & no \\
Sequence baroreflex sens. & 0.78 [0.75, 0.81] & 0.80 [0.79, 0.82] & 0.34 [0.28, 0.39] & yes & yes \\
RMSSD & 0.88 [0.86, 0.89] & 0.90 [0.89, 0.91] & 0.43 [0.38, 0.49] & yes & yes \\
\bottomrule
\end{tabular}}
\end{table}

\paragraph{Slow response relations.}
The same split-half test is applied to the response relations of ClosedLoopBench (App.~\ref{app:benchmark}) with action events as units, with 30-s windows (all 3,442 cases) and on the 2-s grid for a 500-case subsample (cases with vasoactive infusions first, then the longest recordings). A case is eligible if it has at least 6 events of the relation's direction with a clean pre-event period and at least 50\% coverage of the vital sign; events are split in time and each half needs at least 3 finite outcomes. Three response contrasts are computed per event at the relation's horizon (60~s for the ventilator, 120~s for drugs): the observed change (post-event mean minus the last pre-event level over $[t-30~\mathrm{s},t)$; primary), the deviation from the local pre-event trend line, and the bias-corrected interrupted-trend (kink) contrast of the benchmark, which removes case-level drift. A relation is patient-specific iff the CI of $\mathrm{MSE}(\hat m)-\mathrm{MSE}(\mu)$ lies below 0 for both the observed change and the kink contrast; CIs use 1,000 case-bootstrap replicates. No relation passes at either resolution (Table~\ref{tab:slow-spec}). Per-case means have a higher MSE than the pooled mean (CI above 0) in 31 of 33 relation-contrast pairs at 30~s and 30 of 33 at 2~s (11 relations $\times$ 3 contrasts; nitroglycerin$\to$MAP, with 2 eligible cases, is not estimable); the shrunk estimate is better in 0 of 33 at each resolution. At 2~s (10-278 cases per relation), $r_{\mathrm{EL}}$ ranges from $-0.40$ to 0.28 over the 33 pairs, and no CI lies above 0. The 18-relation ClosedLoopBench protocol (Table~\ref{tab:prereg}) registered the same test; under it, no relation is patient-specific at either resolution.

\begin{table}[h]
\centering
\caption{Patient specificity of slow response relations, 30-s windows. Cases: eligible for the observed change. $\Delta$MSE in squared outcome units (mmHg$^2$ for MAP and EtCO$_2$, bpm$^2$ for HR); a shrunk-minus-pooled value of 0 means the estimated between-case variance is 0.}
\label{tab:slow-spec}
\resizebox{\ifdim\width>\linewidth\linewidth\else\width\fi}{!}{%
\begin{tabular}{lrcccc}
\toprule
 & & \multicolumn{2}{c}{$r_{\mathrm{EL}}$ [95\% CI]} & \multicolumn{2}{c}{$\Delta$MSE, observed change [95\% CI]} \\
\cmidrule(lr){3-4}\cmidrule(lr){5-6}
Relation & Cases & Observed change & Kink contrast & Per-case $-$ pooled & Shrunk $-$ pooled \\
\midrule
Ventilator rate $\uparrow\to$ EtCO$_2$ $\downarrow$ & 302 & 0.03 [$-$0.12, 0.17] & $-$0.06 [$-$0.17, 0.06] & 2.67 [1.33, 4.24] & 0.00 [$-$0.03, 0.17] \\
Ventilator rate $\downarrow\to$ EtCO$_2$ $\uparrow$ & 194 & 0.13 [$-$0.03, 0.28] & 0.10 [$-$0.04, 0.23] & 0.75 [0.30, 1.18] & 0.00 [$-$0.06, 0.03] \\
Propofol $\uparrow\to$ MAP $\downarrow$ & 249 & $-$0.01 [$-$0.13, 0.11] & 0.06 [$-$0.07, 0.19] & 143 [87.6, 205] & 0.60 [$-$0.19, 12.7] \\
Propofol $\downarrow\to$ MAP $\uparrow$ & 375 & 0.08 [$-$0.03, 0.18] & $-$0.03 [$-$0.15, 0.07] & 201 [138, 275] & $-$0.11 [$-$1.89, 8.35] \\
Remifentanil $\uparrow\to$ HR $\downarrow$ & 925 & 0.01 [$-$0.09, 0.10] & 0.07 [$-$0.04, 0.18] & 25.9 [20.2, 33.0] & 1.10 [0.01, 3.57] \\
Remifentanil $\uparrow\to$ MAP $\downarrow$ & 898 & $-$0.05 [$-$0.21, 0.09] & 0.02 [$-$0.06, 0.10] & 151 [102, 218] & 9.37 [0.13, 30.4] \\
Remifentanil $\downarrow\to$ HR $\uparrow$ & 914 & 0.00 [$-$0.09, 0.09] & 0.02 [$-$0.06, 0.10] & 15.9 [12.9, 18.7] & 0.15 [$-$0.02, 0.70] \\
Remifentanil $\downarrow\to$ MAP $\uparrow$ & 834 & 0.04 [$-$0.02, 0.10] & 0.00 [$-$0.07, 0.07] & 189 [146, 237] & 1.76 [$-$0.51, 8.27] \\
Phenylephrine $\uparrow\to$ MAP $\uparrow$ & 18 & $-$0.11 [$-$0.50, 0.33] & $-$0.31 [$-$0.79, 0.62] & 87.7 [12.4, 160] & 0.19 [$-$0.14, 48.7] \\
Phenylephrine $\uparrow\to$ HR $\downarrow$ & 16 & $-$0.36 [$-$0.62, $-$0.05] & $-$0.10 [$-$0.63, 0.50] & 13.7 [4.92, 21.1] & 0.00 [0.00, 1.54] \\
Norepinephrine $\uparrow\to$ MAP $\uparrow$ & 10 & 0.05 [$-$0.44, 0.56] & $-$0.14 [$-$0.71, 0.19] & 114 [$-$19.1, 280] & 2.47 [$-$6.44, 50.0] \\
\bottomrule
\end{tabular}}
\end{table}

\subsection{Construct-validity checks}
\label{app:fast-cv}

Seven construct-validity (CV) checks, directional associations fixed from physiology before analysis, are listed in Table~\ref{tab:fast-cv}. An item holds iff its 95\% CI excludes 0 in the expected direction, is reversed iff it excludes 0 in the other direction, and is not shown otherwise. CV7 relates age to between-case PAT, which contains the per-case device delay; it was declared in advance not to count toward the tally. CV7w, registered with the delay-aware PAT analysis, is a within-case consistency check of the same PAT-pressure mechanism and is not counted either. Of the six counted items, four hold, one is reversed and one is not shown. The registered CV3 estimate is dominated by a single epoch with a recorded propofol effect-site concentration of 3,501~$\mu$g\,mL$^{-1}$; excluding epochs above 20~$\mu$g\,mL$^{-1}$ (one epoch; deviation (ii) in App.~\ref{app:prereg}) gives $-0.017$ [$-0.031$, $-0.001$]. For CV1, adjusting for case mean HR and SBP (post hoc) leaves the reversal (partial Spearman $\rho=0.06$ [0.03, 0.10]). For CV7, case PAT tracks the modal lag $D_p$ (Spearman $\rho=0.94$), and after subtracting $D_p$ the association with age is $\rho=+0.04$.

\begin{table}[h]
\centering
\caption{Construct-validity checks. Spearman $\rho$, within-case slope (case fixed effects), difference of group medians (CV5) or median within-case difference (CV6), with 95\% case-bootstrap CI; age items use adult cases. BRS: sequence baroreflex sensitivity; $C_e$: effect-site concentration; pp: percentage points.}
\label{tab:fast-cv}
\footnotesize
\setlength{\tabcolsep}{3pt}
\resizebox{\ifdim\width>\linewidth\linewidth\else\width\fi}{!}{%
\begin{tabular}{@{}lp{0.30\linewidth}crcl@{}}
\toprule
Item & Association & Expected & Cases (epochs) & Estimate [95\% CI] & Outcome \tabularnewline
\midrule
CV1 & \raggedright Age vs.\ case BRS & $<0$ & 2,904 & $\rho=0.09$ [0.05, 0.12] & reversed \tabularnewline
CV2 & \raggedright Age vs.\ case RSA amplitude (controlled ventilation) & $<0$ & 2,881 & $\rho=-0.04$ [$-0.08$, $-0.01$] & holds \tabularnewline
CV3 & \raggedright Within-case log BRS vs.\ propofol $C_e$ (per $\mu$g\,mL$^{-1}$) & $<0$ & 1,557 (41,341) & $-0.0002$ [$-0.0276$, $-0.0002$] & holds \tabularnewline
CV4 & \raggedright Within-case log BRS vs.\ BIS (per unit) & $>0$ & 2,508 (70,846) & $0.0001$ [$-0.0008$, $0.0010$] & not shown \tabularnewline
CV5 & \raggedright Case PPV, vasopressor minus no-vasopressor cases (pp) & $>0$ & 1,755 / 1,160 & $+0.93$ [0.71, 1.18] & holds \tabularnewline
CV6 & \raggedright Within-case PPV, first 30\,min after induction minus ${\geq}\,60$\,min (pp) & $>0$ & 1,580 & $+1.34$ [1.18, 1.52] & holds \tabularnewline
CV7 & \raggedright Age vs.\ case PAT & $<0$ & 2,853 & $\rho=-0.155$ [$-0.193$, $-0.115$] & not counted \tabularnewline
\midrule
CV7w & \raggedright Within-case epoch PAT vs.\ MAP (ms/mmHg) & $<0$ & 2,878 (103,375) & $-0.652$ [$-0.673$, $-0.631$] & \begin{tabular}[t]{@{}l@{}}holds\\(consistency check)\end{tabular} \tabularnewline
\bottomrule
\end{tabular}}
\end{table}

\subsection{Monitor display latency}
\label{app:fast-latency}

Vital-level relations in ClosedLoopBench are estimated from monitor numerics, which the monitor produces with a processing delay. We measured it on the 2,954 cases. Beat HR ($60{,}000/\mathrm{RRI}_n$, placed at $t^{\mathrm R}_{n+1}$, when the interval is complete) and beat MAP (the mean pressure of each pulse, placed at its systolic peak) were averaged into 1-s bins and cross-correlated with the monitor's HR and arterial MAP after subtracting a centered 61-s moving mean, for lags from $-30$ to $+30$~s (positive: monitor later); a case enters if its peak correlation is at least 0.5, and the lag is refined by parabolic interpolation. Monitor HR lags beat HR by a median of 6.94~s [6.88, 7.02] (1,310 cases; IQR 5.99-8.02) and monitor MAP lags beat MAP by 5.95~s [5.83, 6.02] (688 cases; IQR 4.95-7.57); on unfiltered series, which admit almost all cases, the lags are 6.99~s (2,906 cases) and 6.71~s (2,954 cases). The HR-minus-MAP lag is 0.97~s [0.73, 1.13] in the 342 cases that pass both inclusion criteria and 0.33~s [0.24, 0.42] on unfiltered series (2,906 cases). A 2-s lagged regression of HR on its own and MAP's lags 1-5 (expected: negative summed MAP coefficient) has the expected sign in 24.8\% [23.3, 26.4] of 2,954 cases on monitor numerics, in 26.7\% [25.1, 28.3] after shifting monitor HR by each case's relative lag (the median relative lag, 0.97~s, where it is not estimable), and in 53.7\% [51.8, 55.5] on beat-derived 2-s series (exact McNemar $p<10^{-6}$ for each pairwise comparison). Emulating pure delays or causal averaging of the measured size on beat-derived series (exploratory) gives 38.3-47.4\%, so these processing models do not reproduce the 24.8\% of the monitor numerics.

\section{Concept grounding details}
\label{app:grounding}

This appendix gives the cohort, reference and protocols behind Sec.~\ref{sec:exp-grounding} and the full per-concept results. CIs are 95\% patient-cluster percentile intervals from 2,000 bootstrap resamples \citep{efron1994introduction}; within a concept and split, all models and the rule detector are evaluated on the same windows with the same resample weights, so paired differences are valid. Analyses follow protocols fixed before the outcomes were computed (App.~\ref{app:prereg}); items marked \emph{post hoc} were added afterwards.

\subsection{Cohort, splits and models}
\label{app:grounding-cohort}

The cohort consists of the 200 MIMIC-IV-WDB records of 198 patients \citep{moody2022mimic4wdb,goldberger2000physiobank}, split into 140 training, 30 validation and 30 test records, stratified by quartile of ICU length of stay from MIMIC-IV \citep{johnson2023mimic} (records without a link are assigned to training); no training patient contributes an evaluation record. Windows are 30~s long with a 15-s stride at 62.47~Hz; ECG, PPG and respiration are z-scored per segment and arterial pressure enters as $(\mathrm{ABP}/\mathrm{mmHg}-85)/20$. Of the 30 test and 30 validation records, 26 and 27 have processed windows. The test split is primary. A window is evaluated for a concept if it contains the concept's home modality (its waveform in Table~\ref{tab:concept-vocab}), passes the signal-quality flag and has a monitor reference; at most 8 windows per concept fail the quality flag, and these are not evaluated.

All six models share the architecture of App.~\ref{app:concepts}: per-modality encoders, cross-modal fusion and seven linear concept heads that output the concept probability $\hat q_{p,c}(w)$ for window $w$. Supervised training uses 150,000 windows from 23 training patients, with checkpoint selection on 20,000 windows from 8 validation patients (40 epochs; the selected checkpoints come from epochs 2-13). Rule-supervised heads are trained on the rule detector's labels; monitor-supervised heads are trained on the same windows with the monitor labels of App.~\ref{app:grounding-ref}. Each type is trained with three seeds, numbered 1-3. Table~\ref{tab:label-agreement} (App.~\ref{app:concepts-labels}) gives the agreement of the two label sources on the training windows. Operating thresholds $\lambda_c$: for rule-supervised heads, the threshold maximizing the F1 score against rule labels on validation windows (grid 0.05-0.95, step 0.05); for monitor-supervised heads, a fixed 0.5. A secondary, reference-calibrated threshold for rule-supervised heads minimizes validation error against the monitor reference on the same grid. Validation windows are therefore used for checkpoint selection (8 patients) and for threshold fitting (all validation patients). One person contributes one validation and one test record, which are analyzed as separate patients; this validation record is not among the 8 checkpoint-selection patients but enters the threshold fits.

The graph experiments (Sec.~\ref{sec:exp-graphs}, App.~\ref{app:graphs}) use these models only on MIMIC-IV-WDB: their concept states are the activations of rule-supervised seed 1 at thresholds fitted on training and validation windows (App.~\ref{app:concepts-model}). VitalDB graphs use no model; their concept states are thresholds on the 30-s medians of the monitor numerics (App.~\ref{app:data-products}).

\subsection{Monitor reference and time alignment}
\label{app:grounding-ref}

The rule detector estimates the rate from the autocorrelation of the home-modality window and, for pressure, takes the minimum (hypotension) or maximum (hypertension) of sustained 10-s sub-window means; its continuous estimate is used for AUROC and its thresholded value is exactly the rule label. The monitor reference for a window is the median of the monitor numeric over the window (at least 10 valid samples at 1~Hz; validity ranges 20-250~bpm for rates, 2-80\,/min for respiration, 20-200~mmHg for pressure), thresholded at the definition in Table~\ref{tab:concept-vocab}. For pressure the numeric is the invasive arterial mean (channel ABPm, or ARTm when ABPm is absent). The cuff mean pressure (NBPm) nearest to the window center within 300~s is a second, physically separate pressure reference. For pressure, the two sources therefore define labels differently: the rule uses an extremum of sustained 10-s sub-window means of the waveform, the reference the median of the monitor numeric over the 30-s window. Against this reference the rule detector's pressure AUROCs (0.988 for hypotension, 0.994 for hypertension; Table~\ref{tab:grounding-full}) exceed those of the monitor-supervised heads, so we claim no AUROC gain for the pressure concepts.

Windows are mapped to the record clock through the multi-segment layout header (no unmapped windows; 881 segment headers cross-checked with 0 mismatches); time stamps of the monitor numerics are counter ticks at 999.56~Hz. Alignment was validated by lag scans. The waveform-mean arterial pressure matches the monitor ABPm with a median absolute error of 0.46~mmHg at lag 0, with its minimum at $+6$~s (monitor averaging); on windows more than 24~h into a record the best lag is $+4$~s under the counter-tick clock and $-40$~s under a millisecond clock, which confirms the former. The best lags of waveform-derived rates are $+9$~s (ECG HR against monitor HR) and $+13$~s (PPG against pulse rate; respiration against impedance respiratory rate). Using the best lag instead of lag 0 changes 0.2\% (hypertension) to 8.1\% (tachypnea) of reference labels; lag 0 is used throughout.

\subsection{Per-concept results}
\label{app:grounding-results}

Table~\ref{tab:grounding-full} gives AUROC, the area under the precision-recall curve (AUPRC), binary error and balanced accuracy for every concept and score, and Table~\ref{tab:grounding-paired} the paired differences between monitor- and rule-supervised heads pooled over seeds (the mean over seeds of the per-seed difference, computed on every resample). On the validation split the mean AUROC is $0.967\pm0.001$ (monitor-supervised), $0.885\pm0.017$ (rule-supervised) and 0.800 (rule detector); on the 19 validation patients not used for checkpoint selection it is $0.966\pm0.001$, $0.899\pm0.016$ and 0.808.

Positive reference windows come from 20, 5, 22, 25 and 25 of the 25 test patients for the five rate concepts and from 6 and 7 of the 8 for hypotension and hypertension. Scoring against a second monitor numeric (post hoc) gives paired AUROC gains of $+0.086$ [0.019, 0.167] for tachycardia against the SpO$_2$ pulse rate and $+0.175$ [0.051, 0.282] for PPG tachycardia against the ECG heart rate, and $-0.012$ [$-0.259$, 0.146] for bradycardia against the SpO$_2$ pulse rate. Against the reference-calibrated rule-supervised threshold, the monitor-supervised error difference has a CI below 0 for tachycardia, PPG tachycardia, tachypnea, bradypnea and hypotension; extending that threshold grid to 0.005-0.995 with a never-fire option (post hoc) gives $+0.001$ [$-0.033$, 0.051] for bradycardia and $+0.020$ [0.002, 0.035] for bradypnea.

App.~\ref{app:supp-pressure} scores the pressure heads against the cuff pressure, with and without an arterial waveform.

\begin{table}[h]
\centering
\caption{Per-concept results on the test split against the monitor reference. Heads: mean $\pm$ s.d.\ over 3 seeds at their operating thresholds. Prev.: reference prevalence, which equals the error of a classifier that never fires.}
\label{tab:grounding-full}
\resizebox{\ifdim\width>\linewidth\linewidth\else\width\fi}{!}{%
\begin{tabular}{llcccc}
\toprule
Concept (patients / windows / prev.) & Score & AUROC & AUPRC & Error & Balanced accuracy \\
\midrule
Tachycardia (ECG) & Rule detector & 0.823 [0.707, 0.901] & 0.464 & 0.111 & 0.886 \\
\quad 25 / 234,277 / 0.184 & Rule-supervised & $0.899\pm0.028$ & $0.739\pm0.036$ & $0.144\pm0.013$ & $0.824\pm0.028$ \\
 & Monitor-supervised & $0.990\pm0.001$ & $0.965\pm0.004$ & $0.047\pm0.005$ & $0.948\pm0.004$ \\
\midrule
Bradycardia (ECG) & Rule detector & 0.941 [0.866, 0.961] & 0.132 & 0.066 & 0.940 \\
\quad 25 / 234,277 / 0.018 & Rule-supervised & $0.887\pm0.031$ & $0.147\pm0.111$ & $0.129\pm0.053$ & $0.827\pm0.020$ \\
 & Monitor-supervised & $0.985\pm0.007$ & $0.578\pm0.213$ & $0.024\pm0.005$ & $0.943\pm0.016$ \\
\midrule
Tachycardia (PPG) & Rule detector & 0.533 [0.438, 0.684] & 0.441 & 0.120 & 0.673 \\
\quad 25 / 230,672 / 0.174 & Rule-supervised & $0.812\pm0.016$ & $0.684\pm0.012$ & $0.121\pm0.002$ & $0.665\pm0.005$ \\
 & Monitor-supervised & $0.979\pm0.001$ & $0.914\pm0.006$ & $0.056\pm0.004$ & $0.899\pm0.016$ \\
\midrule
Tachypnea & Rule detector & 0.584 [0.493, 0.664] & 0.610 & 0.254 & 0.681 \\
\quad 25 / 226,117 / 0.374 & Rule-supervised & $0.814\pm0.004$ & $0.747\pm0.006$ & $0.254\pm0.004$ & $0.683\pm0.006$ \\
 & Monitor-supervised & $0.924\pm0.002$ & $0.888\pm0.004$ & $0.166\pm0.004$ & $0.795\pm0.007$ \\
\midrule
Bradypnea & Rule detector & 0.633 [0.546, 0.713] & 0.067 & 0.382 & 0.643 \\
\quad 25 / 226,117 / 0.053 & Rule-supervised & $0.674\pm0.015$ & $0.097\pm0.011$ & $0.464\pm0.010$ & $0.604\pm0.011$ \\
 & Monitor-supervised & $0.888\pm0.018$ & $0.371\pm0.039$ & $0.073\pm0.006$ & $0.734\pm0.040$ \\
\midrule
Hypotension & Rule detector & 0.988 [0.984, 0.996] & 0.853 & 0.087 & 0.951 \\
\quad 8 / 40,211 / 0.099 & Rule-supervised & $0.972\pm0.009$ & $0.695\pm0.098$ & $0.139\pm0.011$ & $0.922\pm0.005$ \\
 & Monitor-supervised & $0.983\pm0.002$ & $0.924\pm0.009$ & $0.023\pm0.000$ & $0.946\pm0.003$ \\
\midrule
Hypertension & Rule detector & 0.994 [0.992, 0.996] & 0.681 & 0.017 & 0.991 \\
\quad 8 / 40,211 / 0.034 & Rule-supervised & $0.995\pm0.001$ & $0.737\pm0.021$ & $0.018\pm0.002$ & $0.986\pm0.003$ \\
 & Monitor-supervised & $0.987\pm0.003$ & $0.837\pm0.026$ & $0.016\pm0.002$ & $0.798\pm0.023$ \\
\bottomrule
\end{tabular}}
\end{table}

\begin{table}[h]
\centering
\caption{Paired differences, monitor- minus rule-supervised heads, pooled over seeds [95\% CI]. Error and balanced accuracy at each type's operating threshold; correct on rule errors: share of windows where the rule label disagrees with the reference that the head classifies correctly. Balanced accuracy was not pre-specified.}
\label{tab:grounding-paired}
\resizebox{\ifdim\width>\linewidth\linewidth\else\width\fi}{!}{%
\begin{tabular}{lcccc}
\toprule
Concept & $\Delta$AUROC & $\Delta$Error & $\Delta$Correct on rule errors & $\Delta$Balanced accuracy \\
\midrule
Tachycardia (ECG) & $+0.091$ [0.023, 0.170] & $-0.097$ [$-0.168$, $-0.043$] & $+0.295$ [0.213, 0.478] & $+0.124$ [0.036, 0.204] \\
Bradycardia (ECG) & $+0.098$ [0.039, 0.257] & $-0.105$ [$-0.216$, $-0.018$] & $+0.294$ [0.129, 0.464] & $+0.116$ [0.020, 0.235] \\
Tachycardia (PPG) & $+0.167$ [0.040, 0.282] & $-0.064$ [$-0.149$, $-0.013$] & $+0.639$ [0.451, 0.742] & $+0.233$ [0.127, 0.308] \\
Tachypnea & $+0.110$ [0.080, 0.142] & $-0.088$ [$-0.125$, $-0.057$] & $+0.407$ [0.359, 0.453] & $+0.112$ [0.075, 0.150] \\
Bradypnea & $+0.214$ [0.082, 0.323] & $-0.392$ [$-0.500$, $-0.273$] & $+0.782$ [0.703, 0.836] & $+0.131$ [0.000, 0.238] \\
Hypotension & $+0.012$ [$-0.033$, 0.039] & $-0.116$ [$-0.183$, $-0.025$] & $+0.795$ [0.747, 0.871] & $+0.024$ [$-0.079$, 0.084] \\
Hypertension & $-0.008$ [$-0.032$, $-0.0005$] & $-0.002$ [$-0.009$, 0.012] & $+0.702$ [0.600, 0.763] & $-0.188$ [$-0.208$, $-0.053$] \\
\bottomrule
\end{tabular}}
\end{table}

\subsection{Counterfactual faithfulness protocol}
\label{app:grounding-protocol}

\paragraph{Keys.}
Edits are applied to keys, i.e.\ windows whose baseline is verified normal against the monitor: for ECG and PPG, a rule rate of 65-95~bpm within 5~bpm of the monitor HR or pulse rate; for respiration, 13-19\,/min within 2\,/min of the monitor respiratory rate; for arterial pressure, a window mean of 70-100~mmHg within 5~mmHg of ABPm with sustained 10-s sub-window means in $[67,103]$~mmHg; all with periodicity quality (normalized autocorrelation at the detected lag) of at least 0.5. Up to 25 keys per patient and family are sampled at random from the test and validation patients (including the 8 validation patients used for checkpoint selection): 960 keys from 43 patients (ECG), 985 from 46 (PPG), 923 from 46 (respiration) and 414 from 18 (pressure). Keys depend only on rule estimates and monitor references, so they are identical for every model.

\paragraph{T1: counterfactual edits.}
A time-warp by the factor $f$ = target/baseline (linear interpolation, zero-phase anti-alias filtering when $f>1$, random sub-sample phase, $f\in[0.3,3]$; the following windows supply the extra samples) multiplies every rate in the window by $f$. Heart-rate targets are 40, 45, 50, 55, 65, 70, 75, 85, 90, 95, 105, 110, 120, 130, 140 and 150~bpm; respiratory targets are 6, 8, 9, 10, 11, 13, 14, 16, 18, 19, 21, 22, 24, 27, 30 and 35\,/min. Heart-rate edits warp all pulsatile channels (ECG, PPG, arterial pressure) by the same $f$, because physiology has one heart rate; warping only the home channel is a secondary mode. Pressure edits scale the arterial waveform to a MAP of 45, 50, 55, 60, 62, 68, 70, 75, 85, 95, 100, 102, 108, 110, 120, 130 or 140~mmHg (a level shift that keeps pulse pressure is a secondary mode). The sham edit is the same operation with $f=1$ (the identity for pressure). Neutral edits multiply the home-channel amplitude by 0.6 and 1.6, which leaves the rate unchanged; for pressure they scale the pulsatile component around the window mean, and only edits that leave the label unchanged are counted. The counterfactual label is the target rate against the definition, and for pressure the rule's sustained sub-window statistics recomputed on the edited signal. The primary families are ECG for tachycardia and bradycardia, PPG for PPG tachycardia, respiration for tachypnea and bradypnea, and arterial pressure for hypotension and hypertension. Criteria: (C1) per key, the Spearman correlation between target and probability times the expected sign, averaged within and then across patients, has CI lower bound $>0$; (C2) the balanced accuracy of the thresholded probability against the counterfactual label has CI lower bound $>0.70$; (C3) under the sham edit, the mean absolute probability change and the share of flipped decisions both have CI upper bound $<0.05$; (C4) under neutral edits, the share of flipped decisions has CI upper bound $<0.10$. T1 passes if and only if C1-C4 hold.

\paragraph{Baselines.}
The randomized head replaces each head's weight vector by a random Gaussian direction of equal norm, keeping bias and feature standardization; the re-initialized model re-initializes encoders, fusion and heads; the rule detector is applied to the edited home-channel signal. For each trained model, its randomized head and a re-initialized model use matched-rate thresholds that activate the same fraction of windows as that model. In an artifact arm, a sinusoid of 3 times the channel standard deviation is added to the home channel after the physiological edit, at 2.1-2.7~Hz for ECG and PPG and 0.40-0.60~Hz for respiration (random phase and frequency per key).

\paragraph{T2: evidence pointers.}
A node's evidence pointers are its activation runs: maximal runs of home-modality windows with $\hat q_{p,c}(w)$ at or above threshold. Per patient and concept, up to 6 runs of 2-40 windows are sampled, with up to 12 pointer windows per run; the span of a run is the run extended by 8 windows on each side. Each pointer window is edited in the home channel only and the model is re-run: \emph{neg} swaps in a same-patient rule-negative window with $\hat q<$ threshold/2 outside any run; \emph{pos} swaps in a rule-positive window from another run; \emph{norm} warps the window to 80~bpm or 16\,/min using its own rule rate (pressure: scaled to MAP 85~mmHg); \emph{sham} applies $f=1$. $\mathrm{Share}_{\mathrm{op}}$ is the share of pointer windows still at or above threshold after operation op. Criteria: (P1) $\mathrm{Share}_{\mathrm{pos}}-\mathrm{Share}_{\mathrm{neg}}$ has CI lower bound $>0.5$; (P2) $\mathrm{Share}_{\mathrm{sham}}-\mathrm{Share}_{\mathrm{norm}}$ has CI lower bound $>0.5$; (P3) the drop in mean span probability after \emph{neg} edits in the pointers exceeds the drop after the same number of \emph{neg} edits at random span windows outside the run (CI lower bound $>0$). T2 passes if and only if P1-P3 hold. Content pools for \emph{neg} and \emph{pos} are defined by rule labels for all models.

\paragraph{T3: modality masking.}
Per patient and concept, up to 10 active and 10 random home-modality keys are re-run with each present modality removed from fusion, and with only the home modality kept. Criteria: (M1) for every other modality observed in at least 5 patients, the mean paired difference between the absolute probability change from masking the home modality and from masking that modality has CI lower bound $>0$; (M2) the agreement of the home-only decision with the full-model decision has CI lower bound $>0.85$. T3 passes if and only if M1 and M2 hold.

\subsection{Faithfulness results}
\label{app:grounding-faith}

Table~\ref{tab:faith-pass} gives the pass pattern of every model and Table~\ref{tab:faith-t1} the T1 criteria. The failed T1 criteria are C4 for bradycardia and bradypnea (rule-supervised seeds 1 and 2), C2-C4 for bradycardia (rule-supervised seed 3) and C2 for bradypnea (monitor-supervised seed 3). The randomized heads of the other five models fail T3 for every concept; two fail every test, and three pass T1 for one pressure concept (hypotension for rule-supervised seed 3 and monitor-supervised seed 3, hypertension for monitor-supervised seed 2) and T2 for hypotension. The re-initialized models of all six fail T1 and T3 for every concept. For rule-supervised seed 1 with thresholds calibrated on training and validation windows, the patient-mean probability curve crosses the operating threshold at 99.96 [99.42, 100.45]~bpm for tachycardia (definition 100~bpm), 59.7~bpm for bradycardia (60), 104.1~bpm for PPG tachycardia (100), 20.4\,/min for tachypnea (20), 13.2\,/min for bradypnea (12), 66.2~mmHg for hypotension (65) and 107.0~mmHg for hypertension (105). For monitor-supervised heads at 0.5 the crossings are 99.6-99.9, 57.1-58.5, 100.0-100.2, 21.4-21.7, 8.8-9.9, 65.0 and 107.5-108.6 across seeds. The rule estimate on edited windows lies within 5\% of the target for 96\% (ECG), 92\% (PPG) and 89\% (respiration) of edits.

Warping only the ECG lowers C2 for tachycardia from 0.91 [0.88, 0.94] to 0.71 [0.67, 0.75], so heart-rate heads combine the rate across pulsatile channels. The PPG-tachycardia head reads the PPG: a PPG-only warp passes C2 (0.90 [0.88, 0.92]), and masking the PPG changes its probability by 0.375 [0.325, 0.422] against 0.045 [0.031, 0.059] for the ECG. Under sinusoidal interference, the rule detector's C2 falls to 0.70 [0.65, 0.75] for tachycardia and 0.72 [0.68, 0.76] for bradycardia (rule-supervised seed 1: 0.85 [0.80, 0.88] and 0.74 [0.70, 0.78]). The respiratory interference arm is not interpretable, because its artifact frequency (0.40-0.60~Hz, 24-36\,/min) lies in the tachypnea range.

Table~\ref{tab:faith-t2t3} gives the pointer and masking criteria of rule-supervised seed 1; its randomized head fails T2 for all seven concepts (P1 at most 0.60). Two analyses locate the P2 failures (post hoc). For hypertension, editing the quantity that defines the label (maximum sub-window mean set to 95~mmHg) instead of setting MAP to 85~mmHg gives P2 = 0.93 [0.88, 0.97]; with the MAP edit, 46\% of edited windows remain rule-positive. For the rate concepts, restricting to pointer windows whose rule rate agrees with the monitor (within 5~bpm or 2\,/min) gives P2 of 0.99 (PPG tachycardia), 0.96 (tachypnea), 0.43 (tachycardia), 0.31 (bradycardia) and 0.24 (bradypnea).

\begin{table}[h]
\centering
\caption{Faithfulness pass pattern at each model's operating threshold (App.~\ref{app:grounding-cohort}). T1 per concept (P: pass, F: fail) and pass counts of T1-T3 out of 7. Randomized head and re-initialized model: those of rule-supervised seed 1. $^\dagger$Evaluated with the same criteria after registration (exploratory).}
\label{tab:faith-pass}
\resizebox{\ifdim\width>\linewidth\linewidth\else\width\fi}{!}{%
\begin{tabular}{lcccccccccc}
\toprule
 & \multicolumn{7}{c}{T1 by concept} & \multicolumn{3}{c}{Passes} \\
\cmidrule(lr){2-8}\cmidrule(lr){9-11}
Model & Tachy & Brady & Tachy (PPG) & Tachypnea & Bradypnea & Hypo & Hyper & T1 & T2 & T3 \\
\midrule
Rule-supervised, seed 1 & P & F & P & P & F & P & P & 5 & 3 & 5 \\
Rule-supervised, seed 2 & P & F & P & P & F & P & P & 5 & 4 & 4 \\
Rule-supervised, seed 3 & P & F & P & P & P & P & P & 6 & 4 & 5 \\
Monitor-supervised, seed 1 & P & P & P & P & P & P & P & 7 & 3 & 3 \\
Monitor-supervised, seed 2$^\dagger$ & P & P & P & P & P & P & P & 7 & 3 & 4 \\
Monitor-supervised, seed 3$^\dagger$ & P & P & P & P & F & P & P & 6 & 4 & 3 \\
\midrule
Randomized head & F & F & F & F & F & F & F & 0 & 0 & 0 \\
Re-initialized model & F & F & F & F & F & F & F & 0 & n/a & 0 \\
Rule detector & P & P & P & P & P & P & P & 7 & n/a & n/a \\
\bottomrule
\end{tabular}}
\end{table}

\begin{table}[h]
\centering
\caption{T1 criteria. Rule- and monitor-supervised heads: range over 3 seeds (rule-supervised seed 1 at thresholds calibrated on training and validation windows; its pass pattern is the same at the validation-calibrated thresholds). Baseline columns: C2 of the randomized head and re-initialized model of rule-supervised seed 1, and of the rule detector (point estimates).}
\label{tab:faith-t1}
\footnotesize
\setlength{\tabcolsep}{3pt}
\resizebox{\ifdim\width>\linewidth\linewidth\else\width\fi}{!}{%
\begin{tabular}{@{}lccccccccc@{}}
\toprule
 & \multicolumn{2}{c}{C1 $\rho$} & \multicolumn{2}{c}{C2 balanced accuracy} & \multicolumn{2}{c}{C4 neutral flips} & \multicolumn{3}{c}{C2, baselines} \\
\cmidrule(lr){2-3}\cmidrule(lr){4-5}\cmidrule(lr){6-7}\cmidrule(lr){8-10}
Concept & Rule-sup. & Monitor-sup. & Rule-sup. & Monitor-sup. & Rule-sup. & Monitor-sup. & \begin{tabular}[b]{@{}c@{}}Random\\head\end{tabular} & Re-init. & Detector \\
\midrule
Tachycardia (ECG) & 0.444-0.488 & 0.699-0.726 & 0.888-0.913 & 0.989-0.995 & 0.027-0.031 & 0.004-0.009 & 0.36 & 0.52 & 0.96 \\
Bradycardia (ECG) & 0.278-0.377 & 0.518-0.549 & 0.740-0.794 & 0.832-0.878 & 0.087-0.118 & 0.003-0.006 & 0.52 & 0.46 & 0.98 \\
Tachycardia (PPG) & 0.545-0.586 & 0.705-0.741 & 0.879-0.914 & 0.996-0.998 & 0.000 & 0.000 & 0.56 & 0.46 & 0.92 \\
Tachypnea & 0.742-0.772 & 0.925-0.948 & 0.870-0.907 & 0.883-0.896 & 0.011-0.020 & 0.021-0.028 & 0.72 & 0.51 & 0.96 \\
Bradypnea & 0.736-0.765 & 0.876-0.919 & 0.889-0.903 & 0.728-0.757 & 0.074-0.099 & 0.038-0.057 & 0.53 & 0.50 & 0.98 \\
Hypotension & 0.565-0.634 & 0.323-0.374 & 0.993-0.994 & 0.980-0.985 & 0.006-0.012 & 0.000-0.001 & 0.44 & 0.29 & 1.00 \\
Hypertension & 0.655-0.691 & 0.479-0.617 & 0.978-0.981 & 0.879-0.911 & 0.004-0.005 & 0.000 & 0.44 & 0.66 & 1.00 \\
\bottomrule
\end{tabular}}
\end{table}

\begin{table}[h]
\centering
\caption{Pointer (T2) and masking (T3) criteria, rule-supervised seed 1 at thresholds calibrated on training and validation windows (T3 passes 4 of 7 here and 5 of 7 at the validation-calibrated thresholds). $|\Delta\hat q|$: mean absolute probability change when the home modality, or the most influential other modality, is masked.}
\label{tab:faith-t2t3}
\footnotesize
\setlength{\tabcolsep}{3pt}
\resizebox{\ifdim\width>\linewidth\linewidth\else\width\fi}{!}{%
\begin{tabular}{@{}lrccccccc@{}}
\toprule
 & \multicolumn{5}{c}{T2 evidence pointers} & \multicolumn{3}{c}{T3 masking} \\
\cmidrule(lr){2-6}\cmidrule(lr){7-9}
Concept & \begin{tabular}[b]{@{}r@{}}Pointer\\windows\end{tabular} & P1 & P2 & P3 & Pass & \begin{tabular}[b]{@{}c@{}}$|\Delta\hat q|$\\home / other\end{tabular} & M2 & Pass \\
\midrule
Tachycardia (ECG) & 1,188 & 0.72 [0.64, 0.79] & 0.36 [0.28, 0.43] & 0.073 [0.064, 0.082] & F & 0.315 / 0.122 & 0.83 [0.78, 0.87] & F \\
Bradycardia (ECG) & 1,059 & 0.48 [0.39, 0.57] & 0.34 [0.26, 0.43] & 0.062 [0.053, 0.072] & F & 0.319 / 0.224 & 0.73 [0.68, 0.79] & F \\
Tachycardia (PPG) & 778 & 0.94 [0.92, 0.96] & 0.88 [0.80, 0.93] & 0.093 [0.082, 0.105] & P & 0.375 / 0.045 & 0.89 [0.85, 0.93] & F \\
Tachypnea & 1,196 & 0.95 [0.93, 0.96] & 0.81 [0.75, 0.85] & 0.092 [0.084, 0.100] & P & 0.382 / 0.053 & 0.96 [0.94, 0.97] & P \\
Bradypnea & 1,654 & 0.95 [0.94, 0.96] & 0.48 [0.44, 0.52] & 0.065 [0.057, 0.073] & F & 0.390 / 0.074 & 0.90 [0.87, 0.92] & P \\
Hypotension & 493 & 0.98 [0.96, 0.99] & 0.89 [0.83, 0.94] & 0.123 [0.099, 0.153] & P & 0.533 / 0.031 & 0.96 [0.94, 0.98] & P \\
Hypertension & 550 & 0.98 [0.96, 0.99] & 0.60 [0.50, 0.71] & 0.197 [0.172, 0.225] & F & 0.492 / 0.005 & 0.98 [0.96, 0.99] & P \\
\bottomrule
\end{tabular}}
\end{table}

\subsection{Evidence-deletion metrics without re-running the model}
\label{app:grounding-eraser}

The tests above re-run the model on edited inputs. Evidence-deletion metrics \citep{deyoung2020eraser} can instead be computed from stored window probabilities; this shortcut is uninformative. On two random blocks of 160 home-modality windows per patient and concept, the rationale $\mathcal{W}_R$ is the set of windows at or above threshold, the presence score $\mathrm{Pres}(\mathcal{W})$ is the mean of the $|\mathcal{W}_R|$ largest stored probabilities in a window set $\mathcal{W}$, comprehensiveness is $\mathrm{Pres}(\text{all})-\mathrm{Pres}(\text{all}\setminus\mathcal{W}_R)$, sufficiency is $\mathrm{Pres}(\text{all})-\mathrm{Pres}(\mathcal{W}_R)$, and insertion-deletion is the difference of the areas under the presence curves when windows are inserted into or deleted from the block in order of stored probability. The rank variant replaces probabilities by within-block percentile ranks. Because $\mathcal{W}_R$ is exactly the set of the $|\mathcal{W}_R|$ highest stored probabilities, $\mathrm{Pres}(\mathcal{W}_R)=\mathrm{Pres}(\text{all})$ and sufficiency is 0 for every model; on ranks, comprehensiveness and insertion-deletion depend only on $|\mathcal{W}_R|$ and the block size. Because this computation never re-runs the model, it cannot test whether the evidence drives the prediction. Table~\ref{tab:eraser} compares rule-supervised seed 1 with its randomized head: the rank-based metrics agree within their CIs for every concept, and the probability-based comprehensiveness differs through the spread of the stored probabilities.

\begin{table}[h]
\centering
\caption{Evidence-deletion metrics on stored activations, rule-supervised seed 1 / its randomized head. Sufficiency is 0.000 for both in every concept.}
\label{tab:eraser}
\footnotesize
\resizebox{\ifdim\width>\linewidth\linewidth\else\width\fi}{!}{%
\begin{tabular}{lccc}
\toprule
Concept & Comprehensiveness & Comprehensiveness (rank) & Insertion$-$deletion (rank) \\
\midrule
Tachycardia (ECG) & 0.666 / 0.134 & 0.253 / 0.248 & 0.475 / 0.475 \\
Bradycardia (ECG) & 0.594 / 0.092 & 0.208 / 0.215 & 0.474 / 0.474 \\
Tachycardia (PPG) & 0.502 / 0.076 & 0.231 / 0.206 & 0.475 / 0.475 \\
Tachypnea & 0.577 / 0.119 & 0.257 / 0.264 & 0.470 / 0.470 \\
Bradypnea & 0.556 / 0.158 & 0.362 / 0.368 & 0.471 / 0.471 \\
Hypotension & 0.915 / 0.156 & 0.420 / 0.391 & 0.421 / 0.426 \\
Hypertension & 0.912 / 0.078 & 0.129 / 0.129 & 0.475 / 0.474 \\
\bottomrule
\end{tabular}}
\end{table}

\section{Graph statistics and case studies}
\label{app:graphs}

This appendix defines the content of our patient graphs and the correlational baseline (CKG). It also gives the full evidence for Sec.~\ref{sec:exp-graphs} and two case studies.

\subsection{Graph schema}
\label{app:graphs-schema}

Each patient graph is one JSON document. All times are in seconds on the case clock. Beat references are row indices into the patient's beat table (App.~\ref{app:data}). Table~\ref{tab:graph-schema} lists the objects and their fields. No relation instance stores a patient-specific weight: each carries the population estimate and the patient's evidence pointers. Fast edges (the mechanistic edges of Sec.~\ref{sec:method-fast}) are tested per patient and store a descriptive patient estimate (App.~\ref{app:graphs-fast}).

\begin{table}[!htbp]
\centering
\caption{Objects of a patient graph.}
\label{tab:graph-schema}
\scriptsize
\begin{tabular}{@{}p{0.2\linewidth}p{0.74\linewidth}@{}}
\toprule
Object & Fields \\
\midrule
Concept node $v=(c,\kappa)$ & node id; concept code (SNOMED CT \citep{donnelly2006snomed} where one exists, otherwise a local code); preferred term; tier $\kappa$; evidence set $\Pi(v)$: (record, modality, window start, window end, mean activation) per activation run \\
Episode & id; concept; node; start step and number of steps; start and end time; onset-valid flag (the run starts after two observed off steps) \\
Beat episode & the episode fields plus first and last beat row, number of beats, defining statistic and its median \\
Action node & agent; kind (infusion target, infusion rate, ventilator setting); unit; time-stamped actions (time, value before, value after) \\
Variable node & vital (HR, MAP, EtCO$_2$) or beat variable (respiratory phase, beat SBP, RRI, PAT); unit; number of observed windows \\
Relation & id; type$_r$; level (concept or vital); source; target; action and direction; sign of the population estimate and, for vital-level relations, expected sign $\sigma_r$; estimand; population estimate and 95\% CI; lag bins or horizon $h$; sign status and its basis; estimator evidence (odds ratio, $p$-value, exposed events and patients, maxima over null replicates) \\
Edge (relation instance) & id; relation id; source node; target node; sign and sign status; population estimate and CI; pointers: (source reference, target reference, lag, exposure label, source window, target window) \\
Fast relation & id (F1-F4); status (admitted or abstained) with its reason; estimand; unit; population reference \\
Fast edge & source and target variable nodes; patient estimate with epoch-cluster bootstrap CI; patient-level $p$-value and its Holm-adjusted value; tested epochs; pointers: (epoch start and end, first and last beat row, epoch estimate with CI, per-epoch $p$-value and its BH-adjusted value, beats, breaths, epoch summaries such as PPV, SPV and RSA amplitude (App.~\ref{app:fast-epochs}) or mean SBP, MAP and HR) \\
\bottomrule
\end{tabular}
\end{table}

\subsection{Relation registry and sign status}
\label{app:graphs-status}

The registry holds the population relations that may be instantiated. On VitalDB it has 81 concept-level relations (App.~\ref{app:relations}) and 10 vital-level relations. The vital-level relations are those admitted by the typed pooled estimator (App.~\ref{app:benchmark}) among the 18 action-vital relations of App.~\ref{app:benchmark-30s}. The estimator was run at 30\,s on all 3{,}442 cases and at 2\,s on the 500-case subsample of App.~\ref{app:fast-spec}, calibrated against a single cross-patient transplant draw ($K=1$). Admission requires the calibrated CI to exclude 0 and a Benjamini-Yekutieli-adjusted $p\le0.05$ over the 18 relations \citep{benjamini2001control}; the stability condition of Eq.~\ref{eq:admission} is not applied. Each relation also stores the share of nine calibrators that admit it: the eight cross-patient transplants of App.~\ref{app:benchmark} (one of which is the admission draw) and the random-time stream. On MIMIC-IV-WDB the registry has the 6 association relations admitted in the 51-patient fit. The status rules below were fixed before the graphs were built.
\begin{itemize}[leftmargin=*,itemsep=1pt,topsep=2pt]
\item \emph{Vital-level relations.} \emph{Sign-stable} if admitted with the expected sign at one or more resolutions and never admitted with the opposite sign; \emph{sign-unresolved} if admitted with the opposite sign at either resolution.
\item \emph{Concept-level association and policy relations.} \emph{Sign-stable} if admitted by rule (i) (App.~\ref{app:relations-admission}) in both pooled half-record fits (H1 and H2: first and second halves of every record; App.~\ref{app:relations-admitted}) with the full-data sign. \emph{Sign-unresolved} if either half admits the opposite sign, or if the sign contradicts a pre-specified sign check. Otherwise \emph{unconfirmed}.
\item \emph{Concept-level response relations.} Never sign-stable, because the family fails its pre-specified sign checks (App.~\ref{app:relations-signs}). Sign-unresolved under the same two conditions; otherwise unconfirmed.
\item \emph{MIMIC-IV-WDB association relations.} The half-split rule, applied to the MIMIC-IV-WDB fits.
\end{itemize}
Table~\ref{tab:graph-registry} gives the counts. The two sign-unresolved vital-level relations are the responses of HR and MAP to remifentanil decreases, which are admitted with the sign opposite to pharmacology at 30\,s. The sign-unresolved concept-level relation is remifentanil$\uparrow\to$ hypotension, which failed its sign check (Table~\ref{tab:rel-signs}).

\begin{table}[!htbp]
\centering
\caption{Population relation registry by level, type and sign status.}
\label{tab:graph-registry}
\small
\begin{tabular}{@{}lllccc@{}}
\toprule
Cohort & Level & Type & Sign-stable & Sign-unresolved & Unconfirmed \\
\midrule
VitalDB & concept & association & 12 & 0 & 16 \\
 & concept & policy & 23 & 0 & 10 \\
 & concept & response & 0 & 1 & 19 \\
 & vital & policy & 3 & 0 & 0 \\
 & vital & response & 5 & 2 & 0 \\
MIMIC-IV-WDB & concept & association & 4 & 0 & 2 \\
\bottomrule
\end{tabular}
\end{table}

\subsection{Instantiation and evidence pointers}
\label{app:graphs-pointers}

A registry relation is instantiated in patient $p$ if and only if $p$ has at least one evidence pair. Each pair becomes a pointer. Pairs follow the exposure and lag structure of the estimator, with steps $\Delta=15$\,s:
\begin{itemize}[leftmargin=*,itemsep=1pt,topsep=2pt]
\item \emph{Association} $A\to B$: an episode of $B$ that starts at a valid onset at step $i$, together with an episode of $A$ that covers step $i-2$ (exposure ``state'') or a valid onset of $A$ at step $i'$ with $2\le i-i'\le 40$ (exposures ``onset 0-1\,min'', ``1-3\,min'' and ``3-10\,min'').
\item \emph{Policy} $A\to(a,d)$: a change of $(a,d)$ at step $i$, together with an episode of $A$ that covers step $i-1$ or a valid onset of $A$ at step $i'$ with $1\le i-i'\le 40$.
\item \emph{Response} $(a,d)\to B$: a change of $(a,d)$ at step $j$ and a valid onset of $B$ at $i$ with $5\le i-j\le 40$.
\item \emph{Vital-level response}: an intention-to-treat event of the relation's action and direction at time $\tau$ (no change of that action in the preceding 180\,s), with the vital observed in the baseline window $[\tau-30\,\mathrm{s},\tau)$ and in the outcome window $[\tau+h-15\,\mathrm{s},\tau+h+15\,\mathrm{s}]$; the lag is $h$.
\item \emph{Vital-level policy}: the same event set, with the vital observed in $[\tau-150\,\mathrm{s},\tau-120\,\mathrm{s}]$ and $[\tau-30\,\mathrm{s},\tau)$.
\end{itemize}
Node-level edges connect the endpoint nodes of the pointers: the tier node of each episode, the action node, or the variable node. One relation instance can therefore produce several node-level edges. For relations with $\hat\omega_r<0$ (7 association, 11 policy and 16 response relations), the pointers are exposed events that occurred despite the lower population hazard. Pointers record the events that instantiate a relation in the patient (its provenance). They are not evidence of the relation's direction or sign in that patient, and they do not make a relation patient-specific.

\textbf{Agreement with the estimator.} For each relation we count the patients with at least one pointer that carries the $\omega_r$ exposure. This count equals the estimator's number of contributing patients for 28 of 28 association relations, 20 of 20 response relations and 17 of 33 policy relations. For the other 16 policy relations the pointer rule counts 1-19 more patients, which is 147 of 23{,}998 patient-relation instances (0.6\%). The pointer rule does not apply the estimator's at-risk condition (recorded dose $>0$ at step $i-1$). It also counts source onsets in the onset lag bins, which are not part of the $\omega_r$ exposure (the source on at the source lag); 36.2\% (association) and 43.5\% (policy) of instances rest on such onset-bin pointers only. Of the 183 onset-bin coefficients of admitted association and policy relations, 61 have the sign opposite to $\omega_r$, and 23 of these have a CI that excludes 0.

\textbf{Pointer versus presence.} A presence rule instantiates a relation whenever the source has an episode or change anywhere in the record and the target has an onset or change anywhere. It would create 166{,}563 concept-level instances on VitalDB, and 96{,}309 of them (57.8\%) have at least one in-window evidence pair. Every pointer-based instance is also a presence-rule instance.

\textbf{Pointer verification.} On 100 random VitalDB cases we re-derived concept states and action events from the raw monitor records with an independent implementation of the same definitions. We then checked every pointer: the source is on or has an onset at the claimed step, the target onset is valid, the lag is inside the relation's window, the action is present at the claimed time, and the vital is observed in the claimed windows. All 12{,}346 pointers are valid. By construction, every relation edge has at least one pointer (191{,}490 of 191{,}490 node-level edges on VitalDB and 2{,}586 of 2{,}586 on MIMIC-IV-WDB). On 100 random graphs with beat-level content, all 7{,}594 fast-edge pointers have beat ranges inside their epochs and reproduce the stored epoch estimates. This includes all 2{,}053 PAT-SBP pointers in these graphs. All 6{,}589 beat episodes have valid beat ranges. An independent re-implementation of the fast-coupling tests on 40 cases reproduces the stored epoch amplitudes exactly. A recomputation of the PAT-SBP edge on 82 cases reproduces every instantiation decision under the real data and all three within-patient nulls.

\subsection{Correlational KG construction (CKG)}
\label{app:graphs-ckg}

CKG is a correlational construction: it applies lagged cross-correlation and Granger tests, as used for graphs from physiological time series, and Jaccard co-occurrence to concept series. It is not a reproduction of the time-delay-stability method of \citet{bashan2012network}, which tracks the stability of the cross-correlation delay over consecutive windows and tests it against surrogate data. CKG uses the same concept series and the same concept $\times$ tier nodes as our graphs, on the 15\,s grid. Unobserved steps count as off. A concept is active in a window when its probability is at least its activation threshold. On VitalDB the probabilities are the 0/1 concept states from monitor numerics; on MIMIC-IV-WDB they are the concept-head activations of rule-supervised seed 1. For every unordered pair of concepts $(A,B)$ in the same record, CKG applies three rules:
\begin{itemize}[leftmargin=*,itemsep=1pt,topsep=2pt]
\item \emph{Precedes/follows.} Compute the normalized cross-correlation of the mean-centered binary activation series of $A$ and $B$ at every lag $k$ with $|k|\le 120$ steps ($\pm30$\,min), and take the lag $k^\ast$ of the maximum. If the peak is at least 0.15 and $|k^\ast|\Delta\ge15$\,s, emit a PRECEDES edge from every tier node of the leading concept to every tier node of the lagging concept, and the mirror FOLLOWS edge.
\item \emph{Granger-precedes} \citep{granger1969investigating}. For each direction, run a bivariate Granger $F$-test on the concept probability series at lag orders $\{1,4,12\}$ steps (15\,s, 1\,min, 3\,min). The test needs at least 60 samples and non-zero variance. If the minimum $p$-value over the three lag orders is below 0.05, emit a GRANGER\_PRECEDES edge for every tier-node pair in that direction.
\item \emph{Co-occurrence.} If the Jaccard overlap of the two sets of active windows is at least 0.10, emit a symmetric pair of CO\_OCCURS edges for every tier-node pair.
\end{itemize}
CKG applies no multiplicity control, no negative-control calibration and no action or variable nodes, and it stores no edge pointers. Per VitalDB graph it emits 188.11 edges: 43.31 PRECEDES, 43.31 FOLLOWS, 72.39 GRANGER\_PRECEDES and 29.10 CO\_OCCURS. Removing mirror FOLLOWS edges and one edge of each CO\_OCCURS pair leaves 130.25. Per MIMIC-IV-WDB graph it emits 120.91 edges: 5.57, 5.57, 80.22 and 29.54, in the same order. The CKG temporal layer is the union of the precedes/follows and Granger-precedes rules. In Sec.~\ref{sec:exp-graphs}, CKG replication and false relations are counted as directed concept pairs of this layer among the 48 cross-vital ordered pairs.

\subsection{Graph size}
\label{app:graphs-size}

Table~\ref{tab:graph-size} gives per-patient means. Graphs with beat-level content are the 2{,}954 VitalDB cases that pass beat-level quality control (App.~\ref{app:data}). The other 488 graphs record their registry of fast relations and the reason they have no beat content. MIMIC-IV-WDB graphs have no action records, no beat layer and association relations only. Their concept$\times$tier nodes and evidence pointers are the activation runs of rule-supervised seed 1 (App.~\ref{app:concepts-model}), and their association relations are estimated on the same activations. Their concept episodes are numerous (median 1{,}128 per patient) because concept-head activations switch often at the 15\,s resolution.

\begin{table}[!htbp]
\centering
\caption{Per-patient graph size (means). VitalDB: all 3{,}442 cases and the 2{,}954 with beat-level content; MIMIC-IV-WDB: 167 records. Sign status: sign-stable / sign-unresolved / unconfirmed. n/a: not represented.}
\label{tab:graph-size}
\scriptsize
\begin{tabular}{@{}lccc@{}}
\toprule
 & VitalDB, all & VitalDB, beat content & MIMIC-IV-WDB \\
\midrule
Concept $\times$ tier nodes & 13.12 & 13.48 & 14.05 \\
Beat-concept $\times$ tier nodes & 6.25 & 7.29 & n/a \\
Action nodes & 2.42 & 2.46 & n/a \\
Variable nodes & 6.35 & 6.98 & n/a \\
Nodes, total & 28.14 & 30.20 & 14.05 \\
\addlinespace
Concept episodes & 43.31 & 44.54 & 2{,}073.14 \\
Beat episodes & 60.94 & 71.00 & n/a \\
\addlinespace
Relation instances & 35.74 & 37.24 & 2.14 \\
\quad sign status & 20.44 / 2.00 / 13.31 & 21.37 / 2.11 / 13.76 & 1.21 / 0.00 / 0.93 \\
\quad policy / response / assoc.\ & 14.80 / 10.84 / 10.11 & 15.47 / 11.28 / 10.49 & 0 / 0 / 2.14 \\
Relation edges (node level) & 55.63 & 57.91 & 15.49 \\
\quad policy / response / assoc.\ & 21.24 / 13.42 / 20.97 & 22.24 / 13.97 / 21.69 & 0 / 0 / 15.49 \\
Fast edges (F1 / F2 / F4) & 0.84 / 0.83 / 0.74 & 0.98 / 0.97 / 0.86 & n/a \\
Edges, total & 58.05 & 60.72 & 15.49 \\
\addlinespace
Relation-edge pointers & 125.75 & 131.99 & 395.27 \\
Fast-edge epoch pointers & 67.49 & 78.63 & n/a \\
\addlinespace
CKG nodes / edges & 13.12 / 188.11 & 13.48 / 195.03 & 14.05 / 120.91 \\
\bottomrule
\end{tabular}
\end{table}

\subsection{False relation instances under null streams}
\label{app:graphs-fe}

\textbf{Our graphs.} For each of the 38 null replicates of App.~\ref{app:relations-admission}, we recompute population admission with the full rule, where rule (ii) uses the other 37 replicates. We then instantiate every null-admitted relation in every graph with the pointer rule, applied to the null-mapped series. Every such instance is false, and $\mathrm{FE}_p$ counts them over all 176 candidate concept-level relations. We report $\overline{\mathrm{FE}}$, the mean of $\mathrm{FE}_p$ over patients (on MIMIC-IV-WDB, the 51 patients of the relation fit) and over the 19 replicates of each null kind. The 95\% CI is a percentile bootstrap over the 19 replicate means. A two-stage bootstrap over replicates and then patients gives [0.015, 0.106] (transplant) and [0.032, 0.190] (shift) on VitalDB.

On VitalDB, 12 of the 19 transplant replicates admit no null relation. The other seven give 0.015-0.395 false instances per graph, for a mean of 0.057. Under the shift null, 10 of 19 replicates give zero and the other nine give 0.011-0.589, for a mean of 0.103. An admitted null relation is instantiated in 1-56\% of VitalDB graphs and in 90-100\% of MIMIC-IV-WDB graphs, whose records are longer and have more episodes. Of all (graph, replicate) pairs, 5.5\% (transplant) and 10.1\% (shift) contain a false instance on VitalDB, and 9.5\% and 25.4\% on MIMIC-IV-WDB. False node-level edges per graph are 0.081 and 0.154 on VitalDB and 0.583 and 1.903 on MIMIC-IV-WDB. A real VitalDB graph has 27.98 concept-level relation instances, so false instances amount to 0.20\% and 0.37\% of them. Instantiating the same null-admitted relations by presence instead of by pointers would give 0.20 and 0.25 false instances per graph.

\textbf{CKG and per-patient baselines.} These methods test concept pairs within each patient. For CKG and pairwise Granger, all eight concept series are transplanted or circularly shifted, and each null-mapped series is tested as the source against the real series of every concept on another vital. All 48 cross-vital ordered pairs are then null pairs, and $\mathrm{FE}_p$ counts those with an edge. PCMCI+ fits one multivariate model per case, so it uses a block null: the HR and EtCO$_2$ concepts are replaced by their null-mapped series, the MAP and respiratory-rate concepts stay real, and $\mathrm{FE}_p$ counts edges among the 32 cross-group ordered pairs (BH control remains over the 48 pairs). Table~\ref{tab:fe-baselines} reports the first replicate ($k=1$) on all patients. On replicates $k=1,\dots,5$ of 300 random VitalDB cases, CKG gives 12.63 [12.25, 13.03] (replicate means 12.34-12.84) under the transplant null and 10.18 [9.66, 10.72] (9.96-10.34) under the shift null. On VitalDB, 26.7\% (transplant) and 21.7\% (shift) of the null pairs with both concepts observed receive a CKG edge, and 99.7\% and 97.4\% of graphs contain at least one false CKG edge. These per-patient designs differ from ours, which admits relations at the population level over 176 candidate concept-level relations with 38 null replicates; the per-graph counts therefore compare constructions, not a matched test.

\begin{table}[!htbp]
\centering
\caption{False edges per graph for per-patient methods (App.~\ref{app:relations-replication}; first replicate; 95\% CI from a patient bootstrap). Null pairs per patient: 48 (CKG, pairwise Granger) and 32 (PCMCI+, block null). PCMCI+ on VitalDB uses 400 cases.}
\label{tab:fe-baselines}
\scriptsize
\begin{tabular}{@{}lcccc@{}}
\toprule
 & \multicolumn{2}{c}{VitalDB ($3{,}442$)} & \multicolumn{2}{c}{MIMIC-IV-WDB ($51$)} \\
\cmidrule(lr){2-3}\cmidrule(l){4-5}
Method & transplant & shift & transplant & shift \\
\midrule
CKG cross-correlation & 7.45 & 6.48 & 1.04 & 1.16 \\
CKG Granger-precedes & 7.94 & 5.54 & 10.69 & 8.94 \\
CKG (union) & 12.40 [12.23, 12.58] & 10.07 [9.86, 10.27] & 11.10 [9.75, 12.45] & 9.53 [8.21, 10.77] \\
Pairwise Granger + BH & 2.78 & 1.86 & 0.94 & 0.24 \\
PCMCI+ & 0.79 & 0.85 & 0.14 & 0.10 \\
\bottomrule
\end{tabular}
\end{table}

\textbf{Vital-level relations.} In a null-versus-null test, a cross-patient null draw is used as the input stream and another cross-patient draw as the calibrator. Across the 56 ordered pairs of draws at each resolution, no vital-level relation is admitted. With a random-time null stream as input and each of the eight transplants as calibrator, 30\,s admission yields EtCO$_2$ rising $\to$ ventilator rate$\uparrow$ for 8 of 8 calibrators, MAP rising $\to$ remifentanil$\uparrow$ for 2 of 8, and MAP falling $\to$ propofol$\downarrow$ for 1 of 8. Instantiating these admissions would place 1.15 false vital-level edges in the average graph. The first relation is instantiated in 89\% of graphs. At 2\,s no relation is admitted under the random-time input. MAP rising $\to$ remifentanil$\uparrow$ is also admitted on the real data at 30\,s, with all nine calibrators.

\subsection{Fast-coupling layer}
\label{app:graphs-fast}

The fast-coupling registry is the same in every graph (App.~\ref{app:fast} gives the beat pipeline and the population analyses):
\begin{itemize}[leftmargin=*,itemsep=1pt,topsep=2pt]
\item \emph{F1, respiration $\to$ beat SBP.} Harmonic amplitude of 15-beat-detrended beat SBP on the respiratory phase. Admitted for per-patient testing.
\item \emph{F2, respiration $\to$ RRI} (respiratory sinus arrhythmia). The same statistic on $\mathrm{RRI}_n$. Admitted for per-patient testing.
\item \emph{F3, beat SBP $\to$ RRI} (baroreflex). Abstained: the beat-level slope has the expected positive sign in 63.3\% [61.5, 64.9] of 2{,}944 cases, below the positive-control bar of 90\%. The reason is stored in the graph.
\item \emph{F4, PAT $\leftrightarrow$ beat SBP.} Epoch slope $b_{p,e}$ of detrended $\mathrm{PAT}_n$ on detrended $\mathrm{SBP}_n$. Admitted: the case slope is negative in 94.3\% [93.4, 95.1] of 2{,}884 cases. The edge is stored as undirected: it asserts an association, not a causal direction. Absolute PAT includes a constant per-case device delay, so only within-case variation is used.
\end{itemize}

\textbf{Per-patient test.} The epochs $e$ are the usable 5-min epochs of App.~\ref{app:fast-epochs} with at least 60 beats. F1 and F2 use epochs with controlled ventilation; F4 uses all usable epochs. For each epoch, a $p$-value is computed from 499 surrogates. For F1 and F2 the surrogates add an independent uniform phase offset to each breath. For F4 they circularly shift SBP against PAT by 30 to $n_e-30$ beats, where $n_e$ is the number of beats in the epoch; this test is one-sided, for a negative slope. The patient statistic is the mean over tested epochs, and its $p$-value is computed against the corresponding means of the surrogate draws. A patient needs at least three tested epochs. Holm control at family-wise level 0.05 \citep{holm1979simple} is applied over the patient's eligible relations among F1, F2 and F4. Pointers are the epochs whose per-epoch $p$-value passes BH control at level 0.05 within patient and relation. An edge is instantiated if and only if Holm control passes and the edge has at least one pointer. The edge stores the patient mean with an epoch-cluster bootstrap CI as a descriptive estimate. Beat-level concept nodes (irregular R-R intervals, low heart-rate variability, high pulse-pressure variation and prolonged PAT; App.~\ref{app:concepts-beat}) add 7.29 concept $\times$ tier nodes and 71.00 beat episodes per graph with beat content. Each beat episode points to its first and last beat rows.

\textbf{Instantiation.} Table~\ref{tab:fast-nulls} gives instantiation rates among eligible patients. Graphs with beat content have 2.81 [2.80, 2.83] fast edges on average (8{,}312 edges in 2{,}954 graphs, all with pointers); the 2{,}937 eligible graphs have 2.83 [2.81, 2.84], and 99.7\% [99.5, 99.9] of them have at least one. F4 is instantiated in 2{,}554 patients. Of the other eligible patients, 303 fail Holm control and 20 pass it without a pointer. The median patient mean slope is $-0.613$\,ms/mmHg over instantiated edges (interquartile range $-0.838$ to $-0.435$) and $-0.571$ over all eligible patients. Of the 62{,}304 F4 pointers, 95.1\% have an epoch CI below 0 and none has a positive slope.

\begin{table}[!htbp]
\centering
\caption{Fast-edge instantiation (share of eligible patients [95\% case-bootstrap CI]) and unadjusted patient $p\le0.05$, under real data and null streams. Eligible: 2{,}904 (F1, F2) and 2{,}877 (F4) patients; N3: 600. N3 CI: exact binomial.}
\label{tab:fast-nulls}
\scriptsize
\setlength{\tabcolsep}{3pt}
\begin{tabular}{@{}lcccccc@{}}
\toprule
 & \multicolumn{2}{c}{F1 resp.\ $\to$ SBP} & \multicolumn{2}{c}{F2 resp.\ $\to$ RRI} & \multicolumn{2}{c}{F4 PAT $\leftrightarrow$ SBP} \\
\cmidrule(lr){2-3}\cmidrule(lr){4-5}\cmidrule(l){6-7}
Input & instantiated & $p\le0.05$ & instantiated & $p\le0.05$ & instantiated & $p\le0.05$ \\
\midrule
Real data & 100.0\% [100.0, 100.0] & 100.0\% & 98.3\% [97.8, 98.8] & 98.7\% & 88.8\% [87.7, 90.0] & 89.5\% \\
N1 shift & 97.3\% [96.8, 97.9] & 98.1\% & 72.8\% [71.1, 74.3] & 77.4\% & 0.1\% [0.0, 0.2] & 1.9\% \\
N2 surrogate & 0.1\% [0.0, 0.2] & 4.9\% & 0.4\% [0.2, 0.6] & 5.9\% & 0.0\% [0.0, 0.1] & 2.6\% \\
N-b across-epoch & n/a & n/a & n/a & n/a & 1.0\% [0.7, 1.4] & 6.8\% \\
N3 cross-case & 11.5\% [9.1, 14.3] & 14.2\% & 3.2\% [1.9, 4.9] & 4.7\% & n/a & n/a \\
\bottomrule
\end{tabular}
\end{table}

\textbf{Null streams and their interpretation.} The same test, Holm control and pointer rule are applied to null data, with new surrogate draws.
\begin{itemize}[leftmargin=*,itemsep=1pt,topsep=2pt]
\item \emph{N1, within-patient circular shift.} In each epoch the target series is rolled against its driver by a number of beats drawn uniformly from $[\lceil n_e/4\rceil,\lfloor 3n_e/4\rfloor]$. Controlled ventilation is nearly periodic, so rolling the target rotates the respiratory phase by a nearly constant angle and leaves the phase locking intact. The N1 rates for F1 and F2 therefore measure this property of shift nulls under a periodic driver, not the false-edge rate of the rule; this reading was stated before the analysis. F4 has no periodic driver, and its N1 rate is 0.1\%.
\item \emph{N2, surrogate draw.} The data are replaced by one independent draw from the test's own surrogate distribution. N2 is exchangeable with the surrogates and serves as an implementation check. Unadjusted patient $p\le0.05$ occurs in 4.9\% (F1), 5.9\% (F2) and 2.6\% (F4) of patients. Under N2, 0.5\% [0.3, 0.8] of eligible graphs contain a fast edge.
\item \emph{N-b, across-epoch pairing} (F4 only). The PAT of one epoch is paired with the SBP of another epoch of the same patient. N-b is not exchangeable with the surrogates. It gives 1.0\% false F4 edges, with unadjusted $p\le0.05$ in 6.8\% of patients.
\item \emph{N3, cross-case pairing} (F1 and F2; post hoc; 600 cases; Holm control over F1 and F2). The respiratory phase of one patient is paired with the beat series of another. It gives 11.5\% false F1 edges and 3.2\% false F2 edges. False F1 edges concentrate in pairs whose breaths-per-beat ratios match within 5\% (33.7\% of 86 pairs). The rate is 3.4\% of 294 pairs whose ratios differ by at least 20\%.
\end{itemize}
Across N1, N2 and N-b, the F4 false-edge rate is at most 1\%.

\subsection{Case studies}
\label{app:graphs-case}

\textbf{Case shown in Figure~\ref{fig:casestudy}.} The case was selected by a pre-specified rule: the case with the most prolonged-PAT episode time among cases with an instantiated PAT-SBP edge. It is case 2706. Its graph has 13 concept$\times$tier nodes with 58 episodes, 9 beat-concept$\times$tier nodes with 125 beat episodes, 2 action nodes (remifentanil target and ventilator rate) and 7 variable nodes (3 vital, 4 beat). Its 41 relation instances are realized as 73 node-level edges with 240 pointers. The fast layer has 3 edges with 190 epoch pointers. F1 has a patient amplitude of 3.05\,mmHg [2.84, 3.28] over 82 epochs (81 pointers), F2 3.26\,ms [2.81, 3.75] (65 pointers), and F4 a patient slope of $-0.45$\,ms/mmHg [$-0.81$, $-0.12$] over 84 epochs (44 pointers); each has a patient-level $p$-value of 0.002, the smallest attainable with 499 surrogates, and a Holm-adjusted value of 0.006. The longest prolonged-PAT episode spans 9{,}150-11{,}235\,s (persistent tier; beat rows 10{,}817-13{,}019), with a median PAT 34.4\,ms above the case baseline. F3 is abstained with its stored reason.

\begin{figure}[!htbp]
\centering
\includegraphics[width=\linewidth]{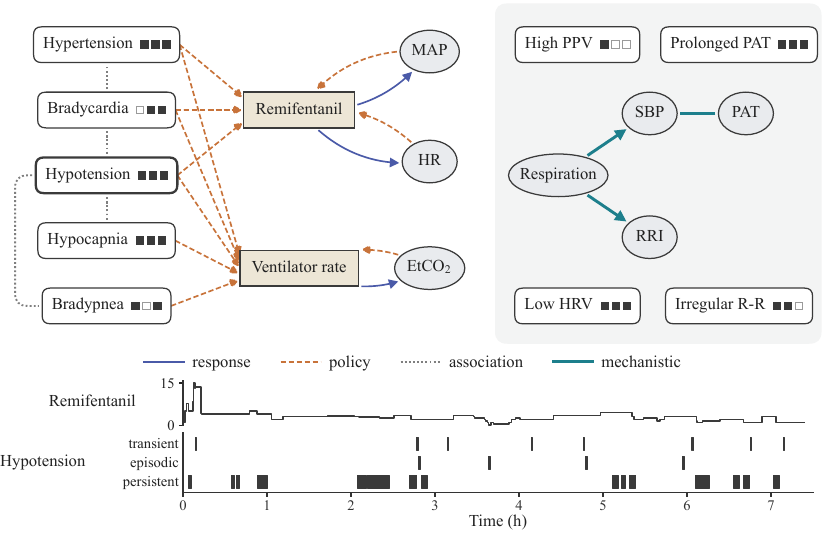}
\caption{Patient graph of VitalDB case 2706. Boxes: concepts (filled squares: tiers present, transient to persistent); shaded boxes: actions; ellipses: variables; gray region: beat level. Sign-stable relations are drawn aggregated over tiers and action direction. Bottom: remifentanil target (ng\,mL$^{-1}$) and the hypotension episodes (bold node) by tier.}
\label{fig:casestudy}
\end{figure}

\textbf{A second case with pointer detail.} A second case is the first, by case number, of three cases drawn at random, with a seed fixed before the graphs were built, from the 3{,}258 VitalDB cases that have at least one sign-stable policy relation, one sign-stable vital-level response relation and one association relation. It is case 52. Its graph has the following content:
\begin{itemize}[leftmargin=*,itemsep=1pt,topsep=2pt]
\item nine concept $\times$ tier nodes with 28 episodes (tachycardia 11, hypotension 10, hypocapnia 3, bradypnea 3, hypertension 1);
\item ten beat-concept $\times$ tier nodes with 69 beat episodes;
\item three action nodes (propofol target, 28 actions; remifentanil target, 29 actions; ventilator rate, 7 actions) and seven variable nodes;
\item 32 relation instances, realized as 40 node-level edges with 151 pointers, and three fast edges.
\end{itemize}
\textbf{Sign-stable relations.} All eight sign-stable concept-level relations are policy relations. Tachycardia $\to$ remifentanil$\uparrow$ (population $\hat\omega_r=0.55$ [0.45, 0.64]) has 17 pointers. The first links a tachycardia episode at 13{,}020-13{,}065\,s to a remifentanil target increase from 2.5 to 3.5\,ng\,mL$^{-1}$ at 13{,}414\,s (lag 375\,s, exposure ``onset 3-10\,min''). Eight vital-level relations are sign-stable: three policy relations (MAP rising and HR rising $\to$ remifentanil$\uparrow$; EtCO$_2$ rising $\to$ ventilator rate$\uparrow$) and five responses (EtCO$_2$ to ventilator-rate increases and decreases; MAP to propofol increases; HR and MAP to remifentanil increases). For example, an increase of the ventilator rate from 13 to 15\,/min at 8{,}867\,s points to the EtCO$_2$ window 60\,s later.

\textbf{Sign-unresolved and unconfirmed relations.} The responses of HR and MAP to remifentanil decreases are sign-unresolved, with four pointers each; the first points to a target decrease from 7 to 2\,ng\,mL$^{-1}$ at 810\,s. The concept-level response remifentanil$\uparrow\to$ hypotension is also sign-unresolved. The remaining 13 instances are unconfirmed: seven concept-level responses, five policy relations and one association.

\textbf{Fast couplings.} F1 has a patient amplitude of 4.32\,mmHg [3.98, 4.68] over 44 controlled-ventilation epochs, all of them pointers. F2 has 8.65\,ms [7.86, 9.56], also with 44 pointers. F4 has a patient slope of $-0.59$\,ms/mmHg [$-0.72$, $-0.47$] over 48 epochs, with 32 pointers. Each has a patient-level $p$-value of 0.002, the smallest attainable with 499 surrogates, and a Holm-adjusted value of 0.006. F3 is abstained with its stored reason.

\section{Pre-specified protocols and reproducibility}
\label{app:prereg}

\paragraph{Protocols.}
Every confirmatory analysis was specified in a written protocol that fixed the cohort and quality rules, the estimands and estimators, the null constructions, the decision rules and, where applicable, predictions. Each protocol was written and hashed (SHA-256) before its confirmatory run; in the appendices, \emph{registered} means written and hashed in this way. No external registry was used; each protocol states when it was written. Every confirmatory output records the protocol hash, and most analyses check the hash before running; the benchmark protocols also archive the registered code. Each protocol lists what had been seen before it was hashed: functional checks on small case subsets (for the concept-level relations, exploratory pilots on 250 and 400 cases that shaped the design, with all cases used in the confirmatory run), data-availability counts and the outcomes of analyses registered earlier in Table~\ref{tab:prereg}. Predictions are reported whatever their outcome.

\begin{table}[!t]
\centering
\caption{Confirmatory analyses. Hash: first eight hexadecimal digits of the protocol's SHA-256.}
\label{tab:prereg}
\scriptsize
\setlength{\tabcolsep}{4pt}
\begin{tabular}{@{}p{2.6cm}p{1.3cm}p{3.6cm}p{5.6cm}@{}}
\toprule
\raggedright Analysis (appendix) & \raggedright Hash & \raggedright Fixed in advance & \raggedright Registered outcome \tabularnewline
\midrule
\raggedright ClosedLoopBench: 29 relations, 2~s, 3,442 cases (C) & \raggedright \texttt{418281f1} & \raggedright grid, clock; relation signs, horizons, domains; 12 estimators; 8 cross-patient derangements and a random-time arm; admission rule; 6 predictions & \raggedright held: ${\le}\,1$ false admission per held-out transplant draw; ${\ge}\,4$ uncalibrated null discoveries (3 estimators); gas-exchange relations recovered more often than drug relations (8 of 12 estimators); ${\ge}\,1$ wrong-signed drug admission. Not met: ${\ge}\,80$\% of gas-exchange admissions correct (76\%); FiO$_2{\to}$SpO$_2$ admitted by an event design \tabularnewline
\addlinespace[1pt]
\raggedright ClosedLoopBench: 18 relations, 30~s (3,442 cases) and 2~s (500-case subsample) (C) & \raggedright \texttt{5e578794} & \raggedright calibration with $K = 8$ draws, stability requirement, admission rule; 5 predictions; patient-specificity test & \raggedright held: ${\le}\,1$ false admission per held-out transplant draw; ${\ge}\,4$ of 18 uncalibrated null discoveries (4 estimators, 30~s); wrong-signed remifentanil-decrease admission by VARX, the typed pooled estimator (30~s). Not met: ${\ge}\,2$ of 6 policy relations with none wrong (VARX: 2 correct, 1 wrong); fewer wrong-signed responses for the sequential-trial contrast than for VARX (2 each at 30~s). Patient-specificity rule met by 0 of 11 relations (shrunk better in 0/33 rows at each resolution) \tabularnewline
\addlinespace[1pt]
\raggedright Slow-response specificity (F) & \raggedright \texttt{482db32f} & \raggedright early/late split; shrunk-versus-pooled prediction rule; 11 relations & \raggedright 0/11 relations pass \tabularnewline
\addlinespace[1pt]
\raggedright Clinician-in-the-loop simulator (D) & \raggedright \texttt{7b75ce95}\newline +\texttt{e9e25964} & \raggedright policies, drugs, control pools, horizons; 14 hypotheses (Table~\ref{tab:sim-hyp}); amendment: descriptive phase-grid analysis & \raggedright 9 of 14 held, including: eligible pool unbiased and future-inaction pool biased for an inert drug; eligible pool biased under unrecorded information. Not met: L3 (reversed-policy sign held at 1 of 3 horizons), L3b (standard policy, decreases, at 110~s), L4 (eligible pool biased for an active drug at 10~s), E1 (kink with eligible pool, $+0.32$ [0.15, 0.48] at 10~s) and B1 (beat-level bias range 0.063 in 1 of 10 cells) \tabularnewline
\addlinespace[1pt]
\raggedright Concept-level typed relations (E) & \raggedright \texttt{43947eaf} & \raggedright concepts, onsets, 3 families; pooled hazard with patient fixed effects; BH plus 38 null replicates; sign checks; split-half replication, held-out likelihood & \raggedright policy sign checks 5/5 correct; response sign checks 0/6 correct, 2 wrong \tabularnewline
\addlinespace[1pt]
\raggedright Patient graphs, concept layer (H) & \raggedright \texttt{c767688d} & \raggedright pointer rule; sign status; false instances under 38 null replicates; prediction ${<}\,0.5$ per graph & \raggedright held: 0.057 (transplant) and 0.103 (shift) false instances per graph \tabularnewline
\addlinespace[1pt]
\raggedright Beat concepts and fast couplings in graphs (F, H) & \raggedright \texttt{eab2d583}\newline +\texttt{20f873f5} & \raggedright beat concepts, cohort thresholds; fast-edge registry, per-patient tests; null graphs; addendum: PAT-SBP edge, \emph{prolonged PAT} & \raggedright respiration$\to$SBP in 100\% and RSA in 98.3\% of eligible cases; baroreflex abstained; circular-shift null not valid for ventilator-locked couplings; PAT-SBP in 88.8\% of 2,877 eligible cases, null false-edge rate ${\le}\,1\%$ \tabularnewline
\addlinespace[1pt]
\raggedright Fast couplings at scale (F) & \raggedright \texttt{af14df06} & \raggedright gates; epoch statistics; baroreflex bar (CI lower bound ${\ge}\,0.90$); 7 construct-validity directions; monitor latency; specificity rule & \raggedright baroreflex bar not met (63.3\% [61.5, 64.9]); 4 of the 6 counted construct-validity items hold (CV7 not counted); specificity rule met for PPV, RSA, sequence baroreflex sensitivity and RMSSD \tabularnewline
\addlinespace[1pt]
\raggedright PAT analyses (F) & \raggedright \texttt{230b57c5} & \raggedright PAT source, gates; within-case nulls; bar for the slope sign; specificity rule & \raggedright bar met (94.3\% [93.4, 95.1] of 2,884 cases); specificity rule met ($r_{\mathrm{EL}} = 0.67$ [0.63, 0.70]) \tabularnewline
\addlinespace[1pt]
\raggedright Fast-coupling specificity, MIMIC-IV-WDB (50) and VitalDB (24) (F) & \raggedright \texttt{6d88096d} & \raggedright 5 beat-level statistics; split-half prediction rule; signal-property surrogates & \raggedright rule met in 0 of 10 statistic-cohort cells \tabularnewline
\addlinespace[1pt]
\raggedright Concept heads, six models (G) & \raggedright \texttt{ce7ed6fe} & \raggedright models, splits, reference, thresholds; hypotheses for PPG tachycardia, bradycardia, tachypnea and bradypnea; faithfulness across seeds & \raggedright mean AUROC 0.962 (monitor-supervised) vs.\ 0.865 (rule-supervised); lower error and more rule errors corrected for all 4; lower error than the rule-supervised head at its reference-calibrated threshold not shown for bradycardia \tabularnewline
\addlinespace[1pt]
\raggedright Counterfactual faithfulness (G) & \raggedright \texttt{6037da73} & \raggedright edits, baselines and pass criteria of 3 tests & \raggedright counterfactual test (T1) 5/7 concepts (randomized head and re-initialized model 0/7); evidence pointers (T2) 3/7; modality masking (T3) 4/7 \tabularnewline
\bottomrule
\end{tabular}
\end{table}

\paragraph{Deviations.}
The protocols' deviation logs record the following. (i)~Unit tests and reporting scripts were added after registration; no registered computation changed. (ii)~A construct-validity slope (CV3, baroreflex sensitivity against propofol effect-site concentration; App.~\ref{app:fast-cv}) was dominated by one epoch with an implausible concentration (3,501~$\mu$g\,mL$^{-1}$); a sensitivity analysis without it is reported as a deviation. (iii)~Faithfulness runs for two additional monitor-supervised seeds were added after registration with unchanged criteria and are labeled exploratory. (iv)~When the PAT-SBP edge was added, the Holm family of per-patient fast-edge tests was amended to include it. (v)~The 90\% positive-control bar for the PAT-SBP sign was registered after two exploratory results were known: negative PAT-SBP slopes in 56.5\% of cases on the beats retained by a fixed 80-600~ms foot-search window, and in 8 of 9 cases with delay-aware matching; no full-cohort delay-aware outcome had been computed. (vi)~Inference of five of the six concept models started before the concept-head protocol was hashed; it produced probabilities only, and no metric was computed before hashing.

\paragraph{Analyses outside the protocols.}
The main analyses that were not pre-specified are the admissions of the typed pooled estimator that populate the vital-level graph layer; the maintenance-phase split of propofol responses (App.~\ref{app:benchmark}); the PPG-amplitude, PAT-level and modal-lag adjustments of the PAT analyses (App.~\ref{app:fast}) and the cross-case null for fast edges (App.~\ref{app:graphs-fast}); balanced accuracy of the concept heads (App.~\ref{app:grounding}); and the synthetic calibration study (App.~\ref{app:supp-synthetic}). The recovery of ventilation$\to$EtCO$_2$ relations by the design-based estimators is an output of the registered benchmark run but was not a registered prediction.

\paragraph{Compute.}
Encoder pre-training used one NVIDIA A100 GPU. Native-rate extraction of the beat, breath, 1-s and PAT tables ran on cloud CPU instances (about 12 CPU-hours for the first three). All other computation ran on an Apple M4 Max laptop: the ClosedLoopBench run on 3,442 cases took 48~min with four worker processes, epoch extraction for the fast couplings 18~min, the PAT analyses 2~min, construction of 3,442 patient graphs 11~min with eight processes, and the synthetic study 24~min with 13 processes. Concept-model inference ran on the laptop GPU.

\paragraph{Software.}
Python 3.11 with PyTorch 2.12, NumPy 2.4, SciPy 1.17, pandas 3.0, statsmodels 0.14, scikit-learn 1.9, tigramite 5.2 (PCMCI+), lingam 1.13 (VAR-LiNGAM), wfdb 4.3 and vitaldb 1.7.2. Bootstrap confidence intervals are percentile intervals \citep{efron1994introduction} with fixed seeds.

\section{Supplementary results}
\label{app:supp}

\subsection{Synthetic calibration of edge tests}
\label{app:supp-synthetic}

On synthetic concept series with known lagged graphs, this study checks whether per-patient edge tests keep false discoveries near the nominal level under serial dependence and shared drivers, and how CKG behaves. It was not pre-specified (App.~\ref{app:prereg}).

\paragraph{Design.}
Each synthetic patient has seven concept series, one per concept head (App.~\ref{app:concepts}), of $T \in \{240, 960, 2400\}$ steps (15-s stride, 1-10~h). Four generators are used. \emph{Global null}: independent autoregressive series with persistence $\phi_{\mathrm{AR}} \in \{0, 0.5, 0.9, 0.97\}$. \emph{Sparse}: a stable vector autoregression with the same persistence and 7 random directed edges at lags $1$ to $\ell_{\max} = 4$, with coefficients of random sign and magnitude $\sqrt{1-\phi_{\mathrm{AR}}^2}$ times a draw from $\mathcal{U}(0.3, 0.6)$. \emph{Confounded}: independent autoregressive series plus one slow shared driver (autoregressive coefficient 0.99) loaded on every concept with its own delay of 0-4 steps; no direct edge exists. \emph{Closed-loop physiology}: HR, MAP and respiratory rate simulated under drug step schedules and mapped to concepts; the truth is the reduced-form lagged graph between different vitals. Latent series are observed either as probabilities (a logistic readout with noise) or as binary activations (thresholding). Every setting has 50 patients (20 for PCMCI+). The 54 non-confounded settings are 24 global-null, 24 sparse and 6 physiology settings; 24 settings are confounded. CKG applies the rules of App.~\ref{app:graphs-ckg} to the observed series on their original scale, with activation threshold 0.5 and a Granger $p$-value threshold equal to the false-discovery-rate (FDR) level $\alpha$; we score its Granger-precedes edges alone and its temporal layer (the union of precedes/follows and Granger-precedes edges). The other methods analyze probability readouts on the log-odds scale. Only lagged directed edges between distinct concepts are scored; contemporaneous links of PCMCI+ and VAR-LiNGAM are ignored. A setting violates FDR control at level $\alpha = 0.1$ when the mean false-discovery proportion minus two standard errors exceeds $\alpha$.

\paragraph{Circular-shift tests.}
A shift test compares a dependence statistic with its values after time-shifting one series against the other; \citet{harris2020shift} obtains a conservative test from non-circular shifts, and \citet{liang2025false} bound the false-positive rate of $p$-values computed over shifts of both series (circular shifts for whole-series statistics) when the series are stationary and independent; the bounds exceed the nominal level by a factor or by error terms that depend on mixing, and the circular-shift result also adds a margin to the shifted statistics and requires a stable statistic. Surrogate-based nulls with FDR control are standard in network inference \citep{kramer2009network,novelli2019large}. Our lattice shift test is a surrogate procedure related to these tests, but their guarantees do not cover it; we assess its calibration by the simulations below, including settings that violate its assumptions. The tested statistic is a conditional Granger statistic: the partial $R^2$ of lags $1$ to $\ell_{\max}$ of the source's AR($\ell_{\max}$) innovations, given lags $1$ to $\ell_{\max}$ of the six other concepts, target included. Only the source innovations are shifted, by multiples of $s = \max(\ell_{\max}+1, \lceil (1+r)/(1-r) \rceil)$, where $r$ is their lag-1 autocorrelation clipped to $[0, 0.995]$; for the innovations of every synthetic series, $s = \ell_{\max}+1 = 5$. Prewhitening leaves $T - \ell_{\max}$ steps; when $s$ does not divide this length, the earliest steps are dropped, leaving $T' = s\lfloor (T-\ell_{\max})/s \rfloor$ steps on which the observed and shifted statistics are computed. The shifts $ks$, $k = 1, \dots, T'/s-1$, and the identity form a cyclic group of $T'/s$ rotations (47, 191 and 479 at $T = 240$, 960 and 2400), all of which are used; the $p$-value is the fraction of them whose statistic is at least the observed one, so it cannot fall below $s/T'$ (0.021 at $T = 240$). This $p$-value would be exact if the source innovations were independent of the other concepts and had the same distribution under every rotation; rotation joins the end of the series to its start, so any serial dependence left in the innovations makes this approximate. Rotation also removes the source's alignment with every concept, not only the target, so the surrogates approximate the null of no direct edge only when conditioning on the other concepts' past blocks indirect paths and no unrecorded driver is shared (see (iii) and (iv) below). The circular-shift nulls for concept relations likewise shift by multiples of a step $s_p$ of at least the maximum lag plus one (App.~\ref{app:relations-admission}). Decisions use Benjamini-Hochberg (BH; \citealp{benjamini1995controlling}) or Benjamini-Yekutieli (BY; \citealp{benjamini2001control}) control over the 42 ordered pairs of a patient.

\begin{table}[!htbp]
\centering
\caption{Synthetic edge tests at FDR level $\alpha = 0.1$. FDR violations over the 54 non-confounded settings; global-null family-wise error rate (FWER) averaged over the 24 null settings; false edges per graph averaged over the 24 shared-driver settings.}
\label{tab:supp-synthetic}
\footnotesize
\begin{tabular}{@{}lccc@{}}
\toprule
Method & FDR violations & Global-null & Shared driver: \\
 & (of 54) & FWER & false edges per graph \\
\midrule
CKG (union) & 54 & 1.00 & 35.8 \\
CKG Granger-precedes & 54 & 1.00 & 32.8 \\
PCMCI+, per-link level & 43 & 0.94 & 11.5 \\
PCMCI+, BH over 42 pairs & 3 & 0.10 & 1.95 \\
VAR-LiNGAM & 41 & 0.34 & 14.6 \\
Conditional Granger $F$-test, BY & 9 & 0.06 & 4.81 \\
Lattice shift test, BH & 3 & 0.03 & 2.54 \\
Lattice shift test, BY & 0 & 0.00 & 0.79 \\
Lattice shift test, BH, driver recorded & n/a & n/a & 0.41 \\
\bottomrule
\end{tabular}
\end{table}

\paragraph{Results.}
Table~\ref{tab:supp-synthetic} summarizes the comparison. The baselines receive the same series: PCMCI+ \citep{runge2020discovering} at a per-link level $\alpha$, or with BH over the 42 pairs applied to each pair's Bonferroni-adjusted minimum $p$-value over lags; VAR-LiNGAM \citep{hyvarinen2010estimation}, with an edge for any lagged coefficient above 0.01 in absolute value; and conditional Granger $F$-tests \citep{granger1969investigating} from a vector autoregression of order $\ell_{\max}$ on all seven series.
(i)~Lattice $p$-values are close to nominal under independence for every persistence and readout. In a separate simulation of the global null (108 cells: three statistics, four persistence values, three lengths and three readouts including the latent series; 4,200 $p$-values per cell), the rejection rate is at most 0.060 at level 0.05 and at most 0.110 at level 0.1. The FWER of the lattice test with BH is at most 0.12 at $\alpha = 0.1$ over the 24 global-null settings (50 graphs each), within binomial error of the nominal level. Shift sets drawn from a minimum-gap window instead of a lattice are anti-conservative: for one independent pair at $T = 240$ and $\phi_{\mathrm{AR}} = 0.9$, $\Pr(p \le 0.01) = 0.055$ over 3,000 replications.
(ii)~CKG has no error control: its temporal layer gives 8-27 false edges per graph under the global null, a family-wise error rate of 1.00 in every global-null setting, and FDR violations in all 54 non-confounded settings.
(iii)~Baselines given the same multiplicity correction are competitive: PCMCI+ with BH over the same 42 pairs violates FDR in 3 of 54 settings, as does the lattice test with BH; conditional Granger $F$-tests with BY violate it in 9. All three BH violations of the lattice test occur in sparse settings with binary readouts at $T = 2400$ (FDR 0.18-0.21). Conditioning on on/off readouts does not remove indirect dependence: in a separate simulation of sparse graphs, pairs without a direct edge but linked through other concepts are rejected at level 0.05 at rates up to 0.099 with binary readouts, against at most 0.049 on the latent series. The power of the lattice test on sparse graphs grows with record length (true-positive rate 0.02, 0.53 and 0.67 at $T = 240$, 960 and 2400 with BH; 0.00, 0.00 and 0.52 with BY).
(iv)~An unrecorded shared driver produces false edges for every method (2.54 per graph for the lattice test with BH). Conditioning on the driver, as recorded clinician actions allow, reduces this to 0.41.

\subsection{Pressure concepts against the cuff and without an arterial line}
\label{app:supp-pressure}

Hypotension and hypertension are read from the arterial waveform. On the MIMIC-IV-WDB test split we score the pressure heads against three references (Table~\ref{tab:supp-pressure}): the monitor's mean arterial pressure from the same transducer (8 patients, 40,211 windows); the oscillometric cuff mean pressure nearest the window center within 300~s while the arterial line is present, a physically separate sensor (8 patients, 2,030 windows); and the cuff pressure in windows without an arterial waveform, where the fused model sees only ECG, PPG and respiration (21 patients, 55,557 windows). Against the arterial numeric, the seed-mean AUROC is at least 0.97 for both supervision sources, reflecting the pressure level carried by the arterial input, which is scaled but not standardized (App.~\ref{app:data-windows}). Against the cuff, the heads reach 0.95-0.96 for hypertension and 0.62-0.65 for hypotension, similar to the rule detector (0.98 and 0.62; CIs are wide with 8 patients). Without the arterial waveform, AUROC across the six models and both concepts ranges from 0.386 to 0.530: pressure concepts are not inferred from ECG, PPG or respiration. The concept-head protocol (Table~\ref{tab:prereg}) specified these comparisons without hypotheses.

\begin{table}[!htbp]
\centering
\caption{Pressure-concept AUROC on the MIMIC-IV-WDB test split by reference (8, 8 and 21 patients). Heads: mean $\pm$ s.d.\ over three seeds; rule detector: 95\% patient-cluster bootstrap CI.}
\label{tab:supp-pressure}
\footnotesize
\setlength{\tabcolsep}{3pt}
\begin{tabular}{@{}llccc@{}}
\toprule
Concept & Model & Arterial numeric & Cuff, line in place & Cuff, no line \\
\midrule
Hypotension & rule-supervised & $0.972 \pm 0.009$ & $0.647 \pm 0.022$ & $0.452 \pm 0.068$ \\
 & monitor-supervised & $0.983 \pm 0.002$ & $0.623 \pm 0.063$ & $0.401 \pm 0.015$ \\
 & rule detector & 0.988 [0.984, 0.996] & 0.617 [0.285, 0.964] & n/a \\
\midrule
Hypertension & rule-supervised & $0.995 \pm 0.001$ & $0.962 \pm 0.011$ & $0.452 \pm 0.025$ \\
 & monitor-supervised & $0.987 \pm 0.003$ & $0.951 \pm 0.006$ & $0.513 \pm 0.007$ \\
 & rule detector & 0.994 [0.992, 0.996] & 0.977 [0.889, 0.989] & n/a \\
\bottomrule
\end{tabular}
\end{table}

\end{document}